\documentclass[twoside,11pt]{article}

\usepackage{amsmath}
\usepackage{amsfonts}
\usepackage{amsthm}
\usepackage{amssymb}
\usepackage{mathabx}
\usepackage{algorithm}
\usepackage{algpseudocode}
\usepackage{mathtools}
\usepackage{etoolbox}

\newcounter{constant} 
\newcommand{\Alg}{\texttt{Alg}}

\let\hat\widehat
\let\tilde\widetilde

\usepackage{tikz}

\makeatletter
\def\maketag@@@#1{\hbox{\m@th\normalfont\normalsize#1}}
\makeatother

\makeatletter
\def\@fnsymbol#1{\ensuremath{\ifcase#1\or  \natural \or \dagger\or * \or \ddagger\or
   \mathsection\or \mathparagraph\or \|\or **\or \dagger\dagger
   \or \ddagger\ddagger \else\@ctrerr\fi}}

\makeatother

\newcommand{\ib}{\mathrm{i}}

\newcommand{\cA}{\mathcal{A}}

\newcommand{\cD}{\mathcal{D}}
\newcommand{\cE}{\mathcal{E}}
\newcommand{\cF}{\mathcal{F}}
\newcommand{\cG}{\mathcal{G}}
\newcommand{\cH}{\mathcal{H}}

\newcommand{\cJ}{\mathcal{J}}

\newcommand{\cM}{\mathcal{M}}
\newcommand{\cN}{\mathcal{N}}
\newcommand{\cO}{\mathcal{O}}

\newcommand{\cQ}{\mathcal{Q}}

\newcommand{\cS}{{\mathcal{S}}}
\newcommand{\cT}{{\mathcal{T}}}
\newcommand{\cU}{\mathcal{U}}

\newcommand{\cW}{\mathcal{W}}
\newcommand{\cX}{\mathcal{X}}
\newcommand{\cY}{\mathcal{Y}}
\newcommand{\cZ}{\mathcal{Z}}

\newcommand{\BB}{\mathbb{B}}
\newcommand{\CC}{\mathbb{C}}
\newcommand{\EE}{\mathbb{E}}

\newcommand{\II}{\mathbb{I}}

\newcommand{\PP}{\mathbb{P}}
\newcommand{\DD}{\mathbb{D}}

\newcommand{\RR}{\mathbb{R}}

\newcommand{\ZZ}{\mathbb{Z}}

\newcommand{\argmin}{\mathop{\mathrm{argmin}}}
\newcommand{\argmax}{\mathop{\mathrm{argmax}}}

\usepackage{esint}

\newcommand{\bignorm}[1]{\bigg|\bigg|#1\bigg|\bigg|}

\newcommand{\wcheck}{\widecheck}

\newcommand{\inner}[2]{\left\langle #1,#2 \right\rangle}
\newcommand{\rbr}[1]{\left(#1\right)}
\newcommand{\sbr}[1]{\left[#1\right]}
\newcommand{\cbr}[1]{\left\{#1\right\}}

\newcommand{\abr}[1]{\left|#1\right|}

\usepackage{blindtext}

\renewcommand{\hat}[1]{\widehat{#1}}
\newcommand{\dia}{\diamond}

\newcommand{\pll}{\kern 0.3em/\kern -0.9em /\kern 0.3em}
\newcommand{\sh}{\sharp}

\newcommand{\Opt}{\text{Opt}}
 \usepackage[preprint]{jmlr2e}

\usepackage{etoolbox}

\usepackage{lastpage}

\ShortHeadings{Interleaved Resampling and Refitting}{Hu and Simchi-Levi}
\firstpageno{1}

\begin{document}

\title{\Large Interleaved Resampling and Refitting: Data and Compute-Efficient Evaluation of Black-Box Predictors}

\author{\name Haichen Hu \email huhc@mit.edu \\
       \addr Laboratory for Information and Decision Systems\\
       Massachusetts Institute of Technology\\
      Cambridge, MA 02139, USA
       \AND
       \name David Simchi-Levi \email dslevi@mit.edu \\
       \addr Laboratory for Information and Decision Systems\\
       Massachusetts Institute of Technology\\
       Cambridge, MA 02139, USA}

\editor{}

\maketitle

\begin{abstract}
We study the problem of evaluating the excess risk of large-scale empirical risk minimization under the square loss. Leveraging the idea of wild refitting and resampling, we assume only black-box access to the training algorithm and develop an efficient procedure for estimating the excess risk. Our evaluation algorithm is both computationally and data efficient. In particular, it requires access to only a single dataset and does not rely on any additional validation data. Computationally, it only requires refitting the model on several much smaller datasets obtained through sequential resampling, in contrast to previous wild refitting methods that require full-scale retraining and might therefore be unsuitable for large-scale trained predictors.

Our algorithm has an interleaved sequential resampling-and-refitting structure. We first construct pseudo-responses through a randomized residual symmetrization procedure. At each round, we thus resample two sub-datasets from the resulting covariate–pseudo-response pairs. Finally, we retrain the model separately on these two small artificial datasets. We establish high-probability excess risk guarantees under both fixed-design and random-design settings, showing that with a suitably chosen noise scale, our interleaved resampling and refitting algorithm yields an upper bound on the prediction error. Our theoretical analysis draws on tools from empirical process theory, harmonic analysis, Toeplitz operator theory, and sharp tensor concentration inequalities.
\end{abstract}

\section{Introduction}
Deep Neural Networks (DNN) and Artificial Intelligence (AI) have been very important in modern machine learning \citep{russell1995modern, Goodfellow-et-al-2016}. More recently, Large Language Models (LLMs) have become a revolutionary force, with tremendous applications in science \citep{liangmapping}, healthcare \citep{haltaufderheide2024ethics}, operations management \citep{wasserkrug2024combining}, education \citep{moore2023empowering}, and so on.

Despite the unprecedented success and widespread deployment of deep neural networks and artificial intelligence (AI) models, one of their most significant, and arguably most concerning characteristics is that they predominantly operate as opaque, black-box predictors \cite{castelvecchi2016can,cahill2025ai}. This opacity is particularly evident in modern Large Language Models (LLMs). Even after extensive empirical observation and rigorous probing, researchers still do not fully understand the internal mechanisms governing the emergent behaviors and decision-making processes of these models. This persistent lack of interpretability stems from two compounding factors of scaling: the astronomical size of the pre-training datasets and the sheer structural complexity of the underlying architectures.

First, modern AI models are pre-trained on text of unprecedented scale, often encompassing billions or even trillions of samples \cite{brown2020language,dubey2024llama}. Auditing this sheer volume of unstructured, web-scale data makes it nearly impossible to trace specific model outputs back to their original training instances or to fully understand the origins of a model's learned associations, reasoning pathways, and potential biases.

Second, the underlying neural architectures are immensely complicated. State-of-the-art Large Language Models rely on deep Transformer networks \cite{vaswani2017attention} characterized by billions of trainable parameters distributed across numerous layers with massive dimensional widths \cite{brown2020language}. The highly non-linear, high-dimensional interactions among these parameters render traditional analytical methods ineffective, completely obscuring the internal logic that transforms an input prompt into a coherent output.

However, rigorously evaluating the reliability and interpretability of these opaque models is a vital prerequisite for real-world deployment \cite{lipton2018mythos}. In high-stakes domains, blindly trusting black-box predictors risks catastrophic consequences. For instance, unrecognized algorithmic biases in clinical healthcare can drive severe misdiagnoses and compromise patient safety \cite{rudin2019stop,obermeyer2019dissecting}. Similarly, in safety-critical systems like autonomous driving, a model's failure to reliably communicate uncertainty when encountering out-of-distribution data can lead to fatal failures \cite{kendall2017uncertainties,amodei2016concrete}. Therefore, before such models are integrated into critical societal infrastructure, it is essential to develop efficient and mathematically sound methods for quantifying their reliability.

Therefore, given the widespread deployment and inherent opacity of LLMs, developing a principled framework to rigorously evaluate these black-box models remains a critical open problem. Crucially, such a framework must achieve efficiency across two dimensions: \emph{computational efficiency and data efficiency}. 

First, let us consider \emph{computational efficiency}. Training large-scale AI systems, such as the pre-training of LLMs, typically involves processing trillions of data points over months on thousands of GPUs \citep{hoffmann2022training,kaplan2020scaling}. Therefore, any practical validation method must avoid incurring a second computational burden comparable to that of the original training process. To this end, a good evaluation procedure should be able to operate on much smaller datasets so that its computational cost constitutes only a marginal fraction of the original computation cost while still providing a valid evaluation of the pre-trained model.

Second, the evaluation framework must be \emph{data-efficient}. Modern LLMs are data-hungry \citep{villalobos2024position}, demanding exhaustive quantities of high-quality data to sustain performance gains \citep{klugscaling}. Consequently, practitioners rarely have the luxury of holding out massive, high-fidelity datasets for traditional evaluation techniques like cross-validation or standard data splitting. Instead, in modern large-scale applications, nearly all available high-quality data is typically used directly for training in order to maximize predictive performance. Thus, an ideal evaluation method must be capable of rigorously assessing the model without relying on abundant reserves of unseen validation data.

\paragraph{Our Contributions.}
In this paper, we develop an algorithmic framework for evaluating black-box predictors that is efficient in both computation and data usage. Specifically, for an arbitrary black-box training algorithm, we study the problem of estimating its excess risk, arguably the most fundamental metric for assessing generalization performance and the overall quality of a learned model. Our evaluation procedure combines \emph{interleaved resampling} \citep{yu2002resampling} with \emph{wild refitting} \citep{wainwright2025wild}, allowing us to decompose an otherwise massive evaluation task into many moderate-scale subproblems. More importantly, our statistical guarantees are model-free and do not rely on classical function-class complexity measures such as Rademacher complexity or covering numbers; instead, we assume only black-box access to the training procedure together with a single dataset. In addition, we establish excess-risk bounds under both fixed-design and random-design settings, whereas previous work on wild refitting has been limited to the fixed-design regime \citep{wainwright2025wild,hu2025perturbingderivativewildrefitting,hu2025doublywildrefittingmodelfree}. In fact, our bounds remain valid under a mere square-integrability assumption, even without requiring continuity of the underlying model class. Technically, our analysis draws on tools from harmonic analysis, the spectral theory of Toeplitz operators, and sharp tensor concentration inequalities.

\subsection{Literature Review}
Statistical learning theory \citep{vapnik2013nature} provides the mathematical foundation for the rigorous analysis of learning algorithms, with empirical risk minimization (ERM) \citep{vapnik1991principles} serving as one of its central paradigms. A core objective in this literature is to characterize the excess risk, which quantifies the gap between empirical performance and population risk. Traditionally, controlling this gap has relied on the rich machinery of empirical process theory. Through a variety of capacity measures—including the VC dimension \citep{vapnik2015uniform,floyd1995sample}, covering numbers in metric spaces \citep{van2000empirical}, intrinsic dimension \citep{ansuini2019intrinsic}, Rademacher complexity \citep{massart2007concentration,bartlett2005local}, fat-shattering dimension \citep{bartlett1994fat}, Littlestone dimension \citep{haussler1994predicting}, and spectral eigenvalue decay conditions \citep{goel2017eigenvalue,hu2025contextual}—researchers have derived sharp risk bounds for a broad range of hypothesis classes. However, obtaining valid and tractable upper bounds on these complexity measures becomes highly challenging when the underlying function class is very rich, as is often the case for black-box predictors such as deep neural networks and large language models.

Generally, evaluating a training algorithm involves splitting the available data into disjoint training and validation sets, and then using a portion of the data to assess the performance of the trained model \citep{browne2000cross,tibshirani2017statistical}. Variants of this approach include leave-one-out cross-validation \citep{wong2015performance,fukunaga2002leave} and $k$-fold cross-validation \citep{wong2019reliable}. However, because all such methods rely on sample splitting, they do not fully allocate the available data to training and are therefore inherently data-inefficient. This limitation is particularly pronounced in modern large-scale settings such as LLM training, where high-quality data is scarce and valuable.

More recently, \citet{wainwright2025wild} introduced the concept of \emph{wild refitting} for regression. This approach upper bounds the excess risk under fixed design by constructing a carefully perturbed synthetic dataset and refitting the model. Subsequently, \citet{hu2025perturbingderivativewildrefitting,hu2025doublywildrefittingmodelfree} broadened this framework to accommodate more general Bregman and convex loss functions. Despite these theoretical advances, existing wild refitting techniques face significant scalability bottlenecks when evaluating massive AI models. Because they require retraining the model on synthetic datasets of a size comparable to the original training corpus, they may be computationally impractical in modern large-scale applications.

Algorithmically, our method is related to resampling in statistics. Classical resampling methods, such as the jackknife \citep{quenouille1949approximate,tukey1958bias} and the bootstrap \citep{efron1992bootstrap,mammen1993bootstrap}, provide effective tools for inference, especially in settings where classical asymptotic approximations are inadequate. Broadly speaking, resampling methods can be divided into those based on sampling with replacement, such as the bootstrap \citep{efron1994introduction}, and those based on sampling without replacement, such as subsampling \citep{politis1994large}; both classes of methods admit consistency guarantees under suitable conditions. In this paper, we utilize oracles based on resampling without replacement. While a naive implementation that evaluates all possible subsamples is computationally infeasible, foundational algorithmic advances in sampling without replacement \citep{vitter1985random,bentley1987programming,antal2014new} avoid this combinatorial explosion and make such procedures highly scalable in modern large-scale applications.
\paragraph{Paper Structure.} As noted earlier, our method draws on harmonic analysis, Toeplitz operator theory, and sharp concentration inequalities for simple tensors. For the reader’s convenience, Appendix \ref{app:math_tools} provides a comprehensive overview of the technical tools used throughout the paper.

We begin in Section \ref{sec:model_setup} by introducing an empirical risk minimization model on a compact univariate covariate space for ease of exposition. We then present our evaluation algorithm in Section \ref{sec:algorithm}. Next, we establish rigorous excess-risk guarantees under the fixed-design setting in Section \ref{sec:statistical_guarantee}.

In particular, Section \ref{sec:statistical_guarantee} is organized as follows. In Subsection \ref{subsec:notation}, we introduce the necessary preliminaries on empirical processes and probability. In Subsection \ref{subsec:bounding_risk_fixed_design}, we present the concrete fixed-design excess-risk bound. We then provide a practical interpretation of this result in Subsection \ref{subsec:bounding_hatr_n}.

Building on the fixed-design excess-risk bound established in Section \ref{sec:statistical_guarantee}, we derive the random-design risk bound in Section \ref{sec:guarantee_random_design}. We then extend our results to empirical risk minimization over high-dimensional multivariate covariate spaces in Section \ref{sec:high_dimension}. Finally, Section \ref{sec:experiment} presents numerical experiments illustrating the practical performance of our algorithm.
\paragraph{Notations.}
We use $[n]$ to denote the set $\cbr{1,2,\cdots,n}$. For any function $f$, we write $\hat{f}(\xi)$ for its Fourier transform at frequency $\xi$. For any dataset $\cD$ and any black-box procedure $\Alg$, we let $\Alg(\cD)$ denote the output of $\Alg$ when applied to $\cD$. For any set $\cX$, we use $\cX^d$ to denote its $d$-fold product space. Given a dataset $\cD=\cbr{x_i}_{i=1}^{n}$ and a function $f$, we define $\|f\|_{\cD}:=\sqrt{\frac{1}{n}\sum_{i=1}^{n}f(x_i)^2}$
to be the empirical $L^2$ norm of $f$ over $\cD$. Given a function class $\cG$, a function $g\in\cG$, and a radius $r>0$, we write $\BB_r(g;\cD):=\cbr{h\in\cG:\|h-g\|_{\cD}\le r}$
for the empirical $L^2$ ball centered at $g$ with radius $r$. Finally, for any $d$ and any multi-index $k\in\ZZ^d$, when there is no confusion regarding the dimension, we let $e_k(x)$ denote the complex exponential function on $\RR^d$ defined by $e_k(x):=e^{2\pi \ib k\cdot x}.$ For any Euclidean set $\Omega$, we use $\mu$ to denote the Lebesgue measure and use $L^2(\Omega)$ to denote the Hilbert space of functions that are square-integrable with respect to the Lebesgue measure. For any function class $\cF$, we use $\DD(\cF)$ to denote its function difference class $\cbr{f-f':f,f'\in\cF}$.

\section{Model Setup: Empirical Risk Minimization with Square Loss}\label{sec:model_setup}
In this section, we introduce the empirical risk minimization (ERM) framework under the squared loss, together with several mild assumptions. For simplicity, we first focus on the case where the covariate is univariate and takes values in $\RR$. Without loss of generality, and for notational convenience in the Fourier-analytic arguments, we assume that the covariate space is $\Omega=[0,1]$ and that the response space satisfies $\cY\subseteq\RR$. We observe a dataset $\cD=\cbr{(x_i,y_i)}_{i=1}^{n}$, where the samples $(x_i,y_i)$ are drawn i.i.d. from an unknown distribution $\PP^*_{X,Y}$. We write $\nu$ for the marginal distribution of the covariate $X$, and model the true data-generating process as
\[
y_i = f^*(x_i) + w_i,
\]
where $f^*$ is unknown and $w_i$ is some zero-mean noise independent of $x_i$. For simplicity, we assume that the noise distribution has compact support $[-\tau,\tau]$ for some $\tau>0$.

In empirical risk minimization, with some pre-specified underlying function class $\cF$, we train a predictor $\Breve{f}$ using a black-box procedure $\Alg$ to minimize the empirical loss, i.e., $\Breve{f}=\Alg(\cD)$ s.t.
\[
\Breve{f}=\argmin_{f\in\cF}\cbr{\frac{1}{n}\sum_{i=1}^{n}(y_i-f(x_i))^2}.
\]
For example, $\cF$ can be interpreted as the underlying training architecture in deep neural networks or AI models. 

Regarding the function class $\cF$, throughout the paper, we consider a well-specified case, where we assume that $f^*\in\cF$. Moreover, we introduce the following assumption about its Fourier coefficients.
\begin{assumption}\label{ass:func_class}
$\cF\subset L^2([0,1])$. $\exists v>1/2$ and $ M_v>0$, s.t. $\sup_{f\in\cF}\hat{f}(k)\le \frac{M_v}{|k|^v}$, $\forall k\in\ZZ\setminus\cbr{0}$, where $\hat{f}(k)$ is the Fourier coefficient of $f$ at frequency $k$.
\end{assumption}
We emphasize that Assumption \ref{ass:func_class} is quite mild and does not compromise the black-box nature of the training procedure $\Alg$. First, for any function $f$, square integrability alone implies, by Lemma \ref{lemma:plancherel}, that $\|f\|_{L^2}^2=\sum_{k\in\ZZ}|\hat{f}(k)|^2<\infty$. Thus, heuristically, if the Fourier coefficients satisfy $|\hat{f}(k)|\asymp |k|^{-v}$, then the convergence of the corresponding $p$-series requires $v>1/2$. Second, we do not even assume that $\cF$ is a continuous class, which allows our black-box framework to accommodate discontinuous prediction tasks such as classification. More generally, Assumption \ref{ass:func_class} is satisfied by a broad range of classical function classes relevant to modern deep learning architectures, including those arising in deep neural networks and large language models \citep{barron2002universal,ongiefunction,kim2021fast,tancik2020fourier,bar2022spectral,ronen2019convergence,xu2019frequency}.
\begin{example}\citep{rahaman2019spectral}\label{example:fourier_relu}
    For ReLU neural network, $f: \RR^d\rightarrow\RR$ such that $f(x)=T^{(L)}\circ\sigma\circ T^{(L-1)}\circ\cdots\circ\sigma\circ T^{(1)}(x)$, where $\sigma$ is ReLU activation and $T^{(j)}(x)=W^{(j)}x+b^{(j)}$ is affine mapping. Then, we have
    $|\hat{f}(k)|\le\frac{\prod_{j=1}^{L}\|W^{(j)}\|_{op}}{|k|^{d+1}}$.
\end{example}
In Example~\ref{example:fourier_relu}, when \(d=1\), the Fourier coefficients exhibit a decay rate of order \(v=2\), and the corresponding constant \(M_v\) is upper bounded by \(\prod_{j=1}^{L}\|W^{(j)}\|_{op}\). More generally, we have the following example, which states that similar conclusions hold for any feed-forward networks with non-expansive activations.
\begin{example}\citep{bar2022spectral}\label{example:fourier_fc}
    Let $\phi:\RR^m\rightarrow\RR$ be a $L$-layers feed-forward network with non-expansive activation functions, and weights $\cbr{W_i}_{i=1}^{L}$, and bias $\cbr{b_i}_{i=1}^{L}$, then, we have that $|\hat{\phi}(k)|\le \frac{\prod_{j=1}^{L}\|W_{j}\|_{op}}{\|k\|_2}$, $\forall k\in\ZZ^m$.
\end{example}
Example \ref{example:fourier_fc} covers many commonly used activation functions, including ReLU, sigmoid, and hyperbolic tangent, and is also compatible with linear layers, pooling layers, and convolutional layers \citep{bar2022spectral}.

In practice, when the quantity \(\prod_{j=1}^{L}\|W_j\|_{op}\) becomes large, spectral normalization is often imposed to keep it on the order of \(\Theta(1)\), thereby stabilizing the training process. Such normalization mechanisms are widely used in modern deep neural network architectures \citep{miyato2018spectral}, including transformer-based architectures for LLMs \citep{zhai2023stabilizing,gray2024normalization}. More specifically, for transformer architectures, whenever layer normalization is incorporated into the self-attention layers \citep{wu2024role,xiong2020layer}, the resulting trained output satisfies Assumption \ref{ass:func_class} in a coordinate-wise way. See Proposition \ref{prop:transformer_func_ass} in Appendix \ref{app:proofs_sec:model_setup} for a precise statement.


Although our theorems hold for any $v>1/2$, throughout this paper, we will mostly focus on the regime $1/2< v \le 1$. The reason is that when \(v>1\), Lemma \ref{lemma:holder_continuity} states that the Fourier decay in Assumption \ref{ass:func_class} implies that functions in \(\cF\) are H\"older continuous, placing the problem squarely within the scope of more traditional statistical learning theory. In that smoother regime, one can typically control the complexity of \(\cF\) via standard upper bounds on quantities such as the covering number $\cN(\epsilon,\cF,\|\cdot\|_{\infty})$ of the model class \citep{kolmogorov2019varepsilon}.

\begin{lemma}\label{lemma:holder_continuity}
$f \in L^2([0,1])$. There exist $v > 1$ and $M_v > 0$ such that the Fourier coefficients of $f$ satisfy $|\hat{f}(k)| \le \frac{M_v}{|k|^v},\ \forall k \in \mathbb{Z} \setminus \{0\}.$
Then $f$ is uniformly continuous. In particular:
\begin{enumerate}
    \item If $1 < v < 2$, $f$ is uniformly Hölder continuous with exponent $\alpha = v - 1$. That is, there exists a constant $C > 0$ depending only on $v$ and $M_v$ such that for all $x, y \in [0,1]$,
    \begin{equation}
        |f(x) - f(y)| \le C |x - y|^{v-1}.
    \end{equation}
    \item If $v \ge 2$, $f$ is Lipschitz continuous.
\end{enumerate}
\end{lemma}
On the other hand, when $v \le 1$, the function class becomes much more complicated. In contrast to the smooth regime $v>1$, where the absolute summability of Fourier coefficients yields uniform regularity and finite metric entropy, the regime $v \le 1$ permits pronounced local irregularities. We make two observations in this setting. First, since we do not assume $\cF$ is bounded in $L^1$ norm, it is possible for the function class $\cF$ to have $\cN(\epsilon,\cF,\|\cdot\|_{L^1})=\infty.$
Second, even under global boundedness constraints, the following lemma shows that at the critical threshold $v=1$, the $L^\infty$ covering number remains infinite even when an explicit uniform boundedness condition is imposed.
\begin{lemma}\label{lemma:infinite_covering_bounded}
Let $M_1 > 0$ and $B \ge 2M_1 \int_0^\pi \frac{\sin(t)}{t} dt \approx 3.70 M_1$. Define the function class $\mathcal{F}$ as continuous functions on $[0,1]$ with bounded supremum norm and $\mathcal{O}(1/|k|)$ Fourier decay:
\begin{equation}
    \mathcal{F} = \left\{ f \in C([0,1]) \;\middle|\; \|f\|_\infty \le B, \text{ and } \forall k \in \mathbb{Z} \setminus \{0\}, \, |k||\hat{f}(k)| \le M_1 \right\}.
\end{equation}
Then, there exists an $\epsilon_0 > 0$ such that for all $0 < \epsilon < \epsilon_0$, the covering number is infinite, i.e., $N(\epsilon, \mathcal{F}, \|\cdot\|_\infty) = \infty$.
\end{lemma}
With the discussion above, we conclude that Assumption \ref{ass:func_class} does not compromise the black-box nature of $\Alg$ and proceed under this assumption.

Now, we take a closer look at the covariate distribution $\nu$. In particular, we do not want $\nu$ to be excessively spiky or irregular relative to the underlying Lebesgue measure $\mu$. We therefore impose the following assumption on the Radon–Nikodym derivative of $\nu$ with respect to $\mu$.
\begin{assumption}\label{ass:covariate_density}
    Denote $w(x)$ to be the Radon–Nikodym derivative of $\nu$ with respect to $\mu$, $w(x)=\frac{d\nu}{d\mu}(x)$. There exist $0<\underline{w}\le \widebar{w}<\infty$, such that $w(x)\in[\underline{w},\widebar{w}],\ \forall x\in[0,1]$. 
\end{assumption}
As a remark, since $\mu([0,1])=\nu([0,1])=1$, we know that $\widebar{w}\ge 1$ and $\underline{w}\le 1$.

In statistical learning theory, the central object of study for evaluating the generalization ability of any data-dependent predictor $\Breve{f}=\Alg(\cD)$\footnote{The notation $\Breve{f}$ may seem slightly unusual. We use it because, in our paper, the notation $\hat{f}$ is already reserved for the Fourier transform of $f$.} is the \emph{excess risk} under a random design setting. Specifically, when employing Empirical Risk Minimization (ERM) under the squared loss, this excess risk takes the form
\[
\cE(\Breve{f}):=\EE_{X\sim\nu}[(\Breve{f}(X)-f^*(X))^2]=\int_{\Omega}|\Breve{f}(x)-f^*(x)|^2d\nu(x),
\]
where the test point $X$ is drawn independently of both the trained predictor $\Breve{f}$ and the training dataset $\cD$. Furthermore, under the \emph{fixed design} setting, where the training covariates $\cbr{x_i}_{i=1}^{n}$ are treated as deterministic, we define the empirical counterpart of $\cE(\Breve{f})$, the \emph{empirical excess risk} $\cE_\cD(\Breve{f})$ as:
\[
\cE_{\cD}(\Breve{f}):=\|\Breve{f}-f^*\|_{\cD}^2=\frac{1}{n}\sum_{i=1}^n(\Breve{f}(x_i)-f^*(x_i))^2.
\]
Recent work \citep{wainwright2025wild,hu2025perturbingderivativewildrefitting,hu2025doublywildrefittingmodelfree} shows that the empirical excess risk $\cE_\cD(\breve{f})$ can be estimated through full-scale retraining. However, this approach becomes computationally prohibitive when the size of the training dataset $\cD$ is exceptionally large. For instance, pre-training a modern large language model typically consumes upwards of trillions of tokens and requires hundreds of GPUs running continuously for weeks or months. Motivated by these practical bottlenecks, we impose the following constraint on the computational scale of the refitting procedure.
\begin{assumption}\label{ass:computation_budget}
   In the evaluation procedure, we are only able to refit the black box procedure $\mathtt{Alg}$ on datasets whose size does not exceed $n^{\beta}$, for some $0 < \beta < 1$\footnote{If $n^{\beta}\notin \mathbb{Z}^+$, we simply round it to an integer. For notational simplicity, we omit this rounding operation here.}.
\end{assumption}
Conceptually, the computational burden of each small-scale refitting step is comparable to that of post-training or supervised fine-tuning. These stages typically operate on much smaller datasets and require only a marginal fraction of the total compute used in pre-training \citep{ouyang2022training,bai2022training}. Accordingly, in the context of evaluating a large language model, our procedure may be viewed as carrying out multiple small-scale refitting steps at the post-training scale to assess black-box predictors originally trained at the pre-training scale.
\section{Algorithm: Interleaved Resampling and Wild Refitting}\label{sec:algorithm}
In this section, we present our evaluation algorithm for the black-box procedure $\Alg$. Our algorithm alternates between sequential resampling and wild refitting.

As described by \citet{wainwright2025wild}, a traditional wild refitting of an algorithm $\Alg$ with an input dataset $\cD$ generally entails two steps. First, we construct perturbed pseudo-responses via Rademacher symmetrization on the residuals and obtain a dataset $\tilde{\cD}$ with the same scale as $\cD$. Second, the model is refitted at full-scale to the synthetic dataset $\tilde{\cD}$ to construct the wild predictor.

However, in this paper, without a full-scale retraining budget, we cannot refit the model at the original scale. Thus, we must approximate the behavior of full-scale retraining. In statistics, resampling provides a principled approach for such approximation. Classical methods such as cross-validation and the bootstrap repeatedly sample the data to approximate quantities of interest, including the test error and the variability of estimators. Motivated by this perspective, our evaluation algorithm adopts an interleaved resampling scheme combined with small-scale refitting. In brief, our evaluation algorithm leverages the statistical properties of wild refitting and the computational efficiency of resampling while respecting the available compute budget. This design enables us to approximate the behavior of full-scale retraining without incurring its prohibitive cost.
\paragraph{Evaluation Algorithm Steps.}
Our algorithm proceeds as follows. Given the original full-scale training on $\cD=\cbr{(x_i,y_i)}_{i=1}^{n}$ and $\Breve{f}=\Alg(\cD)$, we start with a warm-up phase. Specifically, we compute the residuals $\cbr{v_i}_{i=1}^{n}$ between the observed outcomes $\cbr{y_i}_{i=1}^{n}$ and the outputs of a pre-trained recentering pilot predictor $\cbr{\bar{f}(x_i)}_{i=1}^{n}$, namely, $v_i = y_i - \bar{f}(x_i),i \in [n].$ We then generate an i.i.d. sequence of Rademacher random variables $\varepsilon = \{\varepsilon_i\}_{i=1}^{n}$. $\cbr{v_i}_{i=1}^{n}$ and $\cbr{\varepsilon_i}_{i=1}^{n}$ are used in future repeated sampling and wild refitting.
\begin{remark}
    If we do not have a pilot predictor $\bar{f}$, we can simply use $\breve{f}$ in its place.
\end{remark}
After the warm up phase, let $K$ denote the number of resampling rounds. In each round $k \in [K]$, we apply a uniform simple sampling without replacement policy $\pi$ on the index set $[n]$ to obtain a subsample $\cS_k \subset [n]$ with $|\cS_k| = m$. By construction, $\PP_{\pi}(i \in \cS_k) = \frac{m}{n}$ for all $i \in [n]$. We use $\cbr{(x^k_i, y^k_i, \varepsilon^k_i, v^k_i)}_{i=1}^{m}$ to denote the dataset $\cbr{(x_j, y_j, \varepsilon_j, v_j)}_{j \in \cS_k}$. 

Then, for some carefully designed scale parameters $\rho_1^k,\rho_2^k>0$, we construct the pseudo-responses on $\cS_k$ as:
\[
\tilde{y}^k_i = \breve{f}(x^k_i) + \rho^k_1 \varepsilon^k_i v^k_i, 
\quad 
\wcheck{y}^k_i = \breve{f}(x^k_i) - \rho^k_2 \varepsilon^k_i v^k_i, 
\quad i \in [m].
\]
Based on these pseudo-responses, we construct two synthetic datasets \[\tilde{\cD}_k = \cbr{(x^k_i, \tilde{y}^k_i)}_{i=1}^{m},\ 
\wcheck{\cD}_k = \cbr{(x^k_i, \wcheck{y}^k_i)}_{i=1}^{m}.\] We then apply the black-box procedure $\Alg$ to these two datasets separately, yielding two wild predictors $\tilde{f}^k_{\rho}$ and $\wcheck{f}^k_{\rho}$. Finally, we proceed to the next round $k+1$.

At a high level, our algorithm employs an interleaved procedure of resampling and wild refitting. In each iteration, we first draw a random subsample from the original training data. Unlike traditional bootstrap methods—which sample with replacement to construct a synthetic dataset of the original size—our approach samples without replacement to generate a strictly smaller subset, with its cardinality explicitly bounded by the available computational budget. This distinction is critical: our objective is not to naively simulate a full retraining pass, but rather to rigorously approximate its statistical behavior under severe computational constraints.

After obtaining the subsample, we construct wild responses and apply a wild refitting procedure to the resulting artificial random datasets in order to produce wild predictors. These predictors are then used to derive a computable upper bound on the excess risk.

Now, we present the pseudo-code of our evaluation algorithm in Algorithm \ref{alg:wild-refitting}. In the next section, we will show that, with the outputs of Algorithm \ref{alg:wild-refitting}, we can efficiently upper bound the excess risk of $\Breve{f}$ returned by the black-box procedure $\Alg$.

\begin{algorithm}
\begin{algorithmic}[1]
\caption{Interleaved Resampling and Wild Refitting}\label{alg:wild-refitting}
\Require Black Box Procedure $\Alg$, dataset $\cD_0=\cbr{(x_i,y_i)}_{i=1}^{n}$, noise scale sequence $\cbr{(\rho^k_1, \rho^k_2)}_{k=1}^{K}$, recentering pilot predictor $\bar{f}$, $K\in \ZZ^+$, resampling size $m=n^{\beta}$.
\State Apply $\Alg$ on $\cD_0$ and get the empirical risk minimizer $\Breve{f}=\Alg(\cD_0)$.
\State Compute the recentering residuals $v_i=y_i-\bar{f}(x_i)$.
\State Sample a Rademacher sequence $\varepsilon=\cbr{\varepsilon_i}_{i=1}^{n}$.
\For{$k=1:K$}
\State Apply simple random sampling without replacement on the set $[n]$, $\PP_{\pi}(i\in\cS_k)=\frac{m}{n}$.
\State Obtain sub-sample $\cS_k=\{(x_i^k,y_i^k)\}_{k=1}^{m}$.
\State Compute the perturbed wild responses $$\tilde{y}_i^k=\Breve{f}(x_i^k)+\rho_1^k\varepsilon_i^kv_i^k,i=1:m;$$
$$\widecheck{y}_i^k=\Breve{f}(x_i^k)-\rho_2^k\varepsilon_i^kv_i^k,i=1:m.$$
\State Construct the pseudo-dataset $\tilde{\cD}_k=\cbr{(x_i^k,\tilde{y}_i^k)}_{i=1}^{m}$, $\widecheck{\cD}_k=\cbr{(x_i^k,\widecheck{y}_i^k)}_{i=1}^{m}$.
\State Retrain the model on $\tilde{\cD}_k$ and $\wcheck{\cD}_k$ to get $\tilde{f}^k_{\rho^k_1}=\Alg(\tilde{\cD}_k)$ and $\wcheck{f}^k_{\rho^k_2}=\Alg(\wcheck{\cD}_k)$.
\EndFor
\State Output $\Breve{f}$, $\tilde{f}^1_{\rho^1_1},\cdots, \tilde{f}^K_{\rho^K_1}$, $\wcheck{f}^1_{\rho^1_2},\cdots,\wcheck{f}^K_{\rho^K_2}$.
\end{algorithmic}
\end{algorithm}
\paragraph{Practicality.} At the end of this section, we give some remarks about the practicality of Algorithm \ref{alg:wild-refitting} and its implementation. 

First, uniform resampling without replacement is a classical operation with many efficient implementations, including a permutation-based method that draws a uniform random permutation and keeps the first several indices \citep{durstenfeld1964algorithm}; a hash table-based method that incrementally builds the sample while storing selected indices in a set to avoid duplicates \citep{bentley1987programming}; and reservoir sampling \citep{vitter1985random}, which is particularly suitable for large-scale data and yields an exact sample using constant memory.

Second, the choice of the parameters $K$ and $\beta$ is inherently determined by the user’s computational budget, in particular by how many refitting rounds the user can afford and how large a subsampled dataset can be used in each retraining step.

Moreover, when the training dataset $\cD$ consists of i.i.d.\ samples, we expect the predictors $\{\tilde{f}_{\rho}^k,\wcheck{f}_{\rho}^k\}$ to behave similarly across different $k\in[K]$ with high probability. It is therefore natural to expect that, for each $i=1,2$, the corresponding tuning parameters $\rho_i^k$ will also be close across $k\in[K]$. In addition, to further reduce computational cost, one may set $\rho_1^k=\rho_2^k$, as demonstrated in Section \ref{sec:experiment} while maintaining the corresponding estimated excess risk bound efficient. As a result, from a numerical perspective, the noise scale needs to be tuned only once throughout the entire evaluation procedure, thereby making Algorithm \ref{alg:wild-refitting} considerably more practical.

\section{Bounding Empirical Excess Risk $\cE_\cD(\Breve{f})$}\label{sec:statistical_guarantee}
In this section, we provide our statistical guarantees on the empirical excess risk $\cE_\cD(\Breve{f})$. Specifically, we will provide our statistical risk bound in the following order.

\paragraph{Roadmap.}  In Subsection \ref{subsec:notation}, we introduce the notation and definitions for the random variables and empirical processes used throughout the analysis. We then present the fixed-design excess-risk bound for \(\cE_\cD(\Breve{f})\), together with a proof sketch in Subsection \ref{subsec:bounding_risk_fixed_design}. In Subsection \ref{subsec:bounding_hatr_n}, we provide a practical interpretation of the bound. A more detailed roadmap for bounding the fixed-design excess risk is given at the beginning of Subsection \ref{subsec:bounding_risk_fixed_design}.

Since we care about the empirical excess risk $\cE_\cD(\Breve{f})$ in this section, let us temporarily view the covariates $\cbr{x_i}_{i=1}^{n}$ as fixed. By the basic inequality, under the well-specified regime, we have that
\[
\|\Breve{f}-f^*\|_{\cD}^2=\frac{1}{n}\sum_{i=1}^{n}(y_i-\Breve{f}(x_i))^2-\frac{1}{n}\sum_{i=1}^{n}(y_i-f^*(x_i))^2+\frac{2}{n}\sum_{i=1}^{n}w_i(\Breve{f}(x_i)-f^*(x_i))\le \frac{2}{n}\sum_{i=1}^{n}w_i(\Breve{f}(x_i)-f^*(x_i)).
\]
The inequality follows from the assumption that $f^*\in\cF$ and $\Breve{f}$ minimize the empirical risk. We denote the term $\frac{2}{n}\sum_{i=1}^{n}w_i(\Breve{f}(x_i)-f^*(x_i))$ as the \emph{true optimism} $\Opt^*_{\cD}(\Breve{f})$. That is to say, in order to bound $\cE_\cD(\Breve{f})$, we only need to bound $\Opt^*_\cD(\Breve{f})$.

\subsection{Notational Preliminaries}\label{subsec:notation}
In this subsection, we introduce the quantities and notation that will be useful for bounding $\Opt^*_\cD(\Breve{f})$. Roughly speaking, we define several intermediate random variables, empirical process terms, and some quantities in harmonic analysis that facilitate the subsequent analysis.

For any $f\in\cF$, we define the quantities $A_n(f)$ and $C_n(f)$ as 
\[
A_n(f):=\frac{1}{n}\sum_{i=1}^{n}\varepsilon_iv_i(f(x_i)-\Breve{f}(x_i)),\ C_n(f):=\frac{1}{n}\sum_{i=1}^{n}\varepsilon_iv_i(\Breve{f}(x_i)-f(x_i)).
\]
where $\varepsilon_i,i\in[n]$ are i.i.d. Rademacher random variables and $v_i=y_i-\bar{f}(x_i), i\in[n]$ are the residuals of the pilot predictor $\bar{f}$. Correspondingly, for the $k$-th subsample $\cS_k$ in Algorithm \ref{alg:wild-refitting} such that $|\cS_k|=m=n^\beta$, recall that we denote $\cbr{x^k_i,y^k_i,\varepsilon^k_i,v^k_i}_{i=1}^{m}$ as $\cbr{x_i,y_i,\varepsilon_i,v_i}_{i\in\cS_k}$. Then, we define the following Horvitz–Thompson weighted averages \citep{horvitz1952generalization}.
\[
B_{\cS_k}(f):=\frac{1}{n}\sum_{i=1}^{n}\frac{\delta_k^i}{\pi_i}\varepsilon_iv_i(f(x_i)-\Breve{f}(x_i))=\frac{1}{m}\sum_{i=1}^{m}\varepsilon_i^kv_i^k(f(x_i^k)-\Breve{f}(x_i^k)),
\]
\[
D_{\cS_k}(f):=\frac{1}{n}\sum_{i=1}^{n}\frac{\delta_k^i}{\pi_i}\varepsilon_iv_i(\Breve{f}(x_i)-f(x_i))=\frac{1}{m}\sum_{i=1}^{m}\varepsilon_i^kv_i^k(\Breve{f}(x_i^k)-f(x_i)),
\]
where $\delta^i_k=\II(i\in\cS_k)$ is the indicator variable representing whether $i$ is selected into $\cS_k$ or not. The core purpose of involving Horvitz–Thompson weighted averages is to ensure unbiased estimators of $A_n(f)$ and $B_n(f)$ under resampling without replacement. Regarding these random variables, we have the following lemma.
\begin{lemma}\label{lemma:unbias}
Taking expectation only with respect to the sampling policy $\pi$, we have
\[
\EE_{\pi}[B_{\cS_k}(f)]=A_n(f),\ \EE_{\pi}[D_{\cS_k}(f)]=C_n(f).
\]
With noises bounded in $[-\tau,\tau]$, for any $\delta>0$, with probability at least $1-\delta$, we have
\begin{align*}
A_n(f)\le \frac{1}{K}\sum_{k=1}^{K}B_{\cS_k}(f)+\frac{8\|f-\Breve{f}\|_\cD\tau\sqrt{\log(K/\delta)}}{\sqrt{K}}, C_n(f)\le \frac{1}{K}\sum_{k=1}^{K}D_{\cS_k}(f)+\frac{8\|f-\Breve{f}\|_\cD\tau\sqrt{\log(K/\delta)}}{\sqrt{K}},
\end{align*}    
\end{lemma}
We now define several intermediate empirical processes that serve a role analogous to Rademacher complexity in classical statistical learning theory. A fundamental distinction, however, is that our proposed processes depend exclusively on the residuals of $\Bar{f}$ (which frequently coincides with $\Breve{f}$). Consequently, in stark contrast to standard Rademacher complexity, these quantities are fully empirical and exactly computable directly from the output of our algorithm.
\[
\cW_{\cD}^\varepsilon(r):=\sup_{f\in\BB_r(\Breve{f};\cD)}A_n(f)=\sup_{f\in\BB_r(\Breve{f};\cD)}\cbr{\frac{1}{n}\sum_{i=1}^{n}\varepsilon_iv_i(f(x_i)-\Breve{f}(x_i))},
\]
\[
\cH_{\cD}^\varepsilon(r):=\sup_{f\in\BB_r(\Breve{f};\cD)}C_n(f)=\sup_{f\in\BB_r(\Breve{f};\cD)}\cbr{\frac{1}{n}\sum_{i=1}^{n}\varepsilon_iv_i(\Breve{f}(x_i)-f(x_i))},
\]
\[
\cT_{\cS_k}^\varepsilon(r):=\sup_{f\in\BB_r(\Breve{f};\cS_k)}\cbr{\frac{1}{m}\sum_{i=1}^{m}\varepsilon_i^kv_i^k(f(x_i^k)-\Breve{f}(x_i^k))},
\]
\[
\cU_{\cS_k}^\varepsilon(r):=\sup_{f\in\BB_r(\Breve{f};\cS_k)}\cbr{\frac{1}{m}\sum_{i=1}^{m}\varepsilon_i^kv_i^k(f(x_i^k)-\Breve{f}(x_i^k))}.
\]
Specifically, these four empirical processes are defined in terms of the pilot noise variables ${v_i}_{i=1}^n$. We refer to $\cW_{\cD}^{\varepsilon}(r)$ and $\cH_{\cD}^{\varepsilon}(r)$ as the full-scale \emph{pilot noise complexities} centered at $\Breve{f}$. Likewise, we refer to $\cT_{\cS_k}^{\varepsilon}(r)$ and $\cU_{\cS_k}^{\varepsilon}(r)$ as the corresponding sub-scale pilot noise complexities.

Parallel to the pilot noise complexity, with the true noises $\cbr{w_i}_{i=1}^{n}$, we apply stochastic symmetrization and define the \emph{true noise complexity} $\cZ_{\cD}^\varepsilon(r)$ as follows:
\[
\cZ_{\cD}^\varepsilon(r):=\sup_{f\in\BB_r(f^*;\cD)}\cbr{\frac{1}{n}\sum_{i=1}^{n}\varepsilon_i\frac{w_i-w_i'}{2}(f(x_i)-f^*(x_i))},
\]
where $w_i'$ is an independent copy of $w_i$. We use $\tilde{w}_i$ to denote $\frac{w_i-w_i'}{2}$.

At the end of this subsection, we introduce some notation in Fourier analysis. Recall that we use $e_k(x)$ to denote the basis function $e^{2\pi \ib kx}$. Denote $\Phi_{\infty}(x)$ and $\Phi_{N}(x)$ as the following two-sided infinite and finite vectors, 
\[
\begin{array}{cc}
\displaystyle
\Phi_{\infty}(x):=
\begin{pmatrix}
\vdots\\
e^{-6\pi\ib x}\\
e^{-4\pi\ib x}\\
e^{-2\pi\ib x}\\
1\\
e^{2\pi\ib x}\\
e^{4\pi\ib x}\\
e^{6\pi\ib x}\\
\vdots
\end{pmatrix}
&
\displaystyle
\Phi_{N}(x):=
\begin{pmatrix}
e^{-2\pi\ib N x}\\
\vdots\\
e^{-4\pi\ib x}\\
e^{-2\pi\ib x}\\
1\\
e^{2\pi\ib x}\\
e^{4\pi\ib x}\\
\vdots\\
e^{2\pi\ib N x}
\end{pmatrix}
\end{array}.
\]
Similarly, for any function $h$, we denote $\hat{h}$ as its Fourier coefficient vector, and let $\hat{h}_N$ denote the truncated Fourier coefficient vector retaining frequencies from $-N$ to $N$, i.e., 
\[
\begin{array}{cc}
\displaystyle
\hat{h}:=
\begin{pmatrix}
\vdots\\
\hat{h}(-3)\\
\hat{h}(-2)\\
\hat{h}(-1)\\
\hat{h}(0)\\
\hat{h}(1)\\
\hat{h}(2)\\
\hat{h}(3)\\
\vdots
\end{pmatrix}
&
\displaystyle
\hat{h}_{N}:=
\begin{pmatrix}
\hat{h}(-N)\\
\vdots\\
\hat{h}(-2)\\
\hat{h}(-1)\\
\hat{h}(0)\\
\hat{h}(1)\\
\hat{h}(2)\\
\vdots\\
\hat{h}(N)
\end{pmatrix}
\end{array}.
\]
With these quantities, in Subsection \ref{subsec:bounding_risk_fixed_design}, we will show that the outputs from Algorithm \ref{alg:wild-refitting} can efficiently upper bound the empirical excess risk $\cE_\cD(\Breve{f})$.
\subsection{Bounding the Empirical Excess Risk $\cE_\cD(\Breve{f})$}\label{subsec:bounding_risk_fixed_design}
In this subsection, we provide our bound on the fixed-design excess risk \(\cE_\cD(\breve{f})\). For clarity of exposition, we begin with a detailed roadmap of the discussion in this subsection.

\paragraph{Roadmap.} We organize our discussion of the fixed-design excess risk \(\cE_{\cD}(\Breve{f})\) as follows. We begin by presenting, together with Lemma \ref{lemma:unbias}, several key lemmas that will be used in the proof of the main excess-risk guarantee in Theorem \ref{thm:risk_bound_fixed_design}. In particular, Lemma \ref{lemma:wild_optmism_bound_empirical_process} will show that the relevant empirical processes can be controlled directly by the outputs of our interleaved resampling-and-refitting procedure in Algorithm \ref{alg:wild-refitting}. We will then turn to the comparison between the empirical norms induced by the full-scale dataset and by the subsamples. Using tools from harmonic analysis and Toeplitz operator theory, we will show that these norms are equivalent in some sense, which allows us to transfer estimates obtained on the subsamples back to the full-scale dataset. Finally, by combining these ingredients with the unbiasedness result in Lemma \ref{lemma:unbias}, we will derive the desired upper bound on the fixed-design excess risk.

In statistical learning theory, obtaining excess risk bounds while controlling certain empirical process terms is a fundamentally difficult issue \citep{vapnik1991principles,rakhlin2022mathstat,vershynin2018high}. The classical approach is to upper bound these quantities via Dudley’s entropy integral. However, such arguments depend crucially on having a tractable description of the function class and its complexity, which is often unavailable for modern black-box models such as deep neural networks. As demonstrated in \citet{wainwright2025wild,hu2025perturbingderivativewildrefitting}, wild refitting offers an alternative route: it allows some of the empirical process terms introduced above to be upper bounded by the corresponding \emph{wild optimism} evaluated at suitable radii.

In particular, for any noise scale $\rho$ and any subsample $\cS_k$, we define the wild optimism quantities $\tilde{\Opt}_{\cS_k}^{\rho}(\tilde{f}^k_{\rho})$ and $\tilde{\Opt}_{\cS_k}^{\rho}(\wcheck{f}^k_\rho)$ regarding the wild predictors $\{\tilde{f}^k_\rho,\wcheck{f}^k_\rho\}_{k=1}^{K}$ as follows:
\[
\tilde{\Opt}_{\cS_k}^{\rho}(\tilde{f}^k_\rho)=\frac{1}{m}\sum_{i=1}^{m}\varepsilon^k_iv^k_i(\tilde{f}^k_{\rho}(x^k_i)-\Breve{f}(x^k_i)),\ \wcheck{\Opt}_{\cS_k}^{\rho}(\wcheck{f}^k_\rho)=\frac{1}{m}\sum_{i=1}^{m}\varepsilon^k_iv^k_i(\check{f}^k_{\rho}(x^k_i)-\Breve{f}(x^k_i)).
\]
Then, we have the following important lemma.
\begin{lemma}\label{lemma:wild_optmism_bound_empirical_process}
    For any noise scale $\rho>0$, we have
    \[
    \cT_{\cS_k}^{\varepsilon}(\|\tilde{f}^k_{\rho}-\Breve{f}\|_{\cS_k})=\tilde{\Opt}_{\cS_k}^{\rho}(\tilde{f}^k_{\rho}),\ \cU_{\cS_k}^{\varepsilon}(\|\wcheck{f}^k_{\rho}-\Breve{f}\|_{\cS_k})=\wcheck{\Opt}_{\cS_k}^{\rho}(\wcheck{f}^k_{\rho})
    \]
\end{lemma}
Intuitively, Lemma \ref{lemma:wild_optmism_bound_empirical_process} shows that the wild predictors $\wcheck{f}^k_{\rho}$ and $\tilde{f}^k_{\rho}$ attain the corresponding suprema of the empirical processes $\cU_{\cS_k}^{\varepsilon}$ and $\cT_{\cS_k}^{\varepsilon}$, respectively. Moreover, the corresponding wild optimism terms are fully computable from Algorithm \ref{alg:wild-refitting}. This computability is crucial, as it allows us to bypass the notoriously difficult task of analyzing function class complexity when deriving excess risk bounds for the black box procedure $\Alg$.

To control the true optimism $\Opt^*_{\cD}(\Breve{f})$, it is necessary to bound the empirical processes $\cW^\varepsilon_{\cD}$ and $\cH^\varepsilon_{\cD}$, both defined through suprema over the full-scale empirical $L^2$ ball $\BB_{r}(\Breve{f};\cD)$ with respect to the pseudo-norm $\|\cdot\|_{\cD}$. In contrast, Lemma~\ref{lemma:wild_optmism_bound_empirical_process} shows that our sub-scale wild refitting procedure only provides control of $\cT^{\varepsilon}_{\cS_k}$ and $\cU^\varepsilon_{\cS_k}$, where the underlying metric is the subsample empirical $L^2$ norm $\|\cdot\|_{\cS_k}$ rather than the full-data norm $\|\cdot\|_{\cD}$. This discrepancy between the two empirical norm structures introduces a second fundamental difficulty in the theoretical analysis.

Therefore, a key step is to establish an appropriate notion of equivalence between the full-scale empirical $L^2$ norm $\|\cdot\|_\cD$ and the subsample empirical norm $\|\cdot\|_{\cS_k}$ uniformly over the relevant function difference class $\DD(\cF)$. Only in this way can the bounds obtained at the subsample scale be transferred back to the full-scale metric.

Thus, to obtain the equivalence between different empirical norms, the following lemma holds uniformly over $h\in\DD(\cF)=\cbr{f-f'|f,f'\in\cF}$.
\begin{lemma}\label{lemma:norm_equivalence}
    For the covariate data $\cbr{x_i}_{i=1}^n$ in $\cD$ and its subsample $\cS=\cbr{x_{k_i}}_{i=1}^{n^{\beta}}$, $\forall \delta>0$, with probability at least $1-2\delta$, for any $N\in\ZZ^+$ such that $\frac{2N\log(2N/\delta)}{n^{\beta}}\le1$, we have
    \begin{align*}
        \|h\|_{\cS}^2\le4\frac{\widebar{w}+3\widebar{w}\sqrt{\frac{2N\log(2N/\delta)}{n^\beta}}}{\underline{w}-3\widebar{w}\sqrt{\frac{2N\log(2N/\delta)}{n}}}\|h\|_{\cD}^2+\rbr{2\frac{\widebar{w}+3\widebar{w}\sqrt{\frac{2N\log(2N/\delta)}{n^\beta}}}{\underline{w}-3\widebar{w}\sqrt{\frac{2N\log(2N/\delta)}{n}}}+1}\frac{8M_v^2}{2v-1}\frac{1}{N^{2v-1}},
    \end{align*}
    uniformly over $h\in\DD(\cF)$.
\end{lemma}
The proof of this lemma relies crucially on techniques from harmonic analysis and Toeplitz operator theory. In contrast to many classical excess risk bounds in statistical learning, which are distribution-free and hence independent of the covariate distribution $\nu$, Lemma \ref{lemma:norm_equivalence} depends explicitly on the density ratio $\widebar{w}/\underline{w}$ associated with $\nu$. The main technical challenge is to establish concentration bounds for the empirical Toeplitz operator $\frac{1}{m}\sum_{i=1}^{m}\Phi_{\infty}(x_i)\Phi_{\infty}(x_i)^H.$
In our setting, the expected operator $\EE_{x\sim\nu}\left[\Phi_{\infty}(x)\Phi_{\infty}(x)^H\right]$ is non-compact, so standard concentration inequalities for sequences of compact operators do not apply directly. To overcome this difficulty, we replace the infinite-dimensional feature map $\Phi_{\infty}$ with a finite-dimensional surrogate $\Phi_N$, and then use the Fourier decay condition in Assumption \ref{ass:func_class} to control the resulting truncation error. Moreover, because Toeplitz operators typically exhibit a continuous spectrum, this truncation argument is fundamentally different from the more familiar eigenvalue-based truncations for compact operators used, for example, in \citet{hu2025contextual}. We defer the full proof to Appendix \ref{app:proofs_sec:statistical_guarantee}.

 Intuitively, when $n$ is large enough, by Lemma \ref{lemma:norm_equivalence}, we can set $N\asymp \frac{n^{\beta}}{\beta\log n}$ to obtain:
\[
\|h\|_{\cS}^2\lesssim C_{\underline{w},\widebar{w},v}^\delta\rbr{\|h\|_{\cD}^2+\frac{M_v^2}{n^{(2v-1)\beta}}},
\]
where $C_{\underline{w},\widebar{w},v}^\delta$ is a constant that is only related to $\underline{w},\widebar{w},v$ and $\delta$. 

One remarkable feature of Lemma \ref{lemma:norm_equivalence} is that it holds uniformly over the function difference class $\DD(\cF)$. In this way, we establish a precise connection between $\|\cdot\|_{\cD}$ and $\|\cdot\|_{\cS}$, which allows us to transfer bounds obtained for sub-scale pilot noise complexities $\cT_{\cS_k}^{\varepsilon}(r)$ and $\cU_{\cS_k}^{\varepsilon}(r)$ into corresponding bounds for the full-scale pilot noise complexities, and thus upper bound the true optimism $\Opt^*_{\cD}(\Breve{f})$.

Now, we are ready to state the first excess risk bound on $\cE_\cD(\Breve{f})$. We use $\hat{r}_n$ to denote the empirical norm between $\hat{f}$ and $f^*$ in the training dataset $\cD$, i.e., $\hat{r}_n:=\|\Breve{f}-f^*\|_\cD$.
\begin{theorem}\label{thm:risk_bound_fixed_design}
Given Algorithm \ref{alg:wild-refitting}, for any $ 0<\delta<1$, when $n$ is sufficiently large such that $n\gg \Omega(\frac{1}{\delta^{1-\beta}})$, for any radius $r$ such that $\|\Breve{f}-f^*\|_{\cD}=\hat{r}_n\le r$, set $\tilde{r}$ to be 
\[
\tilde{r}:=3\sqrt{\frac{\widebar{w}}{\underline{w}}}r+\frac{7\sqrt{\widebar{w}}M_v}{\sqrt{(2v-1)\underline{w}}}\frac{(\log n)^{v-1/2}}{\sqrt{n^{\beta(2v-1)}}}.
\]
Let $\cbr{\rho^k_1,\rho^k_2}_{k=1}^K$ be the noise scale sequence such that $$\|\tilde{f}^k_{\rho^k_1}-\Breve{f}\|_{\cS_k}=\|\wcheck{f}^k_{\rho^k_2}-\Breve{f}\|_{\cS_k}=2\tilde{r}.$$

Then, with probability at least $1-5\delta$, we have
\begin{align}\label{ineq:thm_fixed_design_risk_bound}
    \Opt^*_{\cD}(\Breve{f})\le \frac{1}{K}\sum_{k=1}^{K}\rbr{\tilde{\Opt}_{\cS_k}^{\rho^k_1}(\tilde{f}^k_{\rho^k_1})+\wcheck{\Opt}_{\cS_k}^{\rho^k_2}(\wcheck{f}^k_{\rho^k_2})}+(R(\delta)+V_{2r}(\widebar{f})).
\end{align}
In the bound, we call $R(\delta)$ probability deviation term and $V_{2r}(\widebar{f})$ the pilot error term.

Specifically, the probability deviation term $R(\delta)$ is given by
\[
R(\delta)=r\frac{10\sqrt{2}\tau \sqrt{\log(1/\delta)}}{\sqrt{n}}+\frac{32r\tau\sqrt{\log(K/\delta)}}{\sqrt{K}}
\]
and the pilot error term $V(\widebar{f})$ is 
\begin{align*}V_{2r}(\bar{f})&=\sup_{f\in\BB_{2r}(\Breve{f};\cD)}\cbr{\frac{1}{n}\sum_{i=1}^{n}\varepsilon_i(\widebar{f}(x_i)-f^*(x_i))(f(x_i)-\Breve{f}(x_i))}\\
        &+\sup_{f\in\BB_{2r}(\Breve{f};\cD)}\cbr{\frac{1}{n}\sum_{i=1}^{n}\varepsilon_i(\widebar{f}(x_i)-f^*(x_i))(\Breve{f}(x_i)-f(x_i))}.
        \end{align*}
\end{theorem}
We begin with a proof sketch of Theorem \ref{thm:risk_bound_fixed_design}, while deferring the complete proof to Appendix \ref{app:proofs_sec:statistical_guarantee}. We then provide further explanation of the main ideas underlying the result.

\paragraph{Proof Sketch of Theorem \ref{thm:risk_bound_fixed_design}}
Consider the true optimism
\[
\Opt^*_{\cD}(\Breve{f})=\frac{1}{n}\sum_{i=1}^{n}\big(\Breve{f}(x_i)-f^*(x_i)\big).
\]
Assume that we find and fix a radius $r$ such that $r\ge \|\Breve{f}-f^*\|_{\cD}$. Then, $\forall\delta>0$, by applying Lemma \ref{lemma:concentration_Lip} and the Rademacher symmetrization, with probability at least $1-\delta$, we have
\[
\Opt^*(\hat{f})\lesssim \rbr{\EE[\cZ_\cD^\varepsilon(r)]+\frac{r\tau \sqrt{\log(1/\delta)}}{\sqrt{n}}}.
\]

Next, to control the empirical process $\cZ_\cD^\varepsilon(r)$, we notice that for any $i$, we have
\[
\varepsilon_iw_i=\varepsilon_iv_i+\varepsilon_i(\widebar{f}(x_i)-f^*(x_i)).
\]
Therefore, for any $r\ge\hat{r}_n$, we can further upper bound $\EE[\cZ_\cD^\varepsilon(r)]$ via the sum of $\cW_\cD^\varepsilon(2r)$ and $\cH_{\cD}^\varepsilon(2r)$, along with the pilot error term $V_{2r}(\widebar{f})$.

After that, we use Lemma \ref{lemma:norm_equivalence} to conclude that the functions that achieve the supremum of $\cW_\cD^\varepsilon(2r)$ and $\cH_{\cD}^\varepsilon(2r)$, denoted by $h_1$ and $h_2$, lie in the sub-scale empirical norm ball $\BB_{2\tilde{r}}(\Breve{f};\cS_k)$ uniformly over $k\in[K]$ with high probability, where
\[
(2\tilde{r})^2=C_{\underline{w},\widebar{w},v}^\delta\rbr{(2r)^2+\frac{M_v^2}{n^{(2v-1)\beta}}}.
\]
Then, utilizing Lemma \ref{lemma:unbias}, with high probability, for $\cW_\cD^{\varepsilon}(2r)$, we have
\begin{align*}
    \cW_\cD^\varepsilon(2r)
    =&\frac{1}{n}\sum_{i=1}^{n}\varepsilon_iv_i(h_1(x_i)-\Breve{f}(x_i))\le \frac{1}{K}\sum_{k=1}^{K}B_{\cS_k}(h_1)+\frac{16\tau\sqrt{\log(K/\delta)}}{\sqrt{K}}
    \le\frac{1}{K}\sum_{k=1}^{K}\cT^\varepsilon_{\cS_k}(2\tilde{r})+\frac{16r\tau\sqrt{\log(K/\delta)}}{\sqrt{K}}.
\end{align*}
In the last inequality, we use the result that $h_1\in\BB_{2\tilde{r}}(\Breve{f};\cS_k)$ according to Lemma \ref{lemma:norm_equivalence}. Using the same argument, we have
\[
\cH_\cD^\varepsilon(2r)\le \frac{1}{K}\sum_{k=1}^{K}\cU^\varepsilon_\cD(2\tilde{r})+\frac{16r\tau\sqrt{\log(K/\delta)}}{\sqrt{K}}.
\]
Finally, with the noise scale condition in Theorem \ref{thm:risk_bound_fixed_design}, we use Lemma \ref{lemma:wild_optmism_bound_empirical_process} to obtain
\[
\cT^\varepsilon_{\cS_k}(2\tilde{r})=\tilde{\Opt}^{\rho_1^k}_{\cS_k}(\tilde{f}^k_{\rho^k_1}),
\ \text{and}\ 
\cU^\varepsilon_\cD(2\tilde{r})=\wcheck{\Opt}^{\rho_2^k}_{\cS_k}(\wcheck{f}^{k}_{\rho_2^k}).
\]
Combining these parts together, we obtain
\[
\Opt^*\cD(\Breve{f})\le \frac{1}{K}\sum_{k=1}^{K}\rbr{\tilde{\Opt}^{\rho_1^k}_{\cS_k}(\tilde{f}^k_{\rho^k_1})+\wcheck{\Opt}^{\rho_2^k}_{\cS_k}(\wcheck{f}^{k}_{\rho_2^k})}+R(\delta)+V_{2r}(\widebar{f}),
\]
which completes our proof.
\paragraph{Theorem \ref{thm:risk_bound_fixed_design} Explanation}

Now, we provide some comments regarding Theorem \ref{thm:risk_bound_fixed_design}. First, in order to make the conclusion valid, we assume that $n$ is large enough. Specifically, for any $\delta$ and $K$, we just need $n$ to be large enough such that
\begin{align*}\label{ineq:n_size}
\widebar{w}\sqrt{\frac{2N\log(2N/\delta)}{n^\beta}}\le \frac{1}{8}\underline{w},\ N=\frac{n^{\beta}}{\log n}.
\end{align*}
This condition is satisfied as long as $n\gtrsim C'_{v,\underline{w}, \widebar{w}}\frac{1}{\delta^{1-\beta}}$ for some constant $C'_{v,\underline{w}, \widebar{w}}$. Under this condition, we can verify that
\begin{align*}
\frac{\widebar{w}+3\widebar{w}\sqrt{\frac{2N\log(2N/\delta)}{n^\beta}}}{\underline{w}-3\widebar{w}\sqrt{\frac{2N\log(2N/\delta)}{n}}}\le\frac{\widebar{w}+\frac{3}{8}\widebar{w}}{1-\frac{3}{8}\underline{w}}=\frac{11\widebar{w}
}{5\underline{w}}.
\end{align*}
Consequently, by Lemma \ref{lemma:norm_equivalence}, we get
\[\|\Breve{f}-f^*\|_{\cS_k}\le \sqrt{\frac{9\widebar{w}}{\underline{w}}\hat{r}_n^2+\frac{48M_v^2}{2v-1}\frac{(\log n)^{2v-1}}{n^{\beta(2v-1)}}}\le 3\sqrt{\frac{\widebar{w}}{\underline{w}}}\hat{r}_n+\frac{7\sqrt{\widebar{w}}M_v}{\sqrt{(2v-1)\underline{w}}}\frac{(\log n)^{v-1/2}}{\sqrt{n^{\beta(2v-1)}}},\]
which is exactly upper bounded by the radius we set in Theorem \ref{thm:risk_bound_fixed_design}.

To interpret our upper bound, note that Inequality \ref{ineq:thm_fixed_design_risk_bound} shows that the true optimism can be upper bounded by the wild optimism for appropriately chosen noise scales, along with the probability deviation term $R(\delta)$ and the pilot error term $V_{2r}(\widebar{f})$.

Intuitively, the excess risk bound consists of three components. The first is the dominant \emph{wild optimism} term,
\[
\frac{1}{K}\sum_{k=1}^{K}\Big\{\tilde{\Opt}_{\cS_k}^{\rho^k_1}\big(\tilde{f}^k_{\rho^k_1}\big)+\wcheck{\Opt}_{\cS_k}^{\rho^k_2}\big(\wcheck{f}^k_{\rho^k_2}\big)\Big\},
\]
which is computed by our interleaved resampling and refitting procedure.

The second component is the \emph{probabilistic deviation term} $R(\delta)$. More precisely, $R(\delta)$ decomposes into two contributions. The first component is $r\,\frac{10\sqrt{2}\tau \sqrt{\log(1/\delta)}}{\sqrt{n}}$, which arises from the concentration arguments in the refitting analysis when upper bounding $\Opt^*_{\cD}(\Breve{f})$ in terms of $\cW_\cD^\varepsilon$ and $\cH_\cD^\varepsilon$. The second one is $\frac{32r\tau\sqrt{\log(K/\delta)}}{\sqrt{K}}$
, which captures the sampling error introduced by the resampling step across $K$ rounds.

For the \emph{pilot error term} $V_{2r}(\bar{f})$, we compare it with the sum of $\cW_\cD^\varepsilon(2r)+\cH_\cD^\varepsilon(2r)$. By definition, we have
\begin{align*}\cW_\cD^\varepsilon(2r)+\cH_\cD^\varepsilon(2r)&=\sup_{f\in\BB_{2r}(\Breve{f};\cD)}\cbr{\frac{1}{n}\sum_{i=1}^{n}\varepsilon_i(\widebar{f}(x_i)-f^*(x_i)-w_i)(\Breve{f}(x_i)-f(x_i))}\\
        &+\sup_{f\in\BB_{2r}(\Breve{f};\cD)}\cbr{\frac{1}{n}\sum_{i=1}^{n}\varepsilon_i(\widebar{f}(x_i)-f^*(x_i)-w_i)(f(x_i)-\Breve{f}(x_i))}.
\end{align*}
Therefore, the pilot error term takes the same form, except that the noisier quantity $\widebar{f}(x_i)-f^*(x_i)-w_i$ is replaced by the more informative discrepancy $\widebar{f}(x_i)-f^*(x_i)$. Consequently, as pointed out by \citet{wainwright2025wild}, one can expect the pilot error to be controlled by $\cW_\cD^\varepsilon(2r)+\cH_\cD^\varepsilon(2r)$, and hence to be dominated by the wild optimism term, as long as $\widebar{f}$ is a reasonably accurate approximation of $f^*$.

More importantly, the following two theorems show that, when $\cF$ is rich enough, $V_{2r}(\widebar{f})$ can be rigorously upper bounded by the wild-optimism term, up to an additional probability-deviation term, in a manner consistent with the intuition above. In particular, as discussed in Section \ref{sec:model_setup}, we focus on the regime $1/2 \le v \le 1$, since in smoother settings (for example, when $v>1$), the excess risk can typically be controlled directly through bounds on the Rademacher complexity of the model class.
\begin{theorem}\label{thm:bound_pilot_error}
    Suppose that $\cF=\cbr{f\in C([0,1])\big|v\in(\frac{1}{2},1), M_v>0, \text{s.t.} \forall k, |k|^v|\hat{f}(k)|\le M_v}$. Assuming that $x_i\neq x_j,\forall i\neq j$ and $|\hat{f}(k)||k|^v<M_v$. Then, for any $\delta>0$, with probability at least $1-2\delta$, we have
        \[
        V_{2r}(\bar{f})\le \cW_\cD^\varepsilon(2r)+\cH_\cD^\varepsilon(2r)+\frac{8r\tau\sqrt{\log(1/\delta)}}{\sqrt{n}};
        \]
\end{theorem}
\begin{theorem}\label{thm:bound_pilot_error_v=1}
   Suppose that $\cF=\cbr{f\in C([0,1])\big|\forall k, |k||\hat{f}(k)|\le M_1}$. Assuming that $x_i\neq x_j,\forall i\neq j$, and $|\hat{f}(k)|=o(\frac{1}{|k|})$, $|\hat{f}(k)|<\frac{M_1}{|k|}$. Then, $\forall\delta>0$, with probability at least $1-2\delta$,
    \[
        V_{2r}(\bar{f})\le \cW_\cD^\varepsilon(2r)+\cH_\cD^\varepsilon(2r)+\frac{8r\tau\sqrt{\log(1/\delta)}}{\sqrt{n}};
        \]
\end{theorem}
With Theorems \ref{thm:bound_pilot_error} and \ref{thm:bound_pilot_error_v=1} in hand, we can upper bound the pilot-error term $V_{2r}(\widebar{f})$ by the sum $\cW_{\cD}^\varepsilon(2r)+\cH_{\cD}^\varepsilon(2r)$. Then, by Lemmas \ref{lemma:unbias} and \ref{lemma:wild_optmism_bound_empirical_process}, these two quantities can, in turn, be bounded by their corresponding wild-optimism terms through the same comparison argument used in the proof sketch of Theorem \ref{thm:risk_bound_fixed_design}. Intuitively, we conclude that when the function class $\cF$ is sufficiently rich, the pilot error is rigorously dominated by the wild-optimism term.

As for the proofs, Theorems \ref{thm:bound_pilot_error} and \ref{thm:bound_pilot_error_v=1} rely on interpolation properties specific to the underlying function class, and the proof techniques depend sensitively on the value of $v$. In particular, the arguments for different regimes of $v$ are qualitatively distinct. Deriving bounds for the pilot-error term is not the main focus of this paper, we present only two illustrative theorems here and leave the study of pilot-error bounds for other function classes to future work.

More generally, a similar intuition should apply to other black-box deep neural networks. If the underlying function class
$\cF$ is close to a universal approximation family, which is the case for many commonly used deep neural network architectures under suitable structural conditions, including standard feedforward networks, ReLU networks of sufficient width or depth, deep convolutional networks, residual networks, and Transformers with positional encodings \citep{funahashi1989approximate,hornik1991approximation,hanin2017approximating,yarotsky2017error,zhou2020universality,hu2025universal,yun2019transformers}. Then, it is highly plausible that the pilot error term can also be rigorously shown to be upper bounded by the wild-optimism term.

\subsection{Practical Interleaved Resampling and Refitting by Bounding the Radius $\hat{r}_n$}\label{subsec:bounding_hatr_n}
To apply Theorem \ref{thm:risk_bound_fixed_design} in practice, it is crucial to know at least a feasible radius $r$, that is, a crude upper bound on $\hat{r}_n$. The reason is that the condition $r \ge \hat{r}_n$ is needed to determine the appropriate noise scales. In this subsection, we present a practical method for obtaining a valid upper bound on $\hat{r}_n$, and then introduce an adjusted version of Algorithm \ref{alg:wild-refitting} suitable for practical implementation.

Specifically, the theoretical foundation of our practical approach is the following theorem.
\begin{theorem}\label{thm:bounding_hatr_n}
    Assuming that $\cF$ is convex, for any $t>0$, if $\hat{r}_n>\frac{t^2}{\sqrt{n}}$, with probability at least $1-4e^{-t^2}$, we have
    \[
\hat{r}_n^2\le \cW_\cD^\varepsilon([2+\frac{1}{t}]\hat{r}_n)+\cH_{\cD}^\varepsilon([2+\frac{1}{t}]\hat{r}_n)+[2+\frac{1}{t}]\hat{r}_n\frac{4\sqrt{2}\tau t}{\sqrt{n}}+V_{[2+\frac{1}{t}]\hat{r}_n}(\bar{f})+4\frac{\tau}{t}\hat{r}_n^2+\frac{2\hat{r}_n\tau t}{\sqrt{n}}.
    \]
\end{theorem}
The proof of Theorem \ref{thm:bounding_hatr_n} is deferred to  Appendix \ref{app:proofs_sec:statistical_guarantee}. 

We now interpret this theoretical result and transform it into a practical approach about finding an upper bound for $\hat{r}_n$. First, based on our previous discussion on the pilot error term, we know that with probability at least $1-4e^{-t^2}$,
\[
V_{[2+\frac{1}{t}]\hat{r}_n}(\widebar{f})\lesssim \cW_\cD^\varepsilon([2+\frac{1}{t}]\hat{r}_n)+\cH_\cD^\varepsilon([2+\frac{1}{t}]\hat{r}_n)+C\frac{[2+\frac{1}{t}]\hat{r}_n\tau t}{\sqrt{n}},
\]
for some constant $C$.

Then, we invoke the following lemma.
\begin{lemma}\label{lemma:concave_W_H}    
If $\cF$ is convex, $\cW_\cD^\varepsilon(r)$ and $\cH_\cD^\varepsilon(r)$ are concave in $r$. With $\cW_\cD^\varepsilon(0)=\cH_\cD^\varepsilon(0)=0$, we have that
\[
\frac{\cW_\cD^\varepsilon(s)}{s}\le \frac{\cW_\cD^\varepsilon(t)}{t},\ \frac{\cH_\cD^\varepsilon(s)}{s}\le \frac{\cH_\cD^\varepsilon(t)}{t},\ \forall s\ge t>0.
\]
\end{lemma}
Then, for any $0<K_1<K$ denoting $r^\dia_\rho$ and $r^\sh$ to be
\[
r^\dia:=\frac{1}{K_1}\sum_{k\in[K_1]}\|\tilde{f}^k_{\rho_1^k}-\Breve{f}\|_{\cS_k},\  r^\sh:=\frac{1}{K_1}\sum_{k\in[K_1]}\|\wcheck{f}^k_{\rho_2^k}-\Breve{f}\|_{\cS_k}.
\]
The analysis based on these two terms are symmetric, we focus on $r^\dia$. We either have $\hat{r}_n\le r^\dia$ or $\hat{r}_n>r^\dia$, in the latter case, by Lemma \ref{lemma:concave_W_H}, we have
\[
\cW_{\cD}^\varepsilon([2+\frac{1}{t}]\hat{r}_n)\le \hat{r}_n\frac{\cW_\cD^\varepsilon([2+\frac{1}{t}]r^\dia)}{r^\dia}.
\]
Similarly, if $\hat{r}_n>r^\sh$, we have
\[
\cH_{\cD}^\varepsilon([2+\frac{1}{t}]\hat{r}_n)\le \hat{r}_n\frac{\cW_\cD^\varepsilon([2+\frac{1}{t}]r^\sh)}{r^\sh}
\]
Plugging these two inequalities to the conclusion of Theorem \ref{thm:bounding_hatr_n}, and combining with the discussion about the pilot error term, we obtain the following claim.
\paragraph{Claim:}For any $t\ge 3$, if $\hat{r}_n>\frac{t^2}{\sqrt{n}}$, $\hat{r}_n> \frac{1}{K_1}\sum_{k\in[K_1]}\|\tilde{f}^k_{\rho_1^k}-\Breve{f}\|_{\cS_k}$, and $\hat{r}_n>\frac{1}{K_1}\sum_{k\in[K_1]}\|\wcheck{f}^k_{\rho_2^k}-\Breve{f}\|_{\cS_k}$, then with probability at least $1-8e^{-t^2}$, we have
\[
\hat{r}_n^2\le 2\rbr{\hat{r}_n\frac{\cW_\cD^\varepsilon([2+\frac{1}{t}]r^\dia)}{r^\dia}+\hat{r}_n\frac{\cH_\cD^\varepsilon([2+\frac{1}{t}]r^\sh)}{r^\sh}}+[2+\frac{1}{t}]\hat{r}_n\frac{4\sqrt{2}\tau t}{\sqrt{n}}+\frac{\tau}{t}\hat{r}_n^2+\frac{2\hat{r}_n\tau t}{\sqrt{n}}+C\frac{[2+\frac{1}{t}]\hat{r}_n\tau t}{\sqrt{n}}.
\]
We set $t$ such that $\frac{4\tau}{t}<1$ to obtain:
\begin{equation}\label{ineq:bound_hatr_n}
(1-\frac{4\tau}{t})\hat{r}_n\le 2\rbr{\frac{\cW_\cD^\varepsilon([2+\frac{1}{t}]r^\dia)}{r^\dia}+\frac{\cH_\cD^\varepsilon([2+\frac{1}{t}]r^\sh)}{r^\sh}}+[2+\frac{1}{t}]\frac{4\sqrt{2}\tau t}{\sqrt{n}}+\frac{2\tau t}{\sqrt{n}}+C\frac{[2+\frac{1}{t}]\tau t}{\sqrt{n}}.
\end{equation}
By definition, $\cW_{\cD}^\varepsilon([2+\frac{1}{t}]r^\dia)/r^\dia$ and $\cH_{\cD}^\varepsilon([2+\frac{1}{t}]r^\sh)/r^\sh$ are only related to the following random variables and functions below:
$$\Breve{f}, \cbr{\varepsilon_i}_{i=1}^{n}, \cbr{v_i}_{i=1}^{n}, \cbr{\tilde{f}^k_{\rho^k_1},\wcheck{f}^k_{\rho^k_2}}_{k=1}^{K_1}.$$
Thus, the realized values of these two empirical processes are entirely computable from the first $K_1$ rounds of interleaved resampling and refitting.

Therefore, combining all these cases together, we have that for any $t>4\tau$, with probability $1-8e^{-t^2}$,
\begin{align}\label{ineq:bound_hatr_n_final}
(1-\frac{4\tau}{t})\hat{r}_n\le& \max\cbr{\frac{t^2}{\sqrt{n}},\frac{\sum_{k=1}^{K_1}\|\tilde{f}^k_{\rho_1^k}-\Breve{f}\|_{\cS_k}}{K_1},\frac{\sum_{k=1}^{K_1}\|\wcheck{f}^k_{\rho_1^k}-\Breve{f}\|_{\cS_k}}{K_1},2\rbr{\frac{\cW_\cD^\varepsilon([2+\frac{1}{t}]r^\dia)}{r^\dia}+\frac{\cH_\cD^\varepsilon([2+\frac{1}{t}]r^\sh)}{r^\sh}}}\nonumber\\ 
+&[2+\frac{1}{t}]\frac{4\sqrt{2}\tau t}{\sqrt{n}}+\frac{2\tau t}{\sqrt{n}}+C\frac{[2+\frac{1}{t}]\tau t}{\sqrt{n}}.
\end{align}
Therefore, we obtain a valid high probability upper bound for $\hat{r}_n$.
With this high probability upper bound, we propose a practical adaptation to our evaluation algorithm. Specifically, we partition the total $K$ resampling rounds by reserving a small initial subset, $K_1<K$. These initial $K_1$ rounds are dedicated exclusively to establishing a high-probability upper bound for $\hat{r}_n$ via Inequality (\ref{ineq:bound_hatr_n_final}). Equipped with this valid upper bound, we then proceed to execute the core interleaved resampling and refitting procedure as detailed in Algorithm \ref{alg:wild-refitting}. For completeness, we present the pseudo-code of our adjusted interleaved resampling and refitting algorithm in Algorithm \ref{alg:wild-refitting_2}. The theoretical guarantees for Algorithm \ref{alg:wild-refitting_2} remain almost unchanged, except that $K$ is replaced by $K-K_1$, which constitutes only a minor adjustment and we omit the restatements here.

Finally, at the end of this section, we give a remark regarding the loss functions in our algorithm. In this paper, we only consider the empirical risk minimization with the square loss function. For general convex losses, by using an adjusted perturbation procedure proposed by \citet{hu2025doublywildrefittingmodelfree}, we can derive a similar interleaved resampling and refitting algorithm that fits general convex losses.

Therefore, we now have the theoretical guarantee regarding $\cE_\cD(\Breve{f})$ and a practical approach to upper bound it. In the next section, we will show how to transform our bound on $\cE_\cD(\Breve{f})$ to the true excess risk bound $\cE(\Breve{f})$.
\begin{algorithm}
\begin{algorithmic}[1]
\caption{Practical Interleaved Resampling and Wild Refitting}\label{alg:wild-refitting_2}
\Require Black Box Procedure $\Alg$, dataset $\cD_0=\cbr{(x_i,y_i)}_{i=1}^{n}$, noise scale sequence $\cbr{(\rho^k_1, \rho^k_2)}_{k=1}^{K_1}$, recentering pilot predictor $\bar{f}$, $K\in \ZZ^+$, resampling size $m=n^{\beta}$.
\State Apply $\Alg$ on $\cD_0$ and get the empirical risk minimizer $\Breve{f}=\Alg(\cD_0)$.
\State Compute the residuals $v_i=y_i-\bar{f}(x_i)$ and sample a Rademacher sequence $\varepsilon=\cbr{\varepsilon_i}_{i=1}^{n}$.
\For{$k=1:K_1$}
\State Randomly sample without replacement on $[n]$ to obtain $\cS_k=\{(x_i^k,y_i^k)\}_{k=1}^{m}$.
\State Compute the perturbed wild responses $$\tilde{y}_i^k=\Breve{f}(x_i^k)+\rho_1^k\varepsilon_i^kv_i^k,\ \widecheck{y}_i^k=\Breve{f}(x_i^k)-\rho_2^k\varepsilon_i^kv_i^k,\ i=1:m.$$
\State Construct the pseudo-dataset $\tilde{\cD}_k=\cbr{(x_i^k,\tilde{y}_i^k)}_{i=1}^{m}$, $\widecheck{\cD}_k=\cbr{(x_i^k,\widecheck{y}_i^k)}_{i=1}^{m}$.
\State Retrain the model on $\tilde{\cD}_k$ and $\wcheck{\cD}_k$ to get $\tilde{f}^k_{\rho^k_1}=\Alg(\tilde{\cD}_k)$ and $\wcheck{f}^k_{\rho^k_2}=\Alg(\wcheck{\cD}_k)$.
\EndFor
\State Apply Inequality (\ref{ineq:bound_hatr_n_final}) to obtain an upper bound $r$ on $\hat{r}_n$:
\State \begin{align*}
(1-\frac{4\tau}{t})r=& \max\{\frac{t^2}{\sqrt{n}},\frac{\sum_{k=1}^{K_1}\|\tilde{f}^k_{\rho_1^k}-\Breve{f}\|_{\cS_k}}{K_1},\frac{\sum_{k=1}^{K_1}\|\wcheck{f}^k_{\rho_1^k}-\Breve{f}\|_{\cS_k}}{K_1},2(\frac{\cW_\cD^\varepsilon([2+\frac{1}{t}]r^\dia)}{r^\dia}+\frac{\cH_\cD^\varepsilon([2+\frac{1}{t}]r^\sh)}{r^\sh})\}\nonumber\\ 
+&[2+\frac{1}{t}]\frac{4\sqrt{2}\tau t}{\sqrt{n}}+\frac{2\tau t}{\sqrt{n}}+C\frac{[2+\frac{1}{t}]\tau t}{\sqrt{n}}.
\end{align*}
\For{$k=K_1+1:K$}
\State Randomly sample without replacement on $[n]$ to obtain $\cS_k=\{(x_i^k,y_i^k)\}_{k=1}^{m}$.
\State Initialize $\rho^k_{1,1}$ and $\rho^k_{2,1}$, $j=1$.
\While{$\|\tilde{f}^k_{\rho^k_{1,j}}-\Breve{f}\|_{\cS_k}$ and $\|\wcheck{f}^k_{\rho^k_{2,j}}-\Breve{f}\|_{\cS_k}$ smaller than $r$}
\State Update $j=j+1$. Tune the noise scales from $\{\rho^k_{1,j},\rho^k_{2,j}\}$ to $\{\rho^k_{1,j+1},\rho^k_{2,j+1}\}$.
\State Compute the perturbed wild responses $$\tilde{y}_{i,j+1}^k=\Breve{f}(x_i^k)-\rho_{1,j+1}^k\varepsilon_i^kv_i^k;\ \widecheck{y}_{i,j+1}^k=\Breve{f}(x_i^k)+\rho_{2,j+1}^k\varepsilon_i^kv_i^k,\ i=1:m.$$
\State Construct the pseudo-dataset $$\tilde{\cD}_k^{j+1}=\{(x_i^k,\tilde{y}_{i,j+1}^k)\}_{i=1}^{m},\  \widecheck{\cD}_k^{j+1}=\{(x_i^k,\widecheck{y}_{i,j+1}^k)\}_{i=1}^{m}.$$
\State Retrain the model on $\tilde{\cD}_k^{j+1}$ and $\wcheck{\cD}_k^{j+1}$ to get $$\tilde{f}^k_{\rho^k_{1,j+1}}=\Alg(\tilde{\cD}_k^{j+1});\ \  \wcheck{f}^k_{\rho^k_{2j+1}}=\Alg(\wcheck{\cD}_k^{j+1}).$$
\EndWhile
\State $\tilde{f}^k_{\rho^k_1}\leftarrow\tilde{f}^k_{\rho^k_{1,j}},\wcheck{f}^k_{\rho^k_1}\leftarrow\wcheck{f}^k_{\rho^k_{2,j}}$ 
\EndFor
\State Output $\Breve{f}$, $\tilde{f}^1_{\rho^1_1},\cdots, \tilde{f}^K_{\rho^K_1}$, $\wcheck{f}^1_{\rho^1_2},\cdots,\wcheck{f}^K_{\rho^K_2}$.
\end{algorithmic}
\end{algorithm}

\section{Statistical Guarantee on Excess risk $\cE(\Breve{f})$}\label{sec:guarantee_random_design}
In this section, we upper bound the excess risk $\cE(\Breve{f})$ by generalizing an existing bound on the empirical excess risk $\cE_\cD(\Breve{f})$ in Section \ref{sec:statistical_guarantee}. First, we have the following lemma saying that we can efficiently transfer our statistical guarantees on $\cE_\cD(\Breve{f})$ to $\cE(\Breve{f})$. 
\begin{lemma}\label{lemma:critical_radius}
    Denote $\DD(\cF)^2$ to be the function class $\cbr{(f-f')^2:f,f'\in\cF}$. For any $\delta>0$, there exist constants $\tilde{C}_v$ such that with probability at least $1-\delta$,
$$\EE_{x\sim\nu}[(f(x)-f'(x))^2]\le \frac{4\widebar{w}}{\underline{w}}\|f-f'\|_{\cD}^2+\tilde{C}_{v}\frac{\widebar{w}}{\underline{w}}\frac{\log n(\log\log n)\log(1/\delta)}{n^{1-\frac{1}{2v}}},$$
uniformly over $f-f'\in\DD(\cF)$.
\end{lemma}
Classical learning theory results \citep{Bousquet2002,bousquet2002bennett} show that such a generalization typically hinges on the \emph{critical radius} (or localized complexity) of the underlying function class $\cF$. In our setting, however, the model class of interest is sufficiently intricate that obtaining a sharp upper bound on its critical radius directly becomes intractable. To circumvent this difficulty, we introduce a surrogate model class that approximates $\cF$, chosen to balance two competing effects: (i) a controllable critical radius for the surrogate class and (ii) a provably small approximation error relative to the original class. The full proof is deferred to Appendix \ref{app:proofs_sec:guarantee_random_design}.

Combining Lemma \ref{lemma:critical_radius} with Theorem \ref{thm:risk_bound_fixed_design}, we achieve the following theorem on the upper bound regarding the excess risk $\cE(\Breve{f})$, which is the final statistical guaranty concerning the excess risk bound. 
\begin{theorem}\label{thm:excess_risk_random_design}
    Assuming that $\cF\subset L^2([0,1])$ and $\forall k\neq 0$, $|k|^v|\hat{f}(k)|\le M_v$ for some $v>1/2$. The black-box procedure $\mathtt{Alg}$ solves the ERM problem by $\hat{f}$.
    $$\hat{f}=\argmin_{f\in\cF}\frac{1}{n}\sum_{i=1}^{n}(y_i-\Breve{f}(x_i))^2.$$
    $\forall 0<\delta<1$, for some $C'_{v,\widebar{w},\underline{w}}$, when $n\ge \frac{C'_{v,\widebar{w},\underline{w}}}{\delta^{1-\beta}}$ is large enough, for any radius $r$ such that $r\ge \hat{r}_n=\|\Breve{f}-f^*\|_{\cD}$, set $\tilde{r}$ to be 
\[
\tilde{r}:=3\frac{\widebar{w}}{\underline{w}}r+\frac{7M_v}{\sqrt{2v-1}}\frac{(\log n)^{v-1/2}}{\sqrt{n^{\beta(2v-1)}}}.
\]
Let $\cbr{\rho^k_1,\rho^k_2}_{k=1}^K$ be the noise scale sequence such that $$\|\tilde{f}^k_{\rho^k_1}-\Breve{f}\|_{\cS_k}=\|\wcheck{f}^k_{\rho^k_2}-\Breve{f}\|_{\cS_k}=2\tilde{r},\forall k\in[K].$$
Then, with probability at least $1-6\delta$, we have
\begin{align*}
    \cE(\Breve{f})\le& \frac{4\widebar{w}}{\underline{w}}\rbr{\frac{1}{K}\sum_{k=1}^{K}\rbr{\tilde{\Opt}_{\cS_k}^{\rho^k_1}(\tilde{f}^k_{\rho^k_1})+\wcheck{\Opt}_{\cS_k}^{\rho^k_2}(\wcheck{f}^k_{\rho^k_2})}+V_{2r}(\widebar{f})}+\frac{4\widebar{w}}{\underline{w}} \rbr{r\frac{10\sqrt{2}\tau \sqrt{\log(1/\delta)}}{\sqrt{n}}+\frac{32r\tau\sqrt{\log(K/\delta)}}{\sqrt{K}}}\\
    +&\tilde{C}_{v}\frac{\widebar{w}}{\underline{w}}\frac{\log n(\log\log n)\log(1/\delta)}{n^{1-\frac{1}{2v}}}
\end{align*}
for some $\tilde{C}_v$ that is only related to $v$. In this bound, we have
    \begin{align*}V_{2r}(\bar{f})&=\sup_{f\in\BB_{2r}(\Breve{f};\cD)}\cbr{\frac{1}{n}\sum_{i=1}^{n}\varepsilon_i(\widebar{f}(x_i)-f^*(x_i))(f(x_i)-\Breve{f}(x_i))}\\
        &+\sup_{f\in\BB_{2r}(\Breve{f};\cD)}\cbr{\frac{1}{n}\sum_{i=1}^{n}\varepsilon_i(\widebar{f}(x_i)-f^*(x_i))(\Breve{f}(x_i)-f(x_i))}.
        \end{align*}
        $V_{2r}(\bar{f})$ is the pilot error and can be efficiently controlled by Theorem \ref{thm:bound_pilot_error} and Theorem
        \ref{thm:bound_pilot_error_v=1}.
\end{theorem}
\begin{proof}[Proof of Theorem \ref{thm:excess_risk_random_design}]
     Since $\hat{f}\in\cF$, we first apply Lemma \ref{lemma:critical_radius} to get that for any $\delta>0$
    \[
    \cE(\Breve{f})=\int_{0}^{1}|\Breve{f}(x)-f^*(x)|^2d\nu(x)\le \frac{4\widebar{w}}{\underline{w}}\|\Breve{f}-f^*\|_{\cD}^2+\tilde{C}_{v}\frac{\widebar{w}}{\underline{w}}\frac{\log n(\log\log n)\log(1/\delta)}{n^{1-\frac{1}{2v}}}
    \]
    with probability at least $1-\delta$.
    
    Note that by the basic inequality, we have $$\cE_\cD(\Breve{f})=\|\Breve{f}-f^*\|_\cD^2\le \Opt^*_{\cD}(\Breve{f}).$$
    Therefore, we apply Theorem \ref{thm:risk_bound_fixed_design} and use a union bound to obtain that for any $\delta>0$, with probability at least $1-6\delta$,
    \[
    \cE(\Breve{f})\le \frac{4\widebar{w}}{\underline{w}}\rbr{\frac{1}{K}\sum_{k=1}^{K}\rbr{\tilde{\Opt}_{\cS_k}^{\rho^k_1}(\tilde{f}^k_{\rho^k_1})+\wcheck{\Opt}_{\cS_k}^{\rho^k_2}(\wcheck{f}^k_{\rho^k_2})}+R(\delta)+V_{2r}(\widebar{f})}+\tilde{C}_{v}\frac{\widebar{w}}{\underline{w}}\frac{\log n(\log\log n)\log(1/\delta)}{n^{1-\frac{1}{2v}}}
    \]
    Plugging in the definitions of $R(\delta)$ and the pilot error $V_{2r}(\widebar{f})$ in Theorem \ref{thm:risk_bound_fixed_design} yields the result, and we finish the proof.
\end{proof}
We emphasize that previous works on wild refitting (Wainwright, 2025; Hu and Simchi-Levi, 2025b,a) focus exclusively on the fixed-design setting, thereby avoiding the need to analyze the critical radius of the function class associated with the black-box procedure. To the best of our knowledge, Theorem \ref{thm:excess_risk_random_design} is the first result showing that wild refitting can yield an efficient excess-risk bound in the random-design setting, thereby substantially extending the scope of wild refitting as a tool for analyzing the generalization performance of arbitrary black-box algorithms.

\section{Evaluating Black Box Predictors with Multivariate Covariates}\label{sec:high_dimension}
In modern deep neural networks, covariate embeddings—such as token representations in large language models—typically reside in high-dimensional spaces. To account for this complexity, we extend our interleaved resampling and refitting procedure from the univariate setting to a d-dimensional covariate domain. Formally, we define the covariate space as $\Omega = [0,1]^d$. We assume the underlying function class $\mathcal{F} \subset L^2([0,1]^d)$ satisfies a coordinate-wise Fourier decay condition; specifically, there exist constants $v > d/2$ and $M_v > 0$ such that for all $f \in \mathcal{F}$, for all $k\in\ZZ^d\setminus\cbr{0}$, we have
\[
|k|^{v}\,|\hat f(k)| \le M_v,
\]
where $\hat f(k)$ denotes the Fourier coefficient of $f$ at frequency $k$, and $|k|=|k_1|+\cdots+|k_d|$ is the magnitude of the multi-index $k\in\ZZ^d$.  

Throughout this section, we restrict our attention to the regime $d/2 < v \le d$, consistent with the rationale established in Section \ref{sec:model_setup}. As in the univariate setting, the strict lower bound $v > d/2$ is essential to guarantee square-integrability via Parseval's identity. Specifically, for any function $f \in L^2([0,1]^d)$ whose Fourier coefficients decay asymptotically as $|\hat{f}(k)| \sim |k|^{-v}$, the condition $v > d/2$ is strictly required to ensure the convergence of the associated multivariate $p$-series.

Algorithmically, the core structure of our procedure remains essentially unchanged. The primary technical challenge in the high-dimensional covariate setting arises from the fact that the key mathematical objects in our analysis can no longer be represented naturally as matrices; consequently, standard matrix concentration inequalities are no longer sufficient. Instead, we must establish concentration bounds for higher-order tensors and multi-product empirical processes, which introduce substantial analytical complexity. In particular, to achieve excess risk guarantees comparable to those established in the univariate case, we require a high-dimensional analogue of Lemma \ref{lemma:norm_equivalence} applicable to the empirical norm. We formally establish this crucial result in Lemma \ref{lemma:norm_equi_tensor}.

\begin{lemma}\label{lemma:norm_equi_tensor}
    Denoting $\DD(\cF)$ as the function class $\cbr{f-f':f,f'\in\cF}$, for the covariate data $\cbr{x_i}_{i=1}^n$ in $\cD$ and its subset $\cS=\cbr{x_{k_i}}_{i=1}^{n^{\beta}}$, $\forall \delta>0$, with probability at least $1-2\delta$, when $n$ is sufficiently large, we have
    \[
\|h\|_S^2\le 9\frac{\widebar{w}}{\underline{w}}\|h\|_\cD^2+\frac{12\widebar{w}S_dM_v^2}{\underline{w}(2v-d)}\frac{(\log n)^{2v-d}}{n^{\frac{(2v-d)\beta}{2d+1}}},
\]
uniformly over $h\in\DD(\cF)$ for some constant $S_d$ related only to the dimension $d$.
\end{lemma}
The proof of Lemma \ref{lemma:norm_equi_tensor} is substantially more complicated, which involves some sharp tensor concentration and Talagrand chaining theorems. We refer the reader to Appendix \ref{app:proofs_sec:high_dimension} for the full proof. Unlike Lemma \ref{lemma:norm_equivalence}, which can be established via matrix concentration, the tensor setting requires higher-order concentration inequalities. These tensor concentration bounds are typically looser; consequently, if one naively specializes the resulting guarantee to the scalar case $d=1$, the bound becomes a little weaker than that in Lemma \ref{lemma:norm_equivalence}.

With Lemma \ref{lemma:norm_equi_tensor}, we can show the following theorem, which states that the same interleaved resampling and refitting procedure can provide an efficient upper bound on the excess risk $\cE(\Breve{f})$.
\begin{theorem}\label{thm:risk_bound_high_dimension}
     Assuming that $\cF\subset L^2([0,1]^d)$ and $\forall k\in\ZZ^d\setminus\cbr{0}$, $|k|^v|\hat{f}(k)|\le M_v$ for some $v>d/2$. The black-box procedure $\mathtt{Alg}$ solves the ERM problem by $\hat{f}$.
    $$\Breve{f}=\argmin_{f\in\cF}\frac{1}{n}\sum_{i=1}^{n}(y_i-\Breve{f}(x_i))^2.$$
    $\forall 0<\delta<1$, when $n$ is large enough, for any radius $r$ such that $r\ge \hat{r}_n=\|\Breve{f}-f^*\|_{\cD}$, set $\tilde{r}$ to be 
\[
\tilde{r}:=3\sqrt{\frac{\widebar{w}}{\underline{w}}}r+4M_v\sqrt{\frac{\widebar{w}S_d}{\underline{w}(2v-d)}}\frac{(\log n)^{v-d/2}}{n^{\frac{(v-d/2)\beta}{2d+1}}}.
\]
Let $\cbr{\rho^k_1,\rho^k_2}_{k=1}^K$ be the noise scale sequence such that $$\|\tilde{f}^k_{\rho^k_1}-\Breve{f}\|_{\cS_k}=\|\wcheck{f}^k_{\rho^k_2}-\Breve{f}\|_{\cS_k}=2\tilde{r},\forall k\in[K].$$
Then, with probability at least $1-6\delta$, we have
\begin{align*}
\cE(\Breve{f})\le \frac{4\widebar{w}}{\underline{w}}\rbr{\frac{1}{K}\sum_{k=1}^{K}\rbr{\tilde{\Opt}_{\cS_k}^{\rho^k_1}(\tilde{f}^k_{\rho^k_1})+\wcheck{\Opt}_{\cS_k}^{\rho^k_2}(\wcheck{f}^k_{\rho^k_2})}+R(\delta)+V_{2r}(\widebar{f})}+\tilde{C}_{v}^d\frac{\widebar{w}}{\underline{w}}\frac{\log n(\log\log n)\log(1/\delta)}{n^{1-\frac{d}{2v}}},
\end{align*}
where
\[
R(\delta)=r\frac{10\sqrt{2}\tau \sqrt{\log(1/\delta)}}{\sqrt{n}}+32r\frac{\tau\sqrt{\log(K/\delta)}}{\sqrt{K}}
\]
and
\begin{align*}V_{2r}(\bar{f})&=\sup_{f\in\BB_{2r}(\Breve{f};\cD)}\cbr{\frac{1}{n}\sum_{i=1}^{n}\varepsilon_i(\widebar{f}(x_i)-f^*(x_i))(f(x_i)-\Breve{f}(x_i))}\\
        &+\sup_{f\in\BB_{2r}(\Breve{f};\cD)}\cbr{\frac{1}{n}\sum_{i=1}^{n}\varepsilon_i(\widebar{f}(x_i)-f^*(x_i))(\Breve{f}(x_i)-f(x_i))}.
        \end{align*}
\end{theorem}
The proof of Theorem \ref{thm:risk_bound_high_dimension} is deferred to Appendix \ref{app:proofs_sec:guarantee_random_design}.

For high-dimensional covariate spaces, we can analogously apply the adapted Algorithm \ref{alg:wild-refitting_2} from Subsection \ref{subsec:bounding_hatr_n} to first establish a valid upper bound for $\hat{r}_n$ prior to executing the core interleaved resampling and refitting procedure. As the algorithmic mechanics remain identical to the univariate setting, we omit the detailed exposition here.

Theorem \ref{thm:excess_risk_random_design} demonstrates that the random-design excess risk $\cE(\Breve{f})$ is fundamentally bounded by the empirical excess risk $\cE_\cD(\Breve{f})$ scaled by the square root of the density ratio $\frac{\widebar{w}}{\underline{w}}$ associated with the covariate shift $\frac{d\nu}{d\mu}$. While the core quantities were previously introduced in Subsection \ref{subsec:bounding_risk_fixed_design}, our focus here is on the novel deviation term given by $\frac{\tilde{C}_{v}^d\widebar{w}}{\underline{w}}\frac{\log n(\log\log n)\log(1/\delta)}{n^{1-\frac{d}{2v}}}$. The structural condition $v > d/2$ guarantees that this quantity strictly vanishes as the sample size $n \to \infty$. Furthermore, in the specific regime where $v = d$, this deviation yields a convergence rate of $\cO\left(\frac{\log(n/\delta)}{\sqrt{n}}\right)$, which is the standard rate established in nonparametric statistics.

\section{ Experiments}\label{sec:experiment}
In this section, we present a series of numerical experiments to illustrate the practical performance of our evaluation algorithm. All experiments are conducted on a MacBook Pro M4 Max. Across four representative settings, we show that our method is computationally efficient for both continuous and discontinuous function classes, applies effectively to modern Transformer-based deep learning architectures, and remains robust in the presence of heavy-tailed noise.

\paragraph{ERM with continuous function class.} We consider a one-dimensional nonlinear regression experiment. The data-generating process is given by $y = f^*(x) + w$, where $f^*(x) = \sin(2\pi x)$. The covariate $x$ is drawn independently from $\mathrm{Uniform}(0,1)$, and the noise $w$ is sampled from $\mathcal{N}(0, 0.2^2)$. We evaluate the procedure at four full-scale training sample sizes, $n \in \{500, 1000, 2000, 4000\}$.
For the black-box procedure, we use a multilayer perceptron (MLP) with two hidden layers of width 32. We adopt the L-BFGS optimizer to ensure stable convergence when refitting on moderately sized sub-samples. We fix the sub-sample scaling exponent at $\beta = 0.6$, so that $m = \lfloor n^{0.6} \rfloor$, and perform $K = 30$ resampling rounds. We tune the noise-scale parameter $\rho$ and keep it fixed within each run. Our upper bound suggests that, with appropriate tuning, the same $\rho$ can empirically be used across different rounds, thereby further improving computational efficiency. For sensitivity analysis, we evaluate the wild-optimism bound over $\rho \in \{0.1, 0.5, 1.0, 2.0, 5.0\}$. To ensure a fair comparison, we use the same sequence of sampled subsets for all $\rho$ configurations. Finally, we compare the empirical excess risk $\mathcal{E}_\mathcal{D}(\Breve{f})$, the true excess risk $\mathcal{E}(\Breve{f})$ approximated using 10,000 independent hold-out samples, and the computed pure wild-optimism bound across different $\rho$ values.
Figure \ref{fig:experiment1} shows that, when $\rho$ is chosen appropriately, our procedure yields a rather tight upper bound on the excess risk. This provides strong empirical evidence that Algorithms \ref{alg:wild-refitting} and \ref{alg:wild-refitting_2} can produce tight excess-risk upper bounds.

\begin{figure}[h]
    \begin{center}
    \label{fig:experiment1}
\centerline{\includegraphics[width=0.95\linewidth]{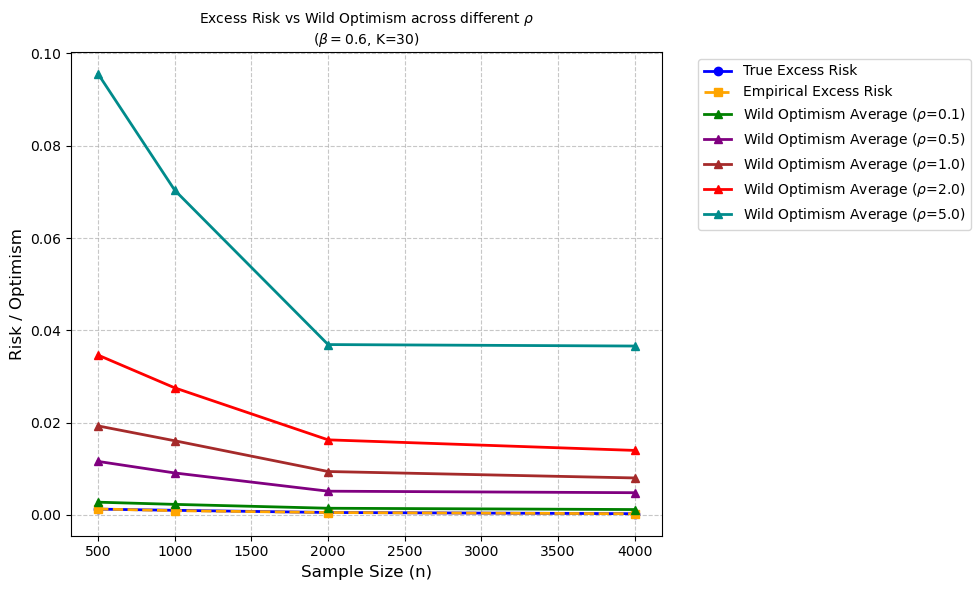}}
\caption{Resampling and refitting for ERM in continuous model}
\end{center}
\vskip -0.2in
\end{figure}
\paragraph{ERM with discontinuous function class.} To further demonstrate the robustness and versatility of the proposed procedure, particularly its behavior under non-smooth decision boundaries, we design a multi-class classification analog experiment using a discontinuous target function. Specifically, the underlying true mapping $f^*(x)$ is defined as a step function that maps the covariate $x \sim \mathrm{Uniform}(0,1)$ to discrete label values $\{0,1,2\}$ over the intervals $[0,0.33)$, $[0.33,0.66)$, and $[0.66,1.0]$, respectively. The observations are corrupted by Gaussian noise $w \sim \mathcal{N}(0,0.15^2)$ to simulate label uncertainty. To examine the asymptotic tightness of the bound, we evaluate the procedure over a wider range of full-scale training sample sizes, from $1{,}000$ to $20{,}000$.
Given the discontinuity of the underlying function, we adopt a random forest regressor as the black-box procedure. Tree-based models naturally partition the feature space into piecewise-constant regions, making them well suited for fitting such step-like classification boundaries. To rigorously test the computational efficiency of our method, we use a more aggressive subsampling exponent $\beta = 0.5$, which reduces the refitting sample size to $m = \lfloor \sqrt{n} \rfloor$, and perform $K = 20$ interleaved resampling rounds. Finally, we compute the wild-optimism bound over a refined grid of perturbation scales $\rho \in \{0.05, 0.1, 0.2, 0.3, 0.4\}$, showing that the proposed upper bound remains effective even for discontinuous functions evaluated using ensemble tree models; see Figure \ref{fig:experiment2}.

\begin{figure}[h]
    \begin{center}
    \label{fig:experiment2}
\centerline{\includegraphics[width=0.85\linewidth]{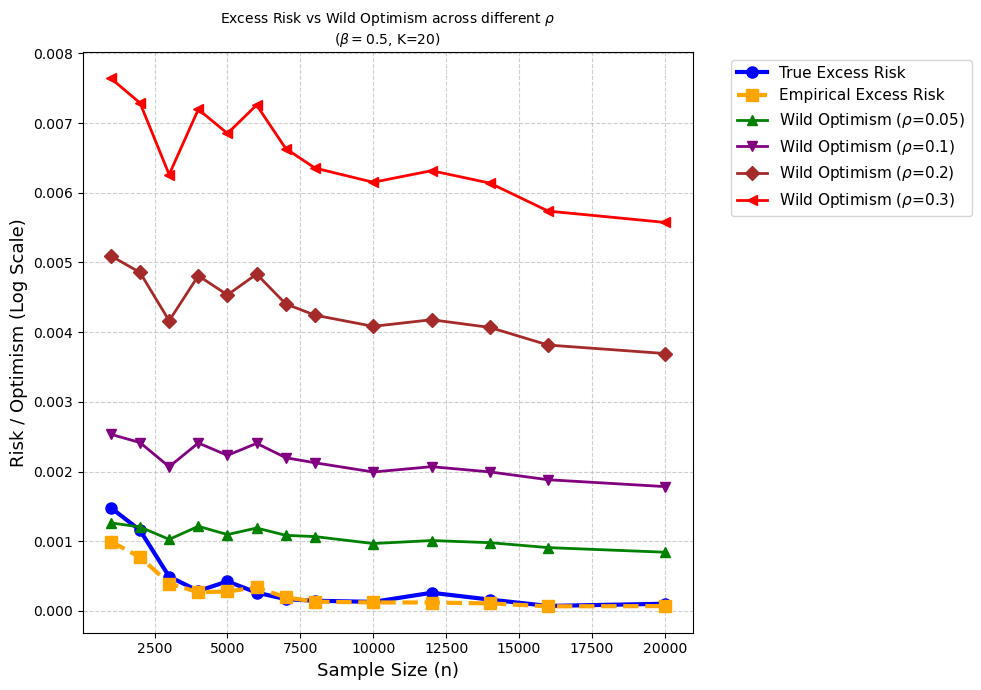}}
\caption{Resampling and refitting for ERM in discontinuous model}
\end{center}
\vskip -0.2in
\end{figure}
\paragraph{ERM with modern Transformer architecture.} To push the boundaries and demonstrate the applicability of our method to modern deep learning, we evaluate the interleaved resampling and wild refitting procedure on a deep sequence model trained on a highly complex high-dimensional target function. Specifically, the covariate is generated as $x \sim \mathrm{Uniform}(-1,1)^5$, and the true regression function is defined by
\[
f^*(x)=\sin(2\pi x_1x_2)+\cos(\pi x_3^3)+\exp(-1.5|x_4|)\operatorname{sign}(x_5)+0.5x_1x_3x_5.
\]
This construction combines high-frequency nonlinear interactions, a cubic cosine term, an exponential decay with a sign discontinuity, and cross-dimensional products. The observations are further corrupted by Gaussian noise $\mathcal{N}(0,0.2^2)$. We study the scaling behavior over training sample sizes $n\in\{2000,4000,8000\}$.

As the black-box empirical risk minimizer, we use a 5-layer Transformer regressor. The model first embeds the 5-dimensional input into a hidden space of dimension $d_{\mathrm{model}}=32$, then applies 5 stacked Transformer encoder layers with 4 attention heads each, followed by average pooling and a linear output head. Training is carried out using Adam. Owing to its large capacity, this architecture is well suited for testing whether our method remains effective in strongly over-parameterized regimes.

To illustrate the computational efficiency of our approach, we fix the subsampling exponent at $\beta=0.6$, so that each refit only uses $m=\lfloor n^{0.6}\rfloor$ samples. The interleaved resampling procedure is repeated for $K=10$ rounds. We then compute the pure wild-optimism bound over the grid $\rho\in\{0.4,0.5,0.6,0.7,0.9\}$ to assess sensitivity to the perturbation scale. The results show that the resulting wild optimism, obtained from these inexpensive refitting steps, consistently tracks the true excess risk $\mathcal{E}(\Breve{f})$ closely, thereby providing strong empirical support for the theoretical guarantees in our paper (Figure \ref{fig:experiment3}).
\begin{figure}[h]
    \begin{center}
    \label{fig:experiment3}
\centerline{\includegraphics[width=0.85\linewidth]{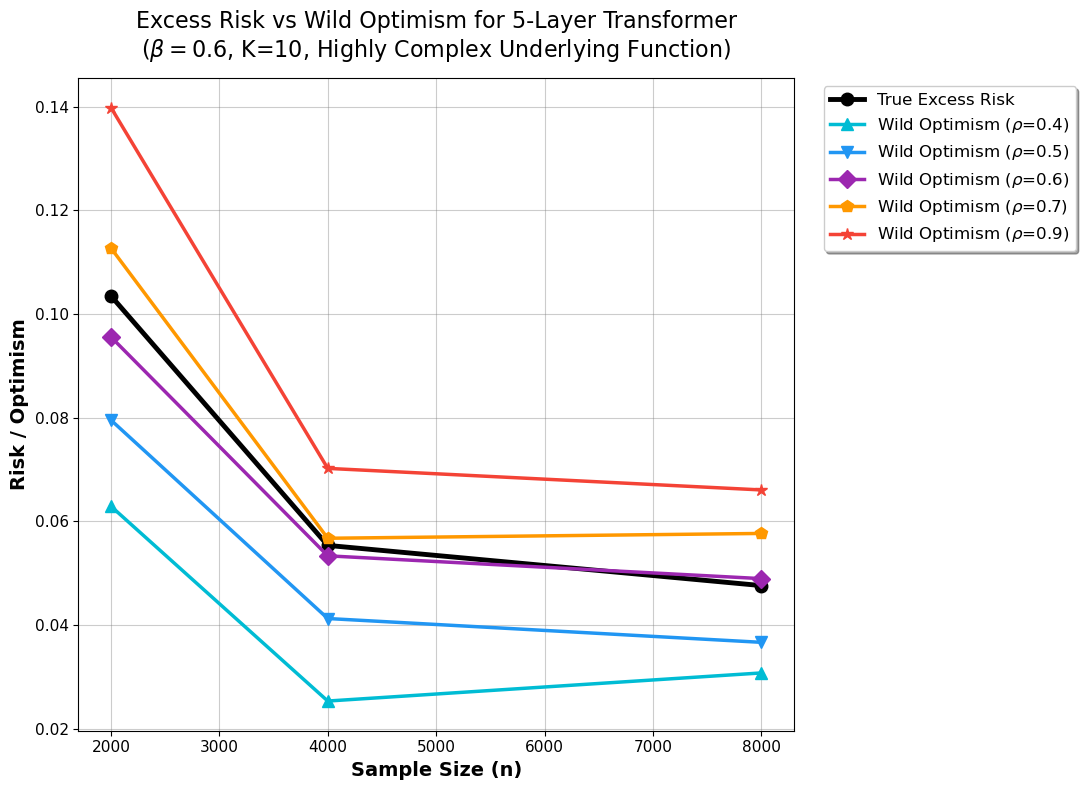}}
\caption{Resampling and refitting for ERM with transformer model}
\end{center}
\vskip -0.2in
\end{figure}
\paragraph{ERM with heavy-tailed noise.}
Finally, we demonstrate that our procedure remains robust in the presence of heavy-tailed noise. To further illustrate the robustness and versatility of the proposed evaluation method, especially under extreme outliers and highly complex black-box architectures, we consider a multi-dimensional nonlinear regression experiment corrupted by heavy-tailed noise. The data-generating process is given by $y=f^*(x)+w$, where $x\sim \mathrm{Uniform}(-1,1)^5$, and the underlying regression function $f^*(x)$ is a highly nonlinear mapping over the 5-dimensional covariate space. Specifically, $f^*(x)$ is defined as:
\begin{equation*}
    f^*(x) = \sin(\pi x_1 x_2) + \cos(\pi x_3) + x_4^2 - |x_5|
\end{equation*}
Crucially, rather than standard Gaussian noise, the observations are corrupted by a heavy-tailed noise distribution. We sample the noise $w$ from a scaled Student's t-distribution with 3 degrees of freedom, i.e., $w \sim 0.2 \cdot t_{3}$. The $t_{3}$ distribution exhibits pronounced thick tails; while its variance is finite, its higher-order moments are infinite. This heavy-tailed characteristic introduces severe, unpredictable outliers into the training data.

For the black-box procedure, we utilize a deep Multilayer Perceptron (MLP) consisting of $7$ layers (one input layer, five hidden layers of width 32 with ReLU activations, and one output layer). To evaluate the behavior of our bounds under these challenging conditions, we examine the procedure at four full-scale training sample sizes, $n \in \{2000, 4000, 6000, 8000\}$. 

Given the depth of the network and the complexity of the data, we fix the sub-sample scaling exponent at $\beta = 0.6$, yielding a refitting sample size of $m = \lfloor n^{0.6} \rfloor$, and perform $K = 10$ interleaved resampling rounds. Finally, we compute the wild-optimism bound over a grid of perturbation scales $\rho \in \{0.05, 0.1, 0.2, 0.5, 1.0\}$. As shown in Figure \ref{fig:experiment4}, our interleaved resampling and refitting procedure successfully and tightly tracks the true excess risk when $\rho$ is appropriately tuned.
\begin{figure}[h]
    \begin{center}
    \label{fig:experiment4}
\centerline{\includegraphics[width=0.85\linewidth]{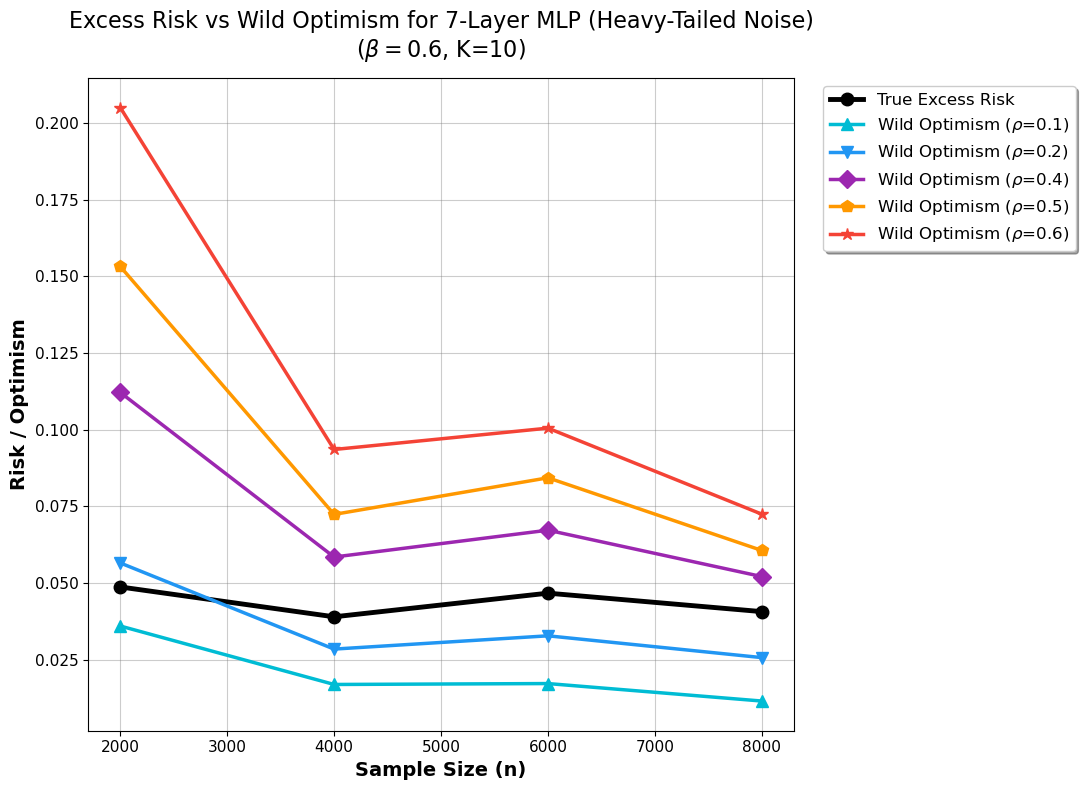}}
\caption{Resampling and refitting for ERM with heavy-tailed noise}
\end{center}
\vskip -0.2in
\end{figure}

\section{Discussion}
In this paper, we propose a compute and data-efficient procedure for rigorously evaluating black-box predictors. To ensure broad applicability to almost arbitrary black-box training algorithms, our statistical guarantees completely bypass the classical function class complexity measures that are traditionally relied upon in learning theory. At its core, our evaluation framework introduces a novel interleaved resampling and refitting structure. To achieve data efficiency and fully utilize the training data we have, we leverage the idea of wild refitting \citep{wainwright2025wild}. Furthermore, to overcome the prohibitive computational burden and enable the efficient evaluation of modern models that are trained on a large scale, we introduce an adjoint resampling subroutine. This critical innovation allows us to refit the predictor on several moderately sized sub-datasets rather than retraining on the original massive scale.

Nevertheless, our algorithm has several limitations, each of which points to an important direction for future work. First, our theoretical guarantees depend explicitly on both the covariate density ratio $\widebar{w}/\underline{w}$ and the ambient dimension $d$ of the covariate space $\Omega$. This dependence is a genuine weakness of the current theory, especially in high-dimensional or distributionally heterogeneous settings, where these quantities may be difficult to control or may lead to overly conservative guarantees. It would therefore be highly desirable to develop a procedure that avoids such dependence while still retaining both computational and statistical efficiency. Second, the practical performance of our method depends on the choice of the noise-scale parameter $\rho$, which currently must be selected manually. This tuning requirement limits the method’s usability in practice and leaves open the possibility that its empirical performance may be sensitive to ad hoc parameter choices. Designing a principled, automatic, and computationally efficient strategy for selecting $\rho$ would therefore substantially strengthen the practical value of the approach. Third, our current analysis relies on a mild Fourier decay assumption on the underlying function class $\cF$. Although this assumption is mathematically convenient, it is still a nontrivial structural restriction, and it remains unclear whether the main conclusions continue to hold without it or under weaker alternatives. Removing or relaxing this assumption is thus an important open problem. Finally, our evaluation procedure requires full access to the original training dataset, which may be unrealistic in modern large-scale applications where the training data are massive, distributed, private, or otherwise inaccessible after training. This limitation restricts the range of settings in which the method can currently be applied, and it would be particularly valuable to develop variants that remain effective under partial or indirect data access.
\clearpage
\appendix
\section{Mathematical Tools}\label{app:math_tools}
In this section, we collect the mathematical preliminaries in Fourier analysis, operator theory, and probability that will be used repeatedly throughout the paper. For results that are not standard or are not stated in a directly applicable way in the classical references, we provide complete, self-contained proofs in Appendix \ref{app:proofs_app:mathtools}.

\begin{lemma}\citep[Theorem 1.1]{demeter2015guide}\label{lemma:Carleson's Theorem}
    $f\in L^2([0,1]^d)$, then the Fourier series of $f$ converges to $f$ almost everywhere with respect to the Lebesgue measure.
\end{lemma}
\begin{lemma}[Plancherel Theorem \citep{yoshizawa1954proof}]\label{lemma:plancherel}
    Let $L^2([0,1]^d)$ be the Hilbert space of complex-valued, Lebesgue squared-integrable functions on $[0,1]^d$, and let $\ell^2(\mathbb{Z}^d)$ be the Hilbert space of square-summable complex sequences over $\ZZ^d$. The Fourier transform; $\hat{f}(k) = \int_0^1 f(x) e^{-2\pi \ib k x} \, dx,k \in \mathbb{Z}^d$ is an isometric isomorphism. Thus, $\forall f, g \in L^2([0,1]^d)$, we have $\langle f,g\rangle_{L^2(\mu)} = \langle \hat{f}, \hat{g} \rangle_{\ell^2}$,
which implies $\|f\|_{L^2(\mu)}^2 = \sum_{k \in \mathbb{Z}^d} |\hat{f}(k)|^2 = \|\hat{f}\|_{\ell^2}^2.$
\end{lemma}
\begin{definition}
    A $d$ dimensional matrix $A=(a_{ij})_{1\le i,j\le d}\in\RR^{d\times d}$, $(1\le d\le \infty)$ is a Toeplitz matrix if $a_{i,j}$ is only a function of $i-j$, $\forall i,j$.
\end{definition}
\begin{lemma}\label{lemma:Toeplitz}
    Given $w(x)\in L^2([0,1])$, denote its Fourier coefficient at frequency $k$ as $\hat{w}(k)=\int_{0}^1w(x)e^{-2\pi\ib kx}dx$. For the Toeplitz Operator $A\in \RR^{\infty\times\infty}$ such that $a_{i,j}=\hat{w}(i-j)$, we have that $\|A\|_{op}=\|w\|_{L^{\infty}}$. Moreover, if $w(x)>0$, $\inf\cbr{\sigma(A)}\ge \inf_{x\in[0,1]}w(x)$.
\end{lemma}
\begin{lemma}\label{lemma:Toeplitz_multidim}
    Given $w(\mathbf{x}) \in L^\infty([0,1]^d)$, denote its multidimensional Fourier coefficient at frequency $\mathbf{k} \in \mathbb{Z}^d$ as 
    $$
    \hat{w}(\mathbf{k}) = \int_{[0,1]^d} w(\mathbf{x}) e^{-2\pi \mathrm{i} \mathbf{k} \cdot \mathbf{x}} d\mathbf{x}.
    $$
    For the Toeplitz Operator $A$ acting on $\ell^2(\mathbb{Z}^d)$ (or its truncation) such that its entries are $a_{\mathbf{i},\mathbf{j}} = \hat{w}(\mathbf{i}-\mathbf{j})$ for $\mathbf{i}, \mathbf{j} \in \mathbb{Z}^d$, we have that $\|A\|_{op} \le \|w\|_{L^{\infty}}$. 
    Moreover, if $w(\mathbf{x}) > 0$, $\inf \left\{ \sigma(A) \right\} \ge \inf_{\mathbf{x}\in[0,1]^d} w(\mathbf{x})$.
\end{lemma}

\begin{lemma}\citep[Theorem 1.6.2]{tropp2015introduction}\label{lemma:matrix_bernstein}
    Denote $S_1,\cdots, S_n$ as independent, centered random matrices with common dimension $d_1\times d_2$, and assume that each one is uniformly bounded $\EE[S_k]=0$, $\|S_k\|_{op}\le L$. Then, introduce the sum $Z=\sum_{i=1}^{n}S_i$ and let $v(Z)$ denote the matrix variance $v(Z):=\max\cbr{\EE[\|ZZ^T\|_{op}],\EE[\|Z^TZ\|_{op}]}.$ For any $t>0$, we have
    \[
    \PP(\|Z\|_{op}\ge t)\le (d_1+d_2)\exp\rbr{\frac{-t^2/2}{v(Z)+Lt/3}}.
    \]
\end{lemma}
\begin{lemma}[Hoeffding Inequality]\label{lemma:Hoeffding}
    If $X_1,\cdots,X_n$ are independent and satisfy $X_i\in[a_i,b_i]$, then for any $t>0$,
    \[
\Pr\!\left( \left| \frac{1}{n}\sum_{i=1}^n X_i - \mathbb{E}\!\left[\frac{1}{n}\sum_{i=1}^n X_i\right] \right| \ge t \right)
\;\le\; 2 \exp\!\left( -\frac{2n^2 t^2}{\sum_{i=1}^n (b_i-a_i)^2} \right).
\]
\end{lemma}
\begin{lemma}\label{lemma:concentration_Lip}\citep[Theorem 3.24]{wainwright2019high}
Consider a vector of independent random variables $(X_1,\cdots,X_n)$, each taking values in $[0,1]$, and let $f:\RR^n\mapsto\RR$ be convex, $L$-Lipschitz with respect to the Euclidean norm. Then, for all $t\ge0$, we have
\[
\PP(|f(X)-\EE[f(X)]|\ge t)\le 2e^{\frac{-t^2}{2L^2}}.
\]
\end{lemma}
\begin{lemma}[Empirical Dual Norm]\label{lemma:dual_norm}
Let $v, b \in \mathbb{R}^n$. Define the empirical inner product as $\langle v, b \rangle_n = \frac{1}{n} \sum_{i=1}^n v_i b_i$ and the corresponding empirical $L_2$ norm as $\lVert v \rVert_n = \sqrt{\langle v, v \rangle_n}$. For any given radius $R > 0$, the following property holds:
\begin{equation*}
    \sup_{\lVert v \rVert_n \le R} \frac{1}{n} \sum_{i=1}^n v_i b_i = R \lVert b \rVert_n
\end{equation*}
\end{lemma}
\begin{lemma}\citep{adamczak2008tail}
    Let $X_1,X_2,\cdots,X_n$ be i.i.d. random variables with values in a measurable space $\cX$ and let $\cF$ be a class of measurable functions $f:\cX\rightarrow\RR$. Assume that for every $f\in\cF$ and every $i$, $\EE[f(X_i)]=0$ and for some $\alpha\in(0,1]$ and all $i$, $\|\sup_{f\in\cF}\|f(X_i)\|_{\psi_\alpha}\|<\infty$. Let
    $Z=\sup_{f\in\cF}|\sum_{i=1}^{n}f(X_i)|$. Define $\sigma^2=\sup_{f\in\cF}\sum_{i=1}^{n}\EE[f(X_i)^2]$.
    Then, for all $0<\eta<1$ and $\delta>0$, there exists a constant $C=C(\alpha,\eta,\delta)$, such that $\forall t>0$,
    \[
    \PP(Z\ge (1+\eta)\EE[Z]+t)\le \exp(-\frac{t^2}{2(1+\delta)\sigma^2})+3\exp\rbr{-\rbr{\frac{t}{C\|\sup_{f\in\cF}|f(X)|\|_{\psi_\alpha}}}^\alpha}
    \]
    and
    \[
    \PP(Z\le (1-\eta)\EE[Z]-t)\le \exp(-\frac{t^2}{2(1+\delta)\sigma^2})+3\exp\rbr{-\rbr{\frac{t}{C\|\sup_{f\in\cF}|f(X)|\|_{\psi_\alpha}}}^\alpha}.
    \]
\end{lemma}
\begin{lemma}\citep{bousquet2002bennett}
    Assuming that $X_1,\cdots,X_n$ are i.i.d. sampled from $\PP$. Let $\cF$ be a set of measurable functions from $\cX$ to $\RR$ and assume that $\forall f\in\cF$, $f$ is square-integrable and satisfy $\EE[f]=0$. If $\sup_{f\in\cF}\|f\|_{\infty}\le 1$, then, we denote
    \[
    Z=\sup_{f\in\cF}\sum_{i=1}^{n}f(X_i),\ \text{or}\ Z=\sup_{f\in\cF}|\sum_{i=1}^{n}f(X_i)|.
    \]
    Let $\sigma$ be a positive real number such that $\sigma^2\ge \sup_{f\in\cF}\text{Var}(f(X))$ almost surely. Then, for all $x\ge 0$, we have
    \[
    \PP(Z\ge \EE[Z]+\sqrt{2xv}+\frac{x}{3})\le e^{-x},
    \]
    where $v=n\sigma^2+2\EE[Z]$.
\end{lemma}
\begin{lemma}\citep[Lemma 40]{rakhlin2022mathstat}\label{lemma:critical_radius_rakhlin}
    Let $\cG$ be a class of functions with values in $[0,1]$. $X_1,\cdots,X_n$ are sampled i.i.d. from $\PP$. Then, there are constants $c$ and $c'$ such that with probability at least $1-e^{-t}$, for all $g\in\cG$,
    \[
    \EE[g(X)]\le \frac{2}{n}\sum_{i=1}^{n}g(X_i)+c\cdot\bar{\delta}^2+\frac{c'(t+\log\log n)}{n},
    \]
    where $\bar{\delta}$ is the critical radius of $\cG$ s.t. $\EE_{\varepsilon}[\sup_{g\in\cG:\frac{1}{n}\sum_{i=1}^{n}g(X_i)\le\delta^2}\frac{1}{n}\sum_{i=1}^{n}\varepsilon_ig(X_i)]\le \frac{\delta^2}{2},\forall \delta\ge\bar{\delta}$.
\end{lemma}
\begin{definition}\label{def:tensor_operator_norm}
    Assume $p\ge 2$, $H_1,\cdots, H_p$ are $p$ Hilbert spaces. $T\in H_1\otimes\cdots\otimes H_p$. The operator norm of $T$ is defined as
    \[
    \|T\|_{op}:=\sup_{\forall k\in[p], w_k\in H_k,\|w_k\|\le 1}\cbr{\abr{\inner{T}{w_1\otimes w_2\otimes\cdots \otimes w_p}}}.
    \]
\end{definition}
\begin{definition}\citep{vershynin2018high}\label{def:orlicz}
    For any function $\cF$ defined on a measure space $(\Omega,\mu)$, we define the Orlicz $\psi_2$-norm of $f$ given by $\|f\|_{\psi_2}:=\inf\cbr{c>0:\EE_{X\sim\mu}[\exp\rbr{\frac{|f(X)|^2}{c^2}}]\le 2}$. The Orlicz norm of $\cF$ is given by $d_{\psi_2}(\cF):=\sup_{f\in\cF}\|f\|_{\psi_2}$.
\end{definition}
\begin{definition}\citep{talagrand2014upper}
    $(\cF,d)$ is a metric space. An admissible sequence is an increasing sequence $\{\cF_s\}_{s\ge 0}\subset\cF$ such that $\cF_s\subset\cF_{s+1}$, $|\cF_0|=1$, $|\cF_s|\le 2^{2^s}$, and $\cup_{s\ge 0}\cF_s$ is dense in $\cF$. Talagrand's $\gamma$ functional is defined to be $\gamma(\cF,d):=\inf\sup_{f\in\cF}\sum_{s\ge 0}2^{s/2}d(f,\cF_s)$, where the infimum is taken over all admissible sequences. When the distance on $\cF$ is induced by the $\psi_2$ Orlicz norm in Definition \ref{def:orlicz}, we write $\gamma(\cF,d)$ as $\gamma(\cF,\psi_2)$.
\end{definition}
\begin{lemma}\citep[Majorizing Measure Theorem]{talagrand2014upper}\label{lemma:majorizing_measure}
    Let $T$ be an arbitrary index set and $(Z_t)_{t \in T}$ be a centered stochastic process. Suppose that the process has sub-Gaussian increments with respect to a metric $d$ on $T$. That is, there exists a constant $K > 0$ such that for all $s, t \in T$ and all $u > 0$:
\[
\mathbb{P}(|Z_s - Z_t| \ge u) \le 2 \exp\left( - \frac{u^2}{K^2 d(s, t)^2} \right).
\]
Then, for some universal $L>0$, the expected supremum of the process is upper bounded by the Talagrand $\gamma$ functional of the metric space $(T, d)$:
\[
\mathbb{E}\left[ \sup_{t \in T} Z_t \right] \le L \gamma(T, d),
\]
Furthermore, if $(Z_t)_{t \in T}$ is specifically a centered Gaussian process, and $d$ is its canonical metric defined as $d(s, t) = (\mathbb{E}|Z_s - Z_t|^2)^{1/2}$, then the upper bound is tight. That is, there exists another universal constant $L > 0$ such that the geometric complexity is also lower bounded by the expected supremum:
\[
\frac{1}{L} \gamma(T, d) \le \mathbb{E}\left[ \sup_{t \in T} Z_t \right] \le L \gamma(T, d).
\]
\end{lemma}
\begin{lemma}\citep{chen2025sharp}\label{lemma:multiproduct_empiricalprocess}
    For any integer $p\ge 2$ and $1\le k\le p$, let $X^k, X_1^k,\cdots,X_N^k\stackrel{i.i.d}{\sim}\mu_k$ be a sequence of random variables on the probability space $(\Omega_k,\mu_k)$, and let $\cF^k$ be a measurable function class such that $0\in\cF^k$ or $\cF^k$ is symmetric. Then,
    \[
    \EE\sbr{\sup_{f^k\in\cF^k,1\le k\le p}\abr{\frac{1}{N}\sum_{i=1}^{N}\prod_{k=1}^{p}f^k(X_i^k)-\EE\prod_{k=1}^{p}f^k(X^k)}}\lesssim \rbr{\prod_{k=1}^{p}d_{\psi_2}(\cF^k)}\cJ_N(\{\cF^k\}_{k=1}^p),
    \]
    where 
    \[
    \cJ_N(\{\cF^k\}_{k=1}^p):=\frac{\sum_{k=1}^{p}\bar{\gamma}(\cF^k,\psi_2)}{\sqrt{N}}+\frac{\prod_{k=1}^{p}\rbr{\bar{\gamma}(\cF^k,\psi_2)+(\log N)^{1/2}}}{N}, \bar{\gamma}(\cF^k,\psi_2)=\frac{\gamma(\cF^k,\psi_2)}{d_{\psi_2}(\cF^k)}.
    \]
\end{lemma}
\begin{lemma}[McDiarmid's Inequality]\label{lemma:Mcdiarmid}
Let $X_1,\dots,X_n$ be independent random variables taking values in a set $\mathcal{X}$, and let
$f:\mathcal{X}^n\to\mathbb{R}$ be a measurable function. Suppose there exist constants
$c_1,\dots,c_n\ge 0$ such that for all $i\in\{1,\dots,n\}$ and all $x_1,\dots,x_n,x_i'\in\mathcal{X}$,
\[
\big|f(x_1,\dots,x_i,\dots,x_n)-f(x_1,\dots,x_i',\dots,x_n)\big|\le c_i.
\]
Then for all $t>0$,
\[
\mathbb{P}\!\left(f(X_1,\dots,X_n)-\mathbb{E}[f(X_1,\dots,X_n)]\ge t\right)
\le \exp\!\left(-\frac{2t^2}{\sum_{i=1}^n c_i^2}\right),
\]
\[
\mathbb{P}\!\left(\mathbb{E}[f(X_1,\dots,X_n)]-f(X_1,\dots,X_n)\ge t\right)
\le \exp\!\left(-\frac{2t^2}{\sum_{i=1}^n c_i^2}\right).\]
\end{lemma}
\begin{lemma}\citep{wellner2013weak}\label{lemma:VC_dim}
    If $\cG$ is a linear space of real-valued functions of finite dimension $p$, then its VC subgraph dimension is upper bounded by $p+2$. Moreover, for any function class with VC dimension $V$, its critical radius $\delta_n$ is upper bounded by $\sqrt{\frac{V\log(n/V)}{n}}$.
\end{lemma}
\section{Proofs in Section \ref{sec:model_setup}}\label{app:proofs_sec:model_setup}
\begin{proposition}\label{prop:transformer_func_ass}
Let $X_0 \in \RR^{N \times d}$ be an input sequence of $N$ tokens, each of dimension $d$ and $(X_0)_{ij}\in[0,1]$. Let $\eta_\epsilon: \mathbb{R}^{N \times d} \to \mathbb{R}^{N \times d}$ be the practical row-wise Layer Normalization operator with a smoothing constant $\epsilon > 0$, defined for the $i$-th row as:
\[ \eta_\epsilon(X)_{i,:} = \frac{\sqrt{d}}{\sqrt{\|X_{i,:}\|_2^2 + d\epsilon}} X_{i,:} \]
Consider an $L$-layer Transformer with layer norm normalization. Each block $l \in \{1, \dots, L\}$ consists of a self-attention mapping $\mathcal{A}^{(l)}$ and a feed-forward mapping $\mathcal{F}^{(l)}$:
\begin{align*}
Y_l = X_{l-1} + \mathcal{A}^{(l)}(\eta_\epsilon(X_{l-1})),\ 
X_l = Y_l + \mathcal{F}^{(l)}(\eta_\epsilon(Y_l)).
\end{align*}
The feed-forward layer $\mathcal{F}^{(l)}$ and self-attention function $\cA^{(l)}$ are explicitly defined as:
$$\mathcal{F}^{(l)}(Z) = \sigma(Z W_1^{(l)} + B_1^{(l)}) W_2^{(l)} + B_2^{(l)},\ 
    \mathcal{A}^{(l)}(Z) = \text{softmax}(Z W_Q^{(l)} {W_K^{(l)}}^\top Z^\top) Z W_V^{(l)}.$$
Then, denote the output as $X_L=(X_L^1(X_0),\cdots,X_L^d(X_0))\in\RR^d$, the spectral norm of the end-to-end network Jacobian, $J = \frac{\partial X_L}{\partial X_0}$, is strictly bounded by a constant determined only by the model parameters and architecture constants:
\[ \|J\|_2 \le \prod_{l=1}^L \left( 1 + \frac{C_{\mathcal{A}}^{(l)}}{\sqrt{\epsilon}} \right) \left( 1 + \frac{C_{\mathcal{F}}^{(l)}}{\sqrt{\epsilon}} \right) \]
where $C_{\mathcal{F}}^{(l)}$ and $C_{\mathcal{A}}^{(l)}$ are parameter-dependent constants given by:
\begin{align*}
C_{\mathcal{F}}^{(l)} = \|W_1^{(l)}\|_2 \|W_2^{(l)}\|_2,\ 
C_{\mathcal{A}}^{(l)} = \|W_V^{(l)}\|_2 \left( 1 + 2Nd \|W_Q^{(l)}\|_2 \|W_K^{(l)}\|_2 \right).
\end{align*}
Finally, for any $k\in\ZZ^{N\cdot d}$, the Fourier coefficient for any coordinate $j\in[d]$, satisfies that
\[
|\hat{X_L^j}(k)|\le \frac{\prod_{l=1}^L \left( 1 + \frac{C_{\mathcal{A}}^{(l)}}{\sqrt{\epsilon}} \right) \left( 1 + \frac{C_{\mathcal{F}}^{(l)}}{\sqrt{\epsilon}} \right)}{\|k\|_2}.
\]
\end{proposition}

\begin{proof}[Proof of Proposition \ref{prop:transformer_func_ass}]
We apply the chain rule and submultiplicativity of the induced 2-norm. The core of this proof relies on establishing a global upper bound for the Jacobian of $\eta_\epsilon$, independent of its input.

\textbf{Step 1: Global Lipschitz Constant of $\eta_\epsilon$} \\
Let $x \in \mathbb{R}^d$ represent an arbitrary row vector. The normalization function is $y(x) = \frac{\sqrt{d}}{\sqrt{\|x\|_2^2 + c}} x$, where $c = d\epsilon > 0$. Taking the derivative of $y$ with respect to $x$ yields the Jacobian matrix:
\[ Dy(x) = \frac{\sqrt{d}}{\sqrt{\|x\|_2^2 + c}} I - \frac{\sqrt{d} x x^\top}{(\|x\|_2^2 + c)^{3/2}} \]
To find the spectral norm $\|Dy(x)\|_2$, we analyze its eigenvalues. For any vector $v$:
\begin{itemize}
    \item If $v \perp x$, $Dy(x) v = \frac{\sqrt{d}}{\sqrt{\|x\|_2^2 + c}} v$. The eigenvalue is $\lambda_1 = \frac{\sqrt{d}}{\sqrt{\|x\|_2^2 + c}}$.
    \item If $v \parallel x$, $Dy(x) x = \left( \frac{\sqrt{d}}{\sqrt{\|x\|_2^2 + c}} - \frac{\sqrt{d} \|x\|_2^2}{(\|x\|_2^2 + c)^{3/2}} \right) x = \frac{\sqrt{d} c}{(\|x\|_2^2 + c)^{3/2}} x$. The eigenvalue is $\lambda_2 = \frac{\sqrt{d} c}{(\|x\|_2^2 + c)^{3/2}}$.
\end{itemize}
Since $c > 0$ and $\|x\|_2^2 \ge 0$, the maximum of $\lambda_1$ occurs when $\|x\|_2 = 0$, giving $\frac{\sqrt{d}}{\sqrt{c}} = \frac{\sqrt{d}}{\sqrt{d\epsilon}} = \frac{1}{\sqrt{\epsilon}}$. Similarly, the maximum of $\lambda_2$ also occurs at $\|x\|_2 = 0$, giving $\frac{\sqrt{d} c}{c^{3/2}} = \frac{\sqrt{d}}{\sqrt{c}} = \frac{1}{\sqrt{\epsilon}}$. Therefore, for any input $X$, the spectral norm of the block-diagonal Jacobian of the entire sequence operation $\eta_\epsilon$ is strictly bounded:
\[ \|D\eta_\epsilon(X)\|_2 \le \frac{1}{\sqrt{\epsilon}} \]

\textbf{Step 2: Jacobians of Sub-layers Evaluated on Bounded Inputs} \\
By definition of $\eta_\epsilon$, the output row vector strictly satisfies $\|\eta_\epsilon(X)_{i,:}\|_2 < \sqrt{d}$. Thus, the Frobenius norm of the normalized sequence matrix $Z = \eta_\epsilon(X)$ satisfies $\|Z\|_F < \sqrt{Nd}$, which implies its spectral norm is bounded: $\|Z\|_2 < \sqrt{Nd}$.

For the FFN layer $\mathcal{F}^{(l)}(Z)$, the Jacobian norm relies only on the 1-Lipschitz property of $\sigma$ and the weight norms:
\[ \|D\mathcal{F}^{(l)}(Z)\|_2 \le \|W_2^{(l)}\|_2 \cdot 1 \cdot \|W_1^{(l)}\|_2 = C_{\mathcal{F}}^{(l)} \]

For the Self-Attention layer $\mathcal{A}^{(l)}(Z)$, let $P = \text{softmax}(Z W_Q^{(l)} {W_K^{(l)}}^\top Z^\top)$. The product rule yields:
\[ d\mathcal{A}^{(l)} = dP \cdot (Z W_V^{(l)}) + P \cdot (dZ W_V^{(l)}) \]
Using the property that $\|P\|_2 \le \sqrt{N}$ and the previously established bound $\|Z\|_2 < \sqrt{Nd}$:
\begin{align*}
\|D\mathcal{A}^{(l)}(Z)\|_2 &\le \|W_V^{(l)}\|_2 + 2 \|Z\|_2^2 \|W_Q^{(l)}\|_2 \|W_K^{(l)}\|_2 \|W_V^{(l)}\|_2 \\
&\le \|W_V^{(l)}\|_2 \left( \sqrt{N} + 2Nd \|W_Q^{(l)}\|_2 \|W_K^{(l)}\|_2 \right) = C_{\mathcal{A}}^{(l)}
\end{align*}

\textbf{Step 3: End-to-End Composition} \\
Using the chain rule for the $l$-th residual block components:
\begin{align*}
\left\| \frac{\partial Y_l}{\partial X_{l-1}} \right\|_2 &= \|I + D\mathcal{A}^{(l)}(\eta_\epsilon(X_{l-1})) D\eta_\epsilon(X_{l-1})\|_2 \le 1 + \frac{C_{\mathcal{A}}^{(l)}}{\sqrt{\epsilon}} \\
\left\| \frac{\partial X_l}{\partial Y_l} \right\|_2 &= \|I + D\mathcal{F}^{(l)}(\eta_\epsilon(Y_l)) D\eta_\epsilon(Y_l)\|_2 \le 1 + \frac{C_{\mathcal{F}}^{(l)}}{\sqrt{\epsilon}}
\end{align*}
Taking the product over all $L$ layers yields the absolute, parameter-only bound:
\[ \|J\|_2 \le \prod_{l=1}^L \left\| \frac{\partial X_l}{\partial Y_l} \right\|_2 \left\| \frac{\partial Y_l}{\partial X_{l-1}} \right\|_2 \le \prod_{l=1}^L \left( 1 + \frac{C_{\mathcal{A}}^{(l)}}{\sqrt{\epsilon}} \right) \left( 1 + \frac{C_{\mathcal{F}}^{(l)}}{\sqrt{\epsilon}} \right). \]
Finally, by definition, we have $J=(\nabla (X_L^1)^T,\cdots,\nabla (X_L^d)^T)^T$. Applying the Parseval's theorem, we have that for any coordinate $j\in[d]$,
\[
\sum_{k\in\ZZ^{N\cdot d}}\|k\|_2^2|\hat{X_L^j}(k)|^2=\int_{[0,1]^{N\cdot d}}\|\nabla X_L^j\|_2^2\le \|J\|_2^2
\]
Observing that $\|k\|_2^2|\hat{X_L^j}(k)|^2\le \sum_{k\in\ZZ^{N\cdot d}}\|k\|_2^2|\hat{X_L^j}(k)|^2$, we finish the proof.
\end{proof}

\begin{proof}[Proof of Lemma \ref{lemma:holder_continuity}]
We first consider the case $1 < v < 2$. Since $v > 1$, the Fourier series is absolutely convergent, and $f(x) = \sum_{k \in \mathbb{Z}} \hat{f}(k) e^{2\pi \ib k x}$ holds everywhere. For any $x, y \in [0,1]$, let $\delta = |x - y| > 0$. We can bound the difference by:
\begin{align*}
    |f(x) - f(y)| &\le \sum_{k \neq 0} |\hat{f}(k)| \left| e^{2\pi \ib k x} - e^{2\pi \ib k y} \right| \le \sum_{k \neq 0} \frac{M_v}{|k|^v} \min\big(2, \, 2\pi |k| \delta\big).
\end{align*}
We split the sum into low frequencies ($|k| \le 1/\delta$) and high frequencies ($|k| > 1/\delta$):
\begin{align*}
    |f(x) - f(y)| &\le 2\pi M_v \delta \sum_{0 < |k| \le 1/\delta} |k|^{1-v} + 2 M_v \sum_{|k| > 1/\delta} |k|^{-v} \\
    &\le C_1 \delta \cdot \left( \frac{1}{\delta} \right)^{2-v} + C_2 \left( \frac{1}{\delta} \right)^{1-v} \\
    &= (C_1 + C_2) \delta^{v-1}.
\end{align*}
Thus, $f$ is Hölder continuous with exponent $\alpha = v - 1$.
\end{proof}

\begin{proof}[Proof of Lemma \ref{lemma:infinite_covering_bounded}]
Consider the sequence of truncated sawtooth waves $$f_N(x) = 2M_1 \sum_{k=1}^N \frac{\sin(2\pi k x)}{k}.$$ As established by classical Fourier analysis, $f_N \in \mathcal{F}$ for all $N \ge 1$ because it is continuous, its Fourier coefficients satisfy the exact bound $|k||\hat{f}_N(k)| = M_1$, and its partial sums are uniformly bounded by $B$.

To explicitly prove the covering number is infinite, we construct an infinite subsequence $\{f_{N_j}\}_{j=1}^\infty \subset \mathcal{F}$ such that the $L^\infty$ distance between any two distinct functions in this subsequence is uniformly bounded away from zero. 

Let $N$ and $M$ be integers such that $N > M$. The difference between them is:
\begin{equation}
    f_N(x) - f_M(x) = 2M_1 \sum_{k=M+1}^N \frac{\sin(2\pi k x)}{k}.
\end{equation}
We evaluate this difference at the specific point $x^*_N = \frac{1}{4N}$. For all indices $k$ in the summation range ($M+1 \le k \le N$), the argument of the sine function satisfies:
\begin{equation}
    2\pi k x^*_N = \frac{\pi k}{2N} \in \left( 0, \frac{\pi}{2} \right].
\end{equation}
On the interval $[0, \pi/2]$, we apply Jordan's inequality, $\sin(\theta) \ge \frac{2}{\pi}\theta$. Substituting this into our summation yields:
\begin{align*}
    f_N(x^*_N) - f_M(x^*_N) &\ge 2M_1 \sum_{k=M+1}^N \frac{1}{k} \left( \frac{2}{\pi} \frac{\pi k}{2N} \right) = 2M_1 \sum_{k=M+1}^N \frac{1}{k} \left( \frac{k}{N} \right) = 2M_1 \sum_{k=M+1}^N \frac{1}{N} \\
    &= 2M_1 \left( \frac{N - M}{N} \right) = 2M_1 \left( 1 - \frac{M}{N} \right).
\end{align*}
Since the $L^\infty$ norm is the supremum over all $x \in [0,1]$, we have that:
\begin{equation}
    \|f_N - f_M\|_\infty \ge \left| f_N(x^*_N) - f_M(x^*_N) \right| \ge 2M_1 \left( 1 - \frac{M}{N} \right).
\end{equation}

Now, define the subsequence by choosing $N_j = 2^j$ for $j = 1, 2, \dots$. For any two distinct indices $i > j$, we have $N_i \ge 2 N_j$. Applying our established bound:
\begin{equation}
    \|f_{N_i} - f_{N_j}\|_\infty \ge 2M_1 \left( 1 - \frac{N_j}{N_i} \right) \ge 2M_1 \left( 1 - \frac{1}{2} \right) = M_1.
\end{equation}

We have found an infinite set of functions $\{f_{2^j}\}_{j=1}^\infty$ in $\mathcal{F}$ such that the pairwise $L^\infty$ distance between any two of them is at least $M_1$. By definition, it is impossible to cover an infinite set of points separated by a distance of $M_1$ with any finite number of balls of radius $\epsilon < M_1 / 2$. Therefore, $N(\epsilon, \mathcal{F}, \|\cdot\|_\infty) = \infty$.
\end{proof}

\section{Proofs in Section \ref{sec:statistical_guarantee}}\label{app:proofs_sec:statistical_guarantee}
\begin{proof}[Proof of Lemma \ref{lemma:unbias}]
Let us define the estimation error for the $k$-th sub-sample fold as $Z_k := A_n(f) - B_{\mathcal{S}_k}(f)$. By substituting the definitions of $A_n(f)$ and $B_{\mathcal{S}_k}(f)$, we can express $Z_k$ as:
\[
Z_k = \frac{1}{n}\sum_{i=1}^{n}\varepsilon_i v_i (f(x_i)-\Breve{f}(x_i)) \left(1 - \frac{\delta_k^i}{\pi_i}\right).
\]
We analyze $Z_k$ conditional on the dataset $\mathcal{D}_n = \{(x_i, y_i)\}_{i=1}^n$ and the Rademacher random variables $\{\varepsilon_i\}_{i=1}^n$. Under this conditioning, the randomness of $Z_k$ arises strictly from the sub-sample indicator variables $\delta_k^i$. Assuming the sampling is unbiased such that $\mathbb{E}[\delta_k^i] = \pi_i$, it follows that $\mathbb{E}[Z_k] = 0$.

To bound $Z_k$ in terms of the empirical norm $\|f - \Breve{f}\|_\cD$, by Cauchy-Schwarz inequality, we have:
\[
|Z_k| \le \left( \frac{1}{n}\sum_{i=1}^{n} \left[ \varepsilon_i v_i \left(1 - \frac{\delta_k^i}{\pi_i}\right) \right]^2 \right)^{1/2} \left( \frac{1}{n}\sum_{i=1}^{n} (f(x_i)-\Breve{f}(x_i))^2 \right)^{1/2}.
\]
Since $\varepsilon_i \in \{-1, 1\}$, we have $\varepsilon_i^2 = 1$. The second term is exactly the empirical norm $\|f - \Breve{f}\|_\cD$. For the coefficient, we have
\[
\left( \frac{1}{n}\sum_{i=1}^{n} v_i^2 \left(1 - \frac{\delta_k^i}{\pi_i}\right)^2 \right)^{1/2}=\rbr{\frac{1}{n}\sum_{i=1}^{n}v_i^2(1-2\frac{\delta^i_k}{\pi_i}+\frac{\delta^i_k}{\pi_i^2})}^{1/2}\le \rbr{\frac{1}{n}\sum_{i=1}^{n}v_i^2}^{1/2}\le \tau.
\]
This provides an almost sure absolute upper bound for $Z_k$:
\[
|Z_k| \le \tau \|f - \Breve{f}\|_\cD.
\]

Because $Z_k$ has zero mean and is strictly bounded in the interval $[-c, c]$ with $c = 4\sqrt{2}\tau \|f - \Breve{f}\|_\cD$, we can invoke Hoeffding's Lemma. The sub-Gaussian parameter for $Z_k$ is:
\[
\sigma_{Z_k}^2 \le \frac{(2c)^2}{4} = c^2 = (\tau \|f - \Breve{f}\|_n)^2 = \tau^2 \|f - \Breve{f}\|_\cD^2.
\]

Next, we consider the average deviation across all $K$ independent sub-sample folds, denoted by $\Delta := A_n(f) - \frac{1}{K}\sum_{k=1}^{K}B_{\mathcal{S}_k}(f) = \frac{1}{K}\sum_{k=1}^{K}Z_k$. Assuming the $K$ sub-samples are drawn independently, $\Delta$ is the average of $K$ independent zero-mean sub-Gaussian random variables. By the properties of independent sub-Gaussian variables, the sub-Gaussian parameter of the average $\Delta$ is reduced by a factor of $K$:
\[
\sigma_{\Delta}^2 = \frac{1}{K^2} \sum_{k=1}^{K} \sigma_{Z_k}^2 = \frac{\tau^2 \|f - \Breve{f}\|_n^2}{K}.
\]
Applying Lemma \ref{lemma:Hoeffding} (Hoeffding's Inequality) to $\Delta$, we have for any $t > 0$:
\[
\mathbb{P}(\Delta \ge t) \le \exp\left( - \frac{t^2}{2 \sigma_{\Delta}^2} \right).
\]
We set the right-hand side equal to the confidence level $\frac{\delta}{K}$ (which naturally accommodates a subsequent Union Bound over the $K$ subsets):
\[
\exp\left( - \frac{K t^2}{\tau^2 \|f - \Breve{f}\|_\cD^2} \right) = \frac{\delta}{K}.
\]
Taking the natural logarithm of both sides and solving for $t$ yields:
\begin{align*}
t &= \frac{\|f - \Breve{f}\|_\cD \tau \sqrt{\log(K/\delta)}}{\sqrt{K}}.
\end{align*}

Therefore, with probability at least $1 - \frac{\delta}{K}$, $\Delta \le t$ holds. Rearranging this inequality gives the first desired bound:
\[
A_n(f) \le \frac{1}{K}\sum_{k=1}^{K}B_{\mathcal{S}_k}(f) + \frac{8\|f-\Breve{f}\|_\cD\tau\sqrt{\log(K/\delta)}}{\sqrt{K}}.
\]

By symmetry, since $C_n(f)$ and $D_{\mathcal{S}_k}(f)$ share the exact same algebraic and sampling structure as $A_n(f)$ and $B_{\mathcal{S}_k}(f)$, applying the identical Cauchy-Schwarz separation and Hoeffding bounding procedure directly yields the second inequality:
\[
C_n(f) \le \frac{1}{K}\sum_{k=1}^{K}D_{\mathcal{S}_k}(f) + \frac{2\|f-\Breve{f}\|_\cD\tau\sqrt{\log(K/\delta)}}{\sqrt{K}}.
\]
We finish the proof.
\end{proof}

\begin{proof}[Proof of Lemma \ref{lemma:wild_optmism_bound_empirical_process}]
    We only prove that $\cT_{\cS_k}^{\varepsilon}(\|\tilde{f}^k_{\rho}-\Breve{f}\|_{\cS_k})=\tilde{\Opt}_{\cS_k}^{\rho}(\tilde{f}^k_{\rho})$. The proof for $\cU_{\cS_k}^{\varepsilon}(\|\wcheck{f}^k_{\rho}-\Breve{f}\|_{\cS_k})=\wcheck{\Opt}_{\cS_k}^{\rho}(\wcheck{f}^k_{\rho})$ follows the same argument.
    By the construction of the wild responses, if we set $\tilde{y}^k_i=\Breve{f}(x_i)+\rho\varepsilon^k_iv^k_i$,  we have
    \begin{align*}
        \argmin_{f\in\cF}\cbr{\frac{1}{2n}\sum_{i=1}^{n}(\tilde{y}^k_i-f(x_i))^2}=\argmin_{f\in\cF}\cbr{\frac{1}{2}\|f-\Breve{f}\|_{\cS_k}^2-\rho\varepsilon_iv^k_i(f(x_i)-\Breve{f}(x_i))}.
    \end{align*}
    By a shell argument, define $\cF(r)=\cbr{f\in\cF;\|f-\Breve{f}\|_{\cS_k}=r}$, we have
    \begin{align*}
    \min_{f\in\cF}\cbr{\frac{1}{2}\|f-\Breve{f}\|_{\cS_k}^2-\rho\frac{1}{n^{\beta}}\sum_{i\in\cS_k}\varepsilon_iv^k_i(f(x_i)-\Breve{f}(x_i))}=&\min_{r\ge 0}\min_{f\in\cF(r)}\cbr{\frac{r^2}{2}-\rho\frac{1}{n^{\beta}}\sum_{i\in\cS_k}\varepsilon_iv^k_i(f(x_i)-\Breve{f}(x_i))}\\
    =&\min_{r\ge 0}\cbr{\frac{r^2}{2}-\rho\cT_{\cS_k}^\varepsilon(r)}.
    \end{align*}
    On the other hand, the LHS is minimized at $\tilde{f}^k_\rho$, so the RHS is minimized at $\tilde{r}_{\cS_k}^\rho=\|\tilde{f}^k_\rho-\Breve{f}\|_{\cS_k}$. We have that 
    \[
    \text{LHS}=\frac{\|\tilde{f}^k_\rho-\Breve{f}\|_{\cS_k}^2}{2}-\rho\tilde{\Opt}^\rho_{\cS_k}(\tilde{f}^k_\rho)=\text{RHS}=\frac{(\tilde{r}_{\cS_k}^\rho)^2}{2}-\rho\cT_{\cS_k}^\varepsilon(\|\tilde{f}^k_\rho-\Breve{f}\|_{\cS_k}).
    \]
    Comparing both sides yields the proof.
\end{proof}

\begin{proof}[Proof of Lemma \ref{lemma:norm_equivalence}]
    For any $h\in\DD(\cF)$, we expand $h$ by its Fourier series and $h(x)=\sum_{k\in\ZZ}\hat{h}(k)e^{2\pi\ib kx}$. Since the infinite dimensional operator $\Sigma=\EE_{x\sim\nu}[\Phi_{\infty}\Phi_{\infty}^H]$ is non-compact, we have to truncate and approximate it by some compact ones. Specifically, we consider $h_N=\sum_{|k|\le N}\hat{h}(k)e^{2\pi\ib kx}$ and $\Sigma_N=\EE_{x\sim\nu}[\Phi_{N}(x)\Phi_{N}(x)^H]$. As for the remainder, we have
    \[
    \sum_{|k|>N}|\hat{h}(k)|^2\le \frac{4M_v^2}{2v-1}\frac{1}{N^{2v-1}}, \forall h\in\DD(\cF).
    \]
    By Lemma \ref{lemma:matrix_bernstein}, for any $\delta>0$, with probability at least $1-\delta$, we have
    \[
    \|\frac{1}{n^{\beta}}\sum_{i\in\cS}\Phi_N(x_i)\Phi_N(x_i)^H-\Sigma_N\|_{op}\le 2\sqrt{\frac{(2N\|\Sigma_N\|_{op}+\|\Sigma_N\|_{op}^2)\log(2N/\delta)}{n^{\beta}}}+\frac{2N\log(2N/\delta)}{3n^{\beta}}
    \]
    
    When $\frac{2N\log (2N/\delta)}{n^\beta}\le1$, we have
    \[
    \bignorm{\frac{1}{n^{\beta}}\sum_{i\in\cS}\Phi_N(x_i)\Phi_N(x_i)^H-\Sigma_N}_{op}\le3\widebar{w}\sqrt{\frac{2N\log(2N/\delta)}{n^{\beta}}}.
    \]
    Therefore, we have that with probability at least $1-\delta$,\[
    h_N^H(\frac{1}{n^{\beta}}\sum_{i\in\cS}\Phi_N(x_i)\Phi_N(x_i)^H)h_N\le h_N^H(\Sigma_N+3\widebar{w}\sqrt{\frac{2N\log(2N/\delta)}{n^{\beta}}}I_N)h_N
    \]
Following the same argument, we have that with probability at least $1-\delta$,\[
|h_N(\frac{1}{n}\sum_{i=1}^{n}\Phi_N(x_i)\Phi_N(x_i)^H)h_N-h_N\Sigma_Nh_N|\le 3\widebar{w}\sqrt{\frac{2N\log(2N/\delta)}{n}}h_N^HI_Nh_N.
\]
    By Parserval's identity, we notice $h_N^H\Sigma_Nh_N=\int_{\Omega}|h_N(x)|^2d\nu(x)$. Since $\Sigma_N$ is a truncation of $\Sigma$, then, by Lemma \ref{lemma:Toeplitz}, we have
    \[
    \lambda_{\max}(\Sigma_N)=\|\Sigma_{N}\|_{op}\le \|\Sigma\|_{op}\le \widebar{w}, \lambda_{\min}(\Sigma_N)\ge \inf\cbr{\sigma(\Sigma)}\ge \underline{w}. 
    \]
    Therefore,
    $$\underline{w}\|h\|_{L^2(\mu)}^2\le h_N^H\Sigma_Nh_N=\int_{\Omega}|h_N(x)|^2d\nu(x)=\int_{\Omega}|h_N(x)|^2\frac{d\nu}{d\mu}(x)d\mu(x)\le \widebar{w}\|h\|_{L^2(\mu)}^2.$$
    Combining these parts together, with probability at least $1-2\delta$, we have
    \begin{align*}
        \|h_N\|_{\cS}^2\le \frac{\widebar{w}+3\widebar{w}\sqrt{\frac{2N\log(2N/\delta)}{n^\beta}}}{\underline{w}-3\widebar{w}\sqrt{\frac{2N\log(2N/\delta)}{n}}}\|h_N\|_{\cD}^2.
    \end{align*}
    By the fact $(a+b)^2\le 2(a^2+b^2)$, and the remainder decay, we have
    \begin{align*}
        \|h\|_{\cS}^2=&\|h_N+\sum_{|k|>N}\hat{h}(k)e_k\|_{\cS}^2\\
        \le& 2\rbr{\|h_N\|_{\cS}^2+\frac{4M_v^2}{2v-1}\frac{1}{N^{2v-1}}}\\
        \le& 2\rbr{\frac{\widebar{w}+3\widebar{w}\sqrt{\frac{2N\log(2N/\delta)}{n^\beta}}}{\underline{w}-3\widebar{w}\sqrt{\frac{2N\log(2N/\delta)}{n}}}\|h_N\|_{\cD}^2+\frac{4M_v^2}{2v-1}\frac{1}{N^{2v-1}}}\\
        \le&2\frac{\widebar{w}+3\widebar{w}\sqrt{\frac{2N\log(2N/\delta)}{n^\beta}}}{\underline{w}-3\widebar{w}\sqrt{\frac{2N\log(2N/\delta)}{n}}}\|h-\sum_{|k|>N}\hat{h}(k)e_k\|_{\cD}^2+\frac{8M_v^2}{2v-1}\frac{1}{N^{2v-1}}\\
        \le&4\frac{\widebar{w}+3\widebar{w}\sqrt{\frac{2N\log(2N/\delta)}{n^\beta}}}{\underline{w}-3\widebar{w}\sqrt{\frac{2N\log(2N/\delta)}{n}}}\|h\|_{\cD}^2+\rbr{2\frac{\widebar{w}+3\widebar{w}\sqrt{\frac{2N\log(2N/\delta)}{n^\beta}}}{\underline{w}-3\widebar{w}\sqrt{\frac{2N\log(2N/\delta)}{n}}}+1}\frac{8M_v^2}{2v-1}\frac{1}{N^{2v-1}}.
    \end{align*}
    We finish the proof.
\end{proof}

\begin{proof}[Proof of Theorem \ref{thm:risk_bound_fixed_design}]
    We prove Theorem~\ref{thm:risk_bound_fixed_design} by combining several lemmas. The proofs of these lemmas are deferred to Appendix~\ref{app:proofs_Appendix_risk_guarantee}.
First,
    \begin{lemma}\label{lemma:Opt*<Z}
        For any radius $r\ge\hat{r}_n=\|\Breve{f}-f^*\|_{\cD}$, $\forall t>0$, with probability at least $1-e^{-t^2}$, 
        \[
        \Opt^*_{\cD}(\Breve{f})\le2\EE_{\tilde{w},\varepsilon}[\cZ^\varepsilon_{\cD}(r)]+\frac{2\sqrt{2}r\tau t}{\sqrt{n}}.
        \]
    \end{lemma}
    With Lemma \ref{lemma:Opt*<Z}, we can just focus on the analysis of $\EE_{\tilde{w},\varepsilon}[\cZ^\varepsilon_{\cD}(r)]$. Note that in $\cZ_{\cD}^\varepsilon$, the center is $f^*$, which is unknown to us. Our next lemma upper bound $\cZ_{\cD}^\varepsilon(r)$ by another empirical process centered at $\hat{f}$, which is known to us.
    \begin{lemma}\label{lemma:Bound_EE[Z]<EE[Q]}
        Define $\cQ_{\cD}^\varepsilon(r):=\sup_{f\in\BB_r(\Breve{f};\cD)}\cbr{\frac{1}{n}\sum_{i=1}^{n}\varepsilon_i\tilde{w}_i(f(x_i)-\Breve{f}(x_i))}$. For any $r>0$, we have $$\EE_{\tilde{w},\varepsilon}[\cZ_{\cD}^\varepsilon(r)]\le \EE_{\tilde{w},\varepsilon}[\cQ_{\cD}^\varepsilon(r+\hat{r}_n)].$$
        Therefore, if $r\ge\hat{r}_n=\|\Breve{f}-f^*\|_n$, we have $\EE_{\tilde{w},\varepsilon}[\cZ_{\cD}^\varepsilon(r)]\le \EE_{\tilde{w},\varepsilon}[\cQ_{\cD}^\varepsilon(2r)].$
    \end{lemma}
    For $\cQ_{\cD}^\varepsilon(r)$, we could apply Lemma \ref{lemma:concentration_Lip} again to obtain that for any $t>0$, with probability at least $1-e^{-t^2}$,
    \[
    \EE_{\tilde{w},\varepsilon}[\cQ_{\cD}^\varepsilon(2r)]\le \cQ_{\cD}^\varepsilon(2r)+\frac{2\sqrt{2}r\tau t}{\sqrt{n}}.
    \]
    Combining Lemma \ref{lemma:Opt*<Z} and \ref{lemma:Bound_EE[Z]<EE[Q]}, we get that for any $r\ge\hat{r}_n$, with probability at least $1-2e^{-t^2}$,
    \[
    \Opt^*_{\cD}(\Breve{f})\le 2\cQ_{\cD}^{\varepsilon}(2r)+r\frac{6\sqrt{2}\tau t}{\sqrt{n}}.
    \]
    Now, we connect $\cQ_{\cD}^{\varepsilon}$ with the empirical processes defined in Subsection \ref{subsec:bounding_risk_fixed_design} in Lemma \ref{lemma:bound_Q_by_W+H}.
    \begin{lemma}\label{lemma:bound_Q_by_W+H}
        For any $r>0$, for any $t>0$, with probability at least $1-2e^{-t^2}$,
        \[
        \cQ_{\cD}^\varepsilon(r)\le \frac{1}{2}\rbr{\cW_\cD^\varepsilon(r)+\cH_{\cD}^\varepsilon(r)}+r\frac{2\sqrt{2}\tau t}{\sqrt{n}}+\frac{1}{2}V_{r}(\bar{f}),
        \]
        where \begin{align*}V_r(\bar{f})&=\sup_{f\in\BB_r(\Breve{f};\cD)}\cbr{\frac{1}{n}\sum_{i=1}^{n}\varepsilon_i(\widebar{f}(x_i)-f^*(x_i))(f(x_i)-\Breve{f}(x_i))}\\
        &+\sup_{f\in\BB_r(\Breve{f};\cD)}\cbr{\frac{1}{n}\sum_{i=1}^{n}\varepsilon_i(\widebar{f}(x_i)-f^*(x_i))(\Breve{f}(x_i)-f(x_i))}
        \end{align*}
        is the pilot error term.     
    \end{lemma}
With Lemma \ref{lemma:bound_Q_by_W+H}, we return to the proof of Theorem \ref{thm:risk_bound_fixed_design}. Denote $h_1$ and $h_2$ as the functions in $\BB_{2r}(\Breve{f};\cD)$ that achieves the supremum in $\cW_\cD^\varepsilon(2r)$ and $\cH_\cD^\varepsilon(2r)$, respectively. Then, by Lemma \ref{lemma:unbias}, we have that
\begin{align*}
    \cW_\cD^\varepsilon(2r)=\frac{1}{n}\sum_{i=1}^{n}\varepsilon_iv_i(h_1(x_i)-\Breve{f}(x_i))=A_n(h_1)\le \frac{1}{K}\sum_{i=1}^{K}B_{\cS_k}(h_1)+\frac{8\|h_1-\Breve{f}\|_{\cD}\tau\sqrt{\log(K/\delta)}}{\sqrt{K}}
\end{align*}
By Lemma \ref{lemma:wild_optmism_bound_empirical_process}, we have that
\begin{align*}
    B_{\cS_k}(h_1)\le \cT_{\cS_k}^\varepsilon(\|\tilde{f}^k_{\rho_1^k}-\Breve{f}\|_{\cS_k})=\tilde{\Opt}^{\rho_1^k}_{\cS_k}(\tilde{f}^k_{\rho_1^k}).
\end{align*}
Thus, we have
\[
\cW_\cD^\varepsilon(2r)\le \frac{1}{K}\sum_{k=1}^{K}\tilde{\Opt}^{\rho_1^k}_{\cS_k}(\tilde{f}^k_{\rho_1^k})+\frac{16r\tau\sqrt{\log(K/\delta)}}{\sqrt{K}}.
\]
We apply the same argument to $\cH_{\cD}^\varepsilon$ to get
\[
\cH_\cD^\varepsilon(2r)\le \frac{1}{K}\sum_{k=1}^{K}\wcheck{\Opt}^{\rho_2^k}_{\cS_k}(\wcheck{f}^k_{\rho_2^k})+\frac{16r\tau\sqrt{\log(K/\delta)}}{\sqrt{K}}.
\]
Adding these two inequalities together, we finish the proof.
\end{proof}
\begin{proof}[Proof of Theorem \ref{thm:bound_pilot_error}]
    For the pilot error term $V_{2r}(\widebar{f})$, we first consider the term $$\sup_{f\in\BB_{2r}(\Breve{f};\cD)}\cbr{\frac{1}{n}\sum_{i=1}^{n}\varepsilon_i(\widebar{f}(x_i)-f^*(x_i))(f(x_i)-\Breve{f}(x_i))}.$$ We define
    the triangle pulse basis function $\Lambda_\epsilon(x)=\max\cbr{0,1-\frac{|x|}{\epsilon}}$. For any vector $g^1\in\RR^n$ such that $\|g^1\|_{n}\le 2r$, we define the following interpolation function. 
    \[
    f_\epsilon(x)=\sum_{j=1}^{n}g^1_j\Lambda_{\epsilon}(x-x_j),
    \]
    where $\epsilon$ is some small number to be determined later.
    We compute the Fourier coefficient $\hat{f}_\epsilon(k)$ by directly evaluating the integral. By definition, we have:
\begin{align*}
    \hat{f}_\epsilon(k) = \int_{0}^{1} f_\epsilon(x) e^{-2\pi \ib k x} dx
    = \sum_{j=1}^{n} g^1_j \int_{0}^{1} \Lambda_\epsilon(x-x_j) e^{-2\pi \ib k x} dx.
\end{align*}
Assuming $\epsilon$ is sufficiently small such that the support of each basis function $\Lambda_\epsilon(x-x_j)$ is strictly strictly within $[0, 1]$ (i.e., $\epsilon < \min(x_j, 1-x_j)$), we can evaluate the integral for each $j$ by applying the change of variables $u = x - x_j$:
\begin{align*}
    \int_{0}^{1} \Lambda_\epsilon(x-x_j) e^{-2\pi \ib k x} dx 
    = \int_{-x_j}^{1-x_j} \Lambda_\epsilon(u) e^{-2\pi \ib k (u+x_j)} du
    = e^{-2\pi \ib k x_j} \int_{-\epsilon}^{\epsilon} \left(1-\frac{|u|}{\epsilon}\right) e^{-2\pi \ib k u} du.
\end{align*}
Using Euler's formula $e^{-2\pi \ib k u} = \cos(2\pi k u) - \ib\sin(2\pi k u)$, and noting that $(1-|u|/\epsilon)\sin(2\pi k u)$ is an odd function over the symmetric interval $[-\epsilon, \epsilon]$, its integral evaluates to zero. Thus, we are left with the even cosine term:
\begin{align*}
    \int_{-\epsilon}^{\epsilon} \left(1-\frac{|u|}{\epsilon}\right) e^{-2\pi \ib k u} du 
    &= 2 \int_{0}^{\epsilon} \left(1-\frac{u}{\epsilon}\right) \cos(2\pi k u) du.
\end{align*}
We proceed with integration by parts. Let $v = 1 - \frac{u}{\epsilon}$ and $dw = \cos(2\pi k u)du$, which gives $dv = -\frac{1}{\epsilon}du$ and $w = \frac{\sin(2\pi k u)}{2\pi k}$. Applying $\int v dw = vw - \int w dv$:
\begin{align*}
    2 \int_{0}^{\epsilon} \left(1-\frac{u}{\epsilon}\right) \cos(2\pi k u) du 
    &= 2 \left[ \left(1-\frac{u}{\epsilon}\right) \frac{\sin(2\pi k u)}{2\pi k} \right]_{0}^{\epsilon} - 2 \int_{0}^{\epsilon} \frac{\sin(2\pi k u)}{2\pi k} \left(-\frac{1}{\epsilon}\right) du\\
    &= 0 + \frac{1}{\pi k \epsilon} \int_{0}^{\epsilon} \sin(2\pi k u) du \\
    &= \frac{1}{\pi k \epsilon} \left[ -\frac{\cos(2\pi k u)}{2\pi k} \right]\Bigg|_{0}^{\epsilon} = \frac{1 - \cos(2\pi k \epsilon)}{2\pi^2 k^2 \epsilon}.
\end{align*}
Using the identity $1 - \cos(\theta) = 2\sin^2(\theta/2)$, we rewrite this as $\frac{1 - \cos(2\pi k \epsilon)}{2\pi^2 k^2 \epsilon}=\epsilon \left( \frac{\sin(\pi \epsilon k)}{\pi \epsilon k} \right)^2.$
Substituting this result back into the sum, we obtain the explicit expression for $\hat{f}(k)$:
\begin{equation*}
    \hat{f}_\epsilon(k) = \epsilon \left( \frac{\sin(\pi \epsilon k)}{\pi \epsilon k} \right)^2 \sum_{j=1}^{n} g^1_j e^{-2\pi \ib k x_j}.
\end{equation*}
Thus, we have $$|k|^v|\hat{f}_\epsilon(k)|\le \epsilon|k|^vn\frac{1}{\pi^2\epsilon^2k^2}\le2 nr \epsilon^{1-v}\max_{z\ge 0}\{z^v\min(1,\frac{1}{\pi^2z^2})\}.$$

Since $\max_{z\ge 0}\{z^v\min(1,\frac{1}{\pi^2z^2})\}=\frac{1}{\pi^v}$, we set $\epsilon=\min\{|x_i-x_j|,(\frac{M_v\pi^v}{2nr})^{\frac{1}{1-v}}\}.$
Therefore, we find a function $f_\epsilon\in\cF$ such that $f_\epsilon(x_j)=g^1_j,\forall j$. Thus, we consider the function $\varphi=\hat{f}+f_\epsilon$, $\varphi\in\BB_{2r}(\Breve{f};\cD)$. Since this argument holds for any $\|g^1\|\le 2r$, we define $g^1$ to be
\[
g^1:=\argmax_{\|g\|_n\le 2r}\cbr{\frac{1}{n}\sum_{i=1}^{n}g_j(\widebar{f}(x_i)-f^*(x_i))}.
\]
By Lemma \ref{lemma:dual_norm}, $g^1_i=2r\frac{\bar{f}(x_i)-f^*(x_i)}{\|\widebar{f}-\Breve{f}\|_{\cD}}$. Then, we have that
\begin{align*}
    &\sup_{f\in\BB_{2r}(\Breve{f};\cD)}\cbr{\frac{1}{n}\sum_{i=1}^{n}\varepsilon_i(\widebar{f}(x_i)-f^*(x_i)-w_i)(f(x_i)-\Breve{f}(x_i))}\\
    &\ge\frac{1}{n}\sum_{i=1}^{n}\varepsilon_i(\widebar{f}(x_i)-f^*(x_i)-w_i)f_\epsilon(x_i)\\
    &=\frac{1}{n}\sum_{i=1}^{n}\varepsilon_i(\widebar{f}(x_i)-f^*(x_i)g^1_i-\frac{1}{n}\sum_{i=1}^{n}\varepsilon_iw_if_\epsilon(x_i)\\
    &=\max_{\|g\|_n\le 2r}\cbr{\frac{1}{n}\sum_{i=1}^{n}g_j(\widebar{f}(x_i)-f^*(x_i))}-\frac{1}{n}\sum_{i=1}^{n}\varepsilon_iw_if_\epsilon(x_i)\\
    &\ge\sup_{f\in\BB_{2r}(\Breve{f};\cD)}\cbr{\frac{1}{n}\sum_{i=1}^{n}\varepsilon_i(\widebar{f}(x_i)-f^*(x_i))(f(x_i)-\Breve{f}(x_i))}-\frac{1}{n}\sum_{i=1}^{n}\varepsilon_iw_if_\epsilon(x_i)
\end{align*}
For the term $\frac{1}{n}\sum_{i=1}^{n}\varepsilon_iw_if_\epsilon(x_i)$, notice that $f_\epsilon$ is independent of $\varepsilon_i$ and $w_i$. By Lemma \ref{lemma:concentration_Lip}, we have that with probability at least $1-\delta$,
\[
\frac{1}{n}\sum_{i=1}^{n}\varepsilon_iw_if_\epsilon(x_i)\le \frac{4r\tau\sqrt{\log(1/\delta)}}{\sqrt{n}}.
\]
Thus, we get that with probability at least $1-\delta$,
\begin{align*}
&\sup_{f\in\BB_{2r}(\Breve{f};\cD)}\cbr{\frac{1}{n}\sum_{i=1}^{n}\varepsilon_i(\widebar{f}(x_i)-f^*(x_i))(f(x_i)-\Breve{f}(x_i))}\\
\le&\sup_{f\in\BB_{2r}(\Breve{f};\cD)}\cbr{\frac{1}{n}\sum_{i=1}^{n}\varepsilon_i(\widebar{f}(x_i)-f^*(x_i)-w_i)(f(x_i)-\Breve{f}(x_i))}+\frac{4r\tau\sqrt{\log(1/\delta)}}{\sqrt{n}}.
\end{align*}
Similarly, we have 
\begin{align*}
&\sup_{f\in\BB_{2r}(\Breve{f};\cD)}\cbr{\frac{1}{n}\sum_{i=1}^{n}\varepsilon_i(\widebar{f}(x_i)-f^*(x_i))(\Breve{f}(x_i)-f(x_i))}\\
\le&\sup_{f\in\BB_{2r}(\Breve{f};\cD)}\cbr{\frac{1}{n}\sum_{i=1}^{n}\varepsilon_i(\widebar{f}(x_i)-f^*(x_i)-w_i)(\Breve{f}(x_i)-f(x_i))}+\frac{4r\tau\sqrt{\log(1/\delta)}}{\sqrt{n}}.
\end{align*}
Combining them together, we have
\[
V_{2r}(\bar{f})\le \cW_\cD^\varepsilon(2r)+\cH_\cD^\varepsilon(2r)+\frac{8r\tau\sqrt{\log(1/\delta)}}{\sqrt{n}}.
\]
with probability at least $1-2\delta$.
\end{proof}
\begin{proof}[Proof of Theorem \ref{thm:bound_pilot_error_v=1}]
For the pilot error term $V_{2r}(\widebar{f})$, we first consider the term $$\sup_{f\in\BB_{2r}(\Breve{f};\cD)}\cbr{\frac{1}{n}\sum_{i=1}^{n}\varepsilon_i(\widebar{f}(x_i)-f^*(x_i))(f(x_i)-\Breve{f}(x_i))}.$$ Since for all $f\in\BB_{2r}(\Breve{f};\cD)$, $f-\Breve{f}\in\BB_{2r}(0;\cD)$, then we have
\[
\sup_{f\in\BB_{2r}(\Breve{f};\cD)}\cbr{\frac{1}{n}\sum_{i=1}^{n}\varepsilon_i(\widebar{f}(x_i)-f^*(x_i))(f(x_i)-\Breve{f}(x_i))}\le \sup_{g\in\BB_{2r}(0;\cD)}\cbr{\frac{1}{n}\sum_{i=1}^{n}\varepsilon_i(\widebar{f}(x_i)-f^*(x_i))g(x_i)}.
\]
We use $g_1^*$ to denote the maximizer. $g_1^*$ is independent of the noise sequence $\cbr{w_i}_{i=1}^{n}$.

Now, we construct a function $h$ such that $h(x_j)=g_1^*(x_j)$, $\hat{h}(k)\le \rho_k:= \frac{M_1}{|k|} - |\hat{f}(k)|$.
By the assumption that $|\hat{f}(k)| < \frac{M_1}{|k|}$, we have $\rho_k > 0$. Since $\hat{f}\in C^1$, then its Fourier coefficients is $o(1/|k|)$. Thus, the residual capacity $\rho_k$ scales as $\Theta(1/|k|)$ and the series $\sum_{k \neq 0} \rho_k$ diverges.

We construct $h(x)$ as a linear combination of truncated residual kernels. Define the truncated kernel up to frequency $N$:
\[
K_N(x) = \sum_{0 < |k| \le N} \rho_k e^{i 2\pi k x}
\]
We propose the function form $h_N(x) = \sum_{j=1}^{n} c_j K_N(x - x_j)$. The Fourier coefficients of $h_N(x)$ are given by:
\[
\hat{h}_N(k) = \begin{cases} 
\rho_k \sum_{j=1}^{n} c_j e^{-\ib 2\pi k x_j}, & \text{if } 0 < |k| \le N \\
0, & \text{otherwise}
\end{cases}
\]
By the triangle inequality, for any $k$, we can bound the magnitude:
\[
|\hat{h}_N(k)| \le \rho_k \sum_{j=1}^{n} |c_j| = \rho_k \|c\|_1.
\]
To satisfy the condition $|\hat{h}_N(k)| \le \rho_k$, it is sufficient to guarantee that $\|c\|_1 \le 1$.

Now we enforce the interpolation condition $h_N(x_i) =y_i$, where $y_i=g_1^*(x_i)$ for all $i \in [n]$. This yields a linear system $A^{(N)} c = y$, where the matrix entries are $A^{(N)}_{i,j} = K_N(x_i - x_j)$.

Let us analyze the behavior of the matrix $A^{(N)}$ as $N \to \infty$:
\begin{itemize}
    \item \textbf{Diagonal entries:} For $i = j$, $A^{(N)}_{i,i} = K_N(0) = \sum_{0 < |k| \le N} \rho_k$. Since the series diverges, $A^{(N)}_{i,i} \to +\infty$ as $N \to \infty$.
    \item \textbf{Off-diagonal entries:} For $i \neq j$, $A^{(N)}_{i,j} = \sum_{0 < |k| \le N} \rho_k e^{i 2\pi k (x_i - x_j)}$. Since $x_i \neq x_j$ and $\rho_k$ decreases monotonically to $0$, by Dirichlet's test, this series converges to a finite value $L_{i,j} < \infty$.
\end{itemize}
Because the diagonal entries grow unbounded but the off-diagonal entries remain bounded, the matrix $A^{(N)}$ becomes strictly diagonally dominant for sufficiently large $N$. Therefore, $(A^{(N)})^{-1}$ exists, and its $\ell_1$ matrix norm vanishes: $\lim_{N \to \infty} \|(A^{(N)})^{-1}\|_1 = 0$.

Therefore, the coefficient vector $c = (A^{(N)})^{-1} y$ satisfies:
\[
\lim_{N \to \infty} \|c\|_1 \le \lim_{N \to \infty} \|(A^{(N)})^{-1}\|_1 \|y\|_1 = 0.
\]
Since $\|y\|_1$ is a fixed finite constant determined by $g_1^*$, there exists a sufficiently large integer $N^*$ such that $\|c\|_1 \le 1$. 

By choosing $h(x) = h_{N^*}(x)$, we have successfully constructed a continuous function that exactly interpolates $g_1^*$ at all sample points (i.e., $h(x_i) = g_1^*(x_i)$), and satisfies that $$|\hat{h}(k)| \le \rho_k \|c\|_1 \le \rho_k = \frac{M_1}{|k|} - |\hat{f}(k)|.$$ This completes the construction.

For function $h$, we have $\hat{f}+h\in\BB_{2r}(\Breve{f};\cD)$ and that $\hat{f}+h\in\cF$. Thus, we have
\begin{align*}
    &\sup_{f\in\BB_{2r}(\Breve{f};\cD)}\cbr{\frac{1}{n}\sum_{i=1}^{n}\varepsilon_i(\widebar{f}(x_i)-f^*(x_i)-w_i)(f(x_i)-\Breve{f}(x_i))}\\
 \ge&\cbr{\frac{1}{n}\sum_{i=1}^{n}\varepsilon_i(\widebar{f}(x_i)-f^*(x_i)-w_i)h(x_i))}\\
 = &\cbr{\frac{1}{n}\sum_{i=1}^{n}\varepsilon_i(\widebar{f}(x_i)-f^*(x_i)-w_i)g_1^*(x_i))}\\
 =&\sup_{f\in\BB_{2r}(\Breve{f};\cD)}\cbr{\frac{1}{n}\sum_{i=1}^{n}\varepsilon_i(\widebar{f}(x_i)-f^*(x_i))(f(x_i)-\Breve{f}(x_i))}-\frac{1}{n}\sum_{i=1}^{n}w_ig_1^*(x_i).
\end{align*}
Since we have shown, $g_1^*$ is independent of $\cbr{w_i}_{i=1}^{n}$, thus, by Lemma \ref{lemma:concentration_Lip}, for any $\delta>0$, with probability at least $1-\delta$,
\[
\frac{1}{n}\sum_{i=1}^{n}w_ig_1^*(x_i)\le \frac{4\tau r\sqrt{\log(1/\delta)}}{\sqrt{n}}.
\]
Thus, we conclude 
\begin{align*}
&\sup_{f\in\BB_{2r}(\Breve{f};\cD)}\cbr{\frac{1}{n}\sum_{i=1}^{n}\varepsilon_i(\widebar{f}(x_i)-f^*(x_i))(f(x_i)-\Breve{f}(x_i))}\\
\le &\sup_{f\in\BB_{2r}(\Breve{f};\cD)}\cbr{\frac{1}{n}\sum_{i=1}^{n}\varepsilon_i(\widebar{f}(x_i)-f^*(x_i)-w_i)(f(x_i)-\Breve{f}(x_i))}+\frac{4\tau r\sqrt{\log(1/\delta)}}{\sqrt{n}}.
\end{align*}
Following the same logic, we have
\begin{align*}
&\sup_{f\in\BB_{2r}(\Breve{f};\cD)}\cbr{\frac{1}{n}\sum_{i=1}^{n}\varepsilon_i(\widebar{f}(x_i)-f^*(x_i))(\Breve{f}(x_i)-f(x_i))}\\
\le &\sup_{f\in\BB_{2r}(\Breve{f};\cD)}\cbr{\frac{1}{n}\sum_{i=1}^{n}\varepsilon_i(\widebar{f}(x_i)-f^*(x_i)-w_i)(\Breve{f}(x_i)-f(x_i))}+\frac{4\tau r\sqrt{\log(1/\delta)}}{\sqrt{n}}.
\end{align*}
Adding these two inequalities together, we finish the proof.
\end{proof}

\begin{proof}[Proof of Theorem \ref{thm:bounding_hatr_n}]
    First, by the basic inequality, we have 
    \[
    \hat{r}_n^2\le \frac{2}{n}\sum_{i=1}^{n}w_i(\Breve{f}(x_i)-f^*(x_i))\le 2\sup_{f\in\BB_{\hat{r}_n}(f^*;\cD)}\cbr{\frac{1}{n}\sum_{i=1}^{n}w_i(f(x_i)-f^*(x_i))}.
    \]
    Applying Lemma \ref{lemma:concentration_Lip}, we have that with probability at least $1-e^{-t^2}$,
    \[
    \sup_{f\in\BB_{\hat{r}_n}(f^*;\cD)}\cbr{\frac{1}{n}\sum_{i=1}^{n}w_i(f(x_i)-f^*(x_i))}\le \EE_{w}\sbr{\sup_{f\in\BB_{\hat{r}_n}(f^*;\cD)}\cbr{\frac{1}{n}\sum_{i=1}^{n}w_i(f(x_i)-f^*(x_i))}}+\frac{2\hat{r}_n\tau t}{\sqrt{n}}.
    \]
    Then, by a symmetrization argument, we have that
    \begin{align*}
    \EE_{w}\sbr{\sup_{f\in\BB_{\hat{r}_n}(f^*;\cD)}\cbr{\frac{1}{n}\sum_{i=1}^{n}w_i(f(x_i)-f^*(x_i))}}&=\EE_{w}\sbr{\sup_{f\in\BB_{\hat{r}_n}(f^*;\cD)}\cbr{\frac{1}{n}\sum_{i=1}^{n}(w_i-\EE[w_i'|x_i])(f(x_i)-f^*(x_i))}}\\
    &\le\EE_{w,w'}\sbr{\sup_{f\in\BB_{\hat{r}_n}(f^*;\cD)}\cbr{\frac{1}{n}\sum_{i=1}^{n}(w_i-w_i')(f(x_i)-f^*(x_i))}}\\
    &=2\EE_{w,w',\varepsilon}[\cZ_\cD^\varepsilon(\hat{r}_n)].
    \end{align*}
By a peeling argument, we have the following lemma.
\begin{lemma}\label{lemma:peeling}
    For any scalar $s\ge 3$, we have
    \[
   \EE_{w,w',\varepsilon}[\cZ_\cD^\varepsilon(r)] \le \cZ_\cD^\varepsilon([1+\frac{1}{s}]r)+\frac{4\tau}{s}r^2.
    \]
    uniformly for all $r\ge \frac{s^2}{\sqrt{n}}$, with probability at least $1-e^{-s^2}$.
\end{lemma}
The proof of Lemma \ref{lemma:peeling} is deferred to  Appendix \ref{app:proofs_Appendix_risk_guarantee}.

With Lemma \ref{lemma:peeling}, we fix some $s\ge 3$; we either have $\hat{r}_n\le \frac{s^2}{\sqrt{n}}$ or $\hat{r}_n>\frac{s^2}{\sqrt{n}}$. We focus on the latter one. Then, by Lemma \ref{lemma:peeling}, we have that with probability at least $1-e^{-t^2}-e^{-s^2}$,
\[
\hat{r}_n^2\le 2\cZ_\cD^\varepsilon([1+\frac{1}{s}]\hat{r}_n)+4\frac{\tau}{s}\hat{r}_n^2+\frac{2\hat{r}_n\tau t}{\sqrt{n}}.
\]
Recall that we define 
\[
\cQ_{\cD}^\varepsilon(r):=\sup_{f\in\BB_r(\Breve{f};\cD)}\cbr{\frac{1}{n}\sum_{i=1}^{n}\varepsilon_i\tilde{w}_i(f(x_i)-\Breve{f}(x_i))}.
\]
Thus, we have
\[
\cZ_\cD^\varepsilon([1+\frac{1}{s}]\hat{r}_n)\le \cQ_{\cD}^\varepsilon([2+\frac{1}{s}]\hat{r}_n).
\]
Now, by Lemma \ref{lemma:bound_Q_by_W+H}, with probability at least $1-3e^{-t^2}-e^{-s^2}$, we have
\[
\cZ_\cD^\varepsilon([1+\frac{1}{s}]\hat{r}_n)\le \cQ_{\cD}^\varepsilon([2+\frac{1}{s}]\hat{r}_n)\le \frac{1}{2}\rbr{\cW_\cD^\varepsilon([2+\frac{1}{s}]\hat{r}_n)+\cH_{\cD}^\varepsilon([2+\frac{1}{s}]\hat{r}_n)}+[2+\frac{1}{s}]\hat{r}_n\frac{2\sqrt{2}\tau t}{\sqrt{n}}+\frac{1}{2}V_{[2+\frac{1}{s}]\hat{r}_n}(\bar{f})
\]

Plugging this back, setting $s=t$, we obtain that if $\hat{r}_n>\frac{t^2}{\sqrt{n}}$, with probability at least $1-4e^{-t^2}$,
\[
\hat{r}_n^2\le \cW_\cD^\varepsilon([2+\frac{1}{t}]\hat{r}_n)+\cH_{\cD}^\varepsilon([2+\frac{1}{t}]\hat{r}_n)+[2+\frac{1}{t}]\hat{r}_n\frac{4\sqrt{2}\tau t}{\sqrt{n}}+V_{[2+\frac{1}{t}]\hat{r}_n}(\bar{f})+4\frac{\tau}{t}\hat{r}_n^2+\frac{2\hat{r}_n\tau t}{\sqrt{n}}.
\]
We finish the proof.
\end{proof}
\begin{proof}[Proof of Lemma \ref{lemma:concave_W_H}]
We only need to prove for $\cW_\cD^\varepsilon$. For any $s,t\ge 0$, let $f_s$, $f_t$ be functions that achieving the suprema in $\cW_\cD^\varepsilon(s)$ and $\cW_\cD^\varepsilon(t)$. For any $\alpha\in[0,1]$, $r=\alpha s+(1-\alpha)t$. Denote $f_r=\alpha f_s+(1-\alpha)f_t$, we then have $\|f_r-\Breve{f}\|_{\cD}\le \alpha s+(1-\alpha)t=r$.
Thue, we have
\begin{align*}
    \alpha\cW_\cD^\varepsilon(s)+(1-\alpha)\cW_\cD^\varepsilon(t)=&\frac{1}{n}\sum_{i=1}^{n}\varepsilon\tilde{w}_i\rbr{\alpha f_s(x_i)+(1-\alpha)f_t(x_i)-\hat{f}(x_i)}\\
    \le&\sup_{f\in\BB_{r}(\hat{f};\cD)}\cbr{\frac{1}{n}\sum_{i=1}^{n}\varepsilon\tilde{w}_i\rbr{f(x_i)-\hat{f}(x_i)}}\\
    =&\cW_\cD^\varepsilon(r).
\end{align*}
Therefore, $\cW_\cD^\varepsilon$ is concave. The same argument holds for $\cH_\cD^\varepsilon$. For any concave function $f$, we have that for any $x_1<x_2<x_3$,
\[
\frac{f(x_3)-f(x_1)}{x_3-x_1}\le \frac{f(x_2)-f(x_1)}{x_2-x_1},
\]
which is standard property from convex analysis \citep{magaril2003convex}. Setting $x_1=0$ and using the fact that $\cW_\cD^\varepsilon(0)=\cH_\cD^\varepsilon(0)=0$, we finish the proof.
\end{proof}
\section{Proofs in Section \ref{sec:guarantee_random_design}}\label{app:proofs_sec:guarantee_random_design}
In this section, we provide proofs regarding the lemmas and theorems in Section \ref{sec:guarantee_random_design}.
\begin{proof}[Proof of Lemma \ref{lemma:critical_radius}]
    Generalizing from $\cE_\cD(\Breve{f})$ to $\cE(\Breve{f})$ typically requires knowledge on the critical radius regarding the underlying function family $\cF$ to bound the local Rademacher complexity \citep{Bousquet2002}. However, as we only assume black-box access to the procedure, we have no knowledge about the critical radius of $\cF$. Therefore, we have to find a surrogate function class with tractable critical radius to approximate $\DD(\cF)^2=\cbr{(f-f')^2:f,f'\in\cF}$.

    For some threshold $N>0$ where the value of $N$ will be determined later, we define the truncated function family
    $\cG_N$ as
    \[
    \cG_N:=\cbr{\rbr{\sum_{|k|\le N, k\neq 0}(\hat{f}(k)-\hat{f'}(k))e^{2\pi\ib kx}}^2:f,f'\in\cF}.
    \]
    That is, the surrogate function family $\cG_N$ truncates the functions in $\cF$ to the lowest $2N$ non-zero frequencies. Since $e^{\ib x}=\cos(x)+\ib\sin(x)$, we know that $\cG_N$ lies in a $8N$ dimensional real linear space. By Lemma \ref{lemma:VC_dim}, the VC dimension of $\cG_N$ $\dim_{VC}(\cG_N)$ is upper bounded by $8N+2$, and the critical radius $\delta_n(\cG_N)$ is upper bounded by $\sqrt\frac{(8N+2)\log(n/(8N+2)))}{n}$.
    
    By the construction of $\cG_N$, we can upper bound the magnitude of $\cG_N$ by
    \[
    \rbr{\sum_{|k|\le N, k\neq 0}(\hat{f}(k)-\hat{f'}(k))e^{2\pi\ib kx}}^2\le \rbr{\sum_{|k|\le N,k\neq 0}|\hat{f}(k)-\hat{f'}(k)|}^2\le \rbr{\sum_{|k|\le N,k\neq 0}\frac{M_v}{|k|^v}}^2.
    \]
    Since $1/2<v\le 1$, we define $\eta_v(N):=\rbr{\sum_{|k|\le N,k\neq 0}\frac{M_v}{|k|^v}}^2$ and have that
    \[
\rbr{\sum_{|k|\le N,k\neq 0}\frac{M_v}{|k|^v}}^2=
\begin{dcases}
\frac{4M_v^2}{(1-v)^2}N^{2-2v}, & 1/2<v<1,\\
4M_v^2(\log N)^2, & v=1.
\end{dcases}
\]
Now, we define the normalized function class $\tilde{\cG}_N$ as
\[
\tilde{\cG}_N:=\cbr{\frac{g_N}{\eta_v(N)}: g_N\in\cG_N}.
\]
By construction, every function in $\tilde{\cG}_N$ takes values in $[0,1]$. Consequently, the critical radius of the normalized class satisfies
\[
\delta_n(\tilde{\cG}_N)=\frac{\delta_n(\cG_N)}{\eta_v(N)}.
\]
For any $\delta>0$, we apply Lemma \ref{lemma:critical_radius_rakhlin}, to obtain that with probability at least $1-\delta$,
\begin{align*}
\EE_{x\sim\nu}\sbr{\frac{1}{\eta_v(N)}\rbr{\sum_{|k|\le N, k\neq 0}(\hat{f}(k)-\hat{f'}(k))e^{2\pi\ib kx}}^2}&\le \frac{2}{n}\sum_{i=1}^{n}\frac{1}{\eta_v(N)}\rbr{\sum_{|k|\le N, k\neq 0}(\hat{f}(k)-\hat{f'}(k))e^{2\pi\ib kx_i}}^2\\
&+c[\delta_n(\tilde{\cG}_N)]^2+\frac{c'(\log(1/\delta)+\log\log n)}{n},
\end{align*}
uniformly over $g\in\tilde{\cG}_N$. Thus, with probability at least $1-\delta$, uniformly over $\cG_N$, we have
\begin{align}\label{ineq:critical_ineq_G_N}
\EE_{x\sim\nu}\sbr{\rbr{\sum_{|k|\le N, k\neq 0}(\hat{f}(k)-\hat{f'}(k))e^{2\pi\ib kx}}^2}&\le \frac{2}{n}\sum_{i=1}^{n}\rbr{\sum_{|k|\le N, k\neq 0}(\hat{f}(k)-\hat{f'}(k))e^{2\pi\ib kx_i}}^2\nonumber\\
&+c[\delta_n(\cG_N)]^2+\frac{c'\eta_v(N)(\log(1/\delta)+\log\log n)}{n}.
\end{align}
On the other hand, we have that for all $(f-f')^2\in\DD(\cF)^2$,
\begin{align*}
    \EE_{x\sim\nu}\sbr{(f(x)-f'(x))^2}=&\int_{0}^{1}|f(x)-f'(x)|^2d\nu(x)\le \widebar{w}\int_{0}^{1}|f(x)-f'(x)|^2d\mu(x)=\widebar{w}\sum_{k\in\ZZ}|\hat{f}(k)-\hat{f'}(k)|^2\\
    =&\widebar{w}\rbr{\sum_{|k|\le N,k\neq 0}|\hat{f}(k)-\hat{f'}(k)|^2+|\hat{f}(0)-\hat{f'}(0)|^2+\sum_{|k|>N}|\hat{f}(k)-\hat{f'}(k)|^2}\\
    =&\widebar{w}\int_{0}^{1}\big|\sum_{|k|\le N, k\neq 0}(\hat{f}(k)-\hat{f'}(k))e^{2\pi\ib kx}\big|^2d\mu(x)+\widebar{w}|\hat{f}(0)-\hat{f'}(0)|^2+\widebar{w}\frac{4M_v^2}{2v-1}\frac{1}{N^{2v-1}}\\
    \le&\frac{\widebar{w}}{\underline{w}}\EE_{x\sim\nu}\sbr{\rbr{\sum_{|k|\le N, k\neq 0}(\hat{f}(k)-\hat{f'}(k))e^{2\pi\ib kx}}^2}+\widebar{w}|\hat{f}(0)-\hat{f'}(0)|^2+\widebar{w}\frac{4M_v^2}{2v-1}\frac{1}{N^{2v-1}}\\
\end{align*}
Them, for the term $\EE_{x\sim\nu}\sbr{\rbr{\sum_{|k|\le N, k\neq 0}(\hat{f}(k)-\hat{f'}(k))e^{2\pi\ib kx}}^2}$, we apply inequality (\ref{ineq:critical_ineq_G_N}) to obtain
\begin{align*}
    &\frac{\widebar{w}}{\underline{w}}\EE_{x\sim\nu}\sbr{\rbr{\sum_{|k|\le N, k\neq 0}(\hat{f}(k)-\hat{f'}(k))e^{2\pi\ib kx}}^2}+\widebar{w}|\hat{f}(0)-\hat{f'}(0)|^2+\widebar{w}\frac{4M_v^2}{2v-1}\frac{1}{N^{2v-1}}\\
    \le&\frac{\widebar{w}}{\underline{w}}\frac{2}{n}\sum_{i=1}^{n}\rbr{\sum_{|k|\le N, k\neq 0}(\hat{f}(k)-\hat{f'}(k))e^{2\pi\ib kx_i}}^2+\frac{\widebar{w}c}{\underline{w}}[\delta_n(\cG_N)]^2+\frac{c'\widebar{w}\eta_v(N)(\log(1/\delta)+\log\log n)}{\underline{w}n}\\
    &+\frac{\widebar{w}}{\underline{w}}|\hat{f}(0)-\hat{f'}(0)|^2+\widebar{w}\frac{4M_v^2}{2v-1}\frac{1}{N^{2v-1}}\\
    \le&\frac{2\widebar{w}}{\underline{w}}\bignorm{(f(x)-f'(x))-\sum_{|k|>N}(\hat{f}(k)-\hat{f'}(k))e^{2\pi\ib kx_i}}_{\cD}^2+\frac{\widebar{w}c}{\underline{w}}[\delta_n(\cG_N)]^2+\widebar{w}\frac{4M_v^2}{2v-1}\frac{1}{N^{2v-1}}\\
    &+\frac{c'\widebar{w}\eta_v(N)(\log(1/\delta)+\log\log n)}{\underline{w}n}
\end{align*}
Since $(a+b)^2\le 2(a^2+b^2)$, we have
\[
\bignorm{(f(x)-f'(x))-\sum_{|k|>N}(\hat{f}(k)-\hat{f'}(k))e^{2\pi\ib kx_i}}_{\cD}^2\le 2\rbr{\|f-f'\|_{\cD}^2+\frac{4M_v^2}{2v-1}\frac{1}{N^{2v-1}}}.
\]
Plugging this bound back to the last inequality we obtained, we have that
\begin{align*}
    &\frac{2\widebar{w}}{\underline{w}}\bignorm{(f(x)-f'(x))-\sum_{|k|>N}(\hat{f}(k)-\hat{f'}(k))e^{2\pi\ib kx_i}}_{\cD}^2+\frac{\widebar{w}c}{\underline{w}}[\delta_n(\cG_N)]^2+\frac{4\widebar{w}M_v^2}{(2v-1)N^{2v-1}}\\
    &+\frac{c'\widebar{w}\eta_v(N)(\log(\frac{1}{\delta})+\log\log n)}{\underline{w}n}\\
    \le&\frac{4\widebar{w}}{\underline{w}}\|f-f'\|_{\cD}^2+\frac{20\widebar{w}M_v^2}{\underline{w}(2v-1)}\frac{1}{N^{2v-1}}+\frac{\widebar{w}c}{\underline{w}}[\delta_n(\cG_N)]^2+\frac{c'\widebar{w}\eta_v(N)(\log(\frac{1}{\delta})+\log\log n)}{\underline{w}n}.
\end{align*}
Finally, notice the upper bounds of $\delta_n(\cG_N)\le \sqrt\frac{(8N+2)\log(n/(8N+2)))}{n}$ and $\eta_v(N)$ as discussed before, we set $N$ to balance the rate of all the terms to obtain
\begin{enumerate}
    \item if $1<v<1/2$, we set $N\asymp n^{\frac{1}{2v}}$ to have
    \[
    \EE_{x\sim\nu}[(f(x)-f'(x))^2]\le \frac{4\widebar{w}}{\underline{w}}\|f-f'\|_{\cD}^2+\frac{20\widebar{w}M_v^2}{\underline{w}(2v-1)}\frac{1}{n^{1-\frac{1}{2v}}}+\frac{9\widebar{w}c}{\underline{w}}\frac{\log(n^{1-\frac{1}{2v}}/8)}{n^{1-\frac{1}{2v}}}+\frac{4c'\widebar{w}M_v^2(\log(\frac{1}{\delta})+\log\log n)}{\underline{w}(1-v)^2n^{2-\frac{1}{v}}},
    \]
    \item if $v=1$, we also set $N\asymp n^{\frac{1}{2v}}$ to have
    \[
    \EE_{x\sim\nu}[(f(x)-f'(x))^2]\le \frac{4\widebar{w}}{\underline{w}}\|f-f'\|_{\cD}^2+\frac{20\widebar{w}M_1^2}{\underline{w}}\frac{1}{\sqrt{n}}+\frac{9\widebar{w}c}{\underline{w}}\frac{\log(\sqrt{n}/8)}{\sqrt{n}}+\frac{8c'\widebar{w}M_1^2\log n (\log(\frac{1}{\delta})+\log\log n)}{\underline{w}n}.
    \]
\end{enumerate}
Combining these two circumstances together, we have that with probability at least $1-\delta$,
$$\EE_{x\sim\nu}[(f(x)-f'(x))^2]\le \frac{4\widebar{w}}{\underline{w}}\|f-f'\|_{\cD}^2+\tilde{C}_{v}\frac{\widebar{w}}{\underline{w}}\frac{\log n(\log\log n)\log(1/\delta)}{n^{1-\frac{1}{2v}}},$$
uniformly over $f-f'\in\DD(\cF)$. Thus, we finish the proof.
\end{proof}
\section{Proofs in Section \ref{sec:high_dimension}}\label{app:proofs_sec:high_dimension}
In this section, we provide proofs about the lemmas and theorems in Section \ref{sec:high_dimension}.
\begin{proof}[Proof of Lemma \ref{lemma:norm_equi_tensor}]
    For any $h\in\DD(\cF)$, we first expand $h$ into its Fourier series $h(x)=\sum_{k\in\ZZ^d}\hat{h}(k)e^{2\pi\ib k\cdot x}$.
    Again, we slice the Fourier series and consider its truncation $h_N=\sum_{k\in\ZZ^d,|k_i|\le N}\hat{h}(k)e^{2\pi\ib k\cdot x}$. The remainder is $h-h_N=\sum_{k\in\ZZ^d, \|k\|_{\infty}>N}|\hat{h}(k)|^2\le \frac{4M_v^2 S_d}{2v-d}\frac{1}{N^{2v-d}}.$ 
    
    Note that in $\RR^d$, the Euclidean norm is lower-bounded by the maximum norm, i.e., $|k| = \|k\|_2 \ge \|k\|_\infty$. Thus, $|k|^{-2v} \le \|k\|_\infty^{-2v}$. This yields:
    \[
        \sum_{\|k\|_\infty > N} |\hat{h}(k)|^2 \le 4M_v^2 \sum_{\|k\|_\infty > N} \|k\|_\infty^{-2v}.
    \]
    To estimate the right-hand sum, we partition the summation domain into shells based on the maximum norm. Let $S_r = \{k \in \ZZ^d : \|k\|_\infty = r\}$ for integers $r > N$. The number of integer points on the surface of the hypercube of radius $r$ is exactly $|S_r| = (2r+1)^d - (2r-1)^d$. By the mean value theorem, for $r \ge 1$, we can bound $|S_r|$ by:
    \[
        |S_r| \le 2d(2r+1)^{d-1} \le 2d(3r)^{d-1} = S_d r^{d-1},
    \]
    where $S_d$ is a constant depending solely on the dimension $d$. Rewriting the multiple sum as a single sum over $r$, we have:
    \[
        \sum_{\|k\|_\infty > N} \|k\|_\infty^{-2v} = \sum_{r=N+1}^\infty \sum_{k \in S_r} r^{-2v} \le \sum_{r=N+1}^\infty S_d r^{d-1} r^{-2v} = S_d \sum_{r=N+1}^\infty r^{d-1-2v}.
    \]
    Since $v > d/2$, we have $d - 1 - 2v < -1$. The function $x \mapsto x^{d-1-2v}$ is strictly monotonically decreasing for $x > 0$. We can therefore bound the series by an integral:
    \[
        \sum_{r=N+1}^\infty r^{d-1-2v} \le \int_N^\infty x^{d-1-2v} dx = \left[ \frac{x^{d-2v}}{d-2v} \right]_N^\infty = \frac{N^{d-2v}}{2v-d}.
    \]
    Combining all the inequalities, we obtain $\sum_{\|k\|_\infty > N} |\hat{h}(k)|^2 \le \frac{4M_v^2 S_d}{2v-d} N^{d-2v}$.
    Now, we study $h_N(x)=\sum_{k\in\ZZ^d,|k_i|\le N}\hat{h}(k)e_k(x)$. By definition, we have
    \[
    \|h_N(x)\|_{\cS}^2=\frac{1}{n^{\beta}}\sum_{i\in\cS}(\sum_{k\in\ZZ^d,|k_i|\le N}\hat{h}(k)e_k(x_i))^2=\frac{1}{n^{\beta}}\sum_{i\in\cS}\rbr{\sum_{k,l\in\ZZ^d,|k_i|,|l_i|\le N}\hat{h}(k)\hat{h}(l)e_l(x_i)e_k(x_i)}.
    \]
    We denote the covariate $x_i=(x_i^1,x_i^2,\cdots,x_i^d)\in\RR^d$ and the vector-valued mapping $v_N(t)\in\CC^{2N+1}$ to be
    \[
\begin{array}{cc}
\displaystyle
v_{N}(t):=
\begin{pmatrix}
e^{-2\pi\ib Nt}\\
\vdots\\
e^{-4\pi\ib t}\\
e^{-2\pi\ib x}\\
1\\
e^{2\pi\ib x}\\
e^{4\pi\ib x}\\
\vdots\\
e^{2\pi\ib Nt}
\end{pmatrix}
\end{array}.
\]
We denote $X_i^1=v_N(x_i^1),\cdots, X_i^d=v_N(x_i^d),X_i^{d+1}=\widebar{v_N(x_i^1)},\cdots,X_i^{2d}=\widebar{v_N(x_i^d)}$.

Let $H\in\CC^{(2N+1)\times\cdots\times(2N+1)}$ to be the $d$-order tensor such that $$H(k_1,\cdots,k_d)=\hat{h}(k_1,\cdots,k_d),\forall (k_1,\cdots,k_d)\in\ZZ^d, |k_i|\le N.$$
Then, $X_1^k,\cdots,X_N^k\stackrel{i.i.d}{\sim} X^k$, we have that
\[
\frac{1}{n^{\beta}}\sum_{i\in\cS}(\sum_{k\in\ZZ^d,|k_i|\le N}\hat{h}(k)e_k(x_i))^2=\inner{\widebar{H}\otimes H}{\frac{1}{n^\beta}\sum_{i\in\cS}X_i^1\otimes\cdots\otimes X_i^{2d}}.
\]
For any test tensor $w_1\otimes w_2\otimes \cdots\otimes w_{2d}$,  we have that
\begin{align*}
    &\inner{\frac{1}{n^\beta}\sum_{i\in\cS}X_i^1\otimes\cdots\otimes X_i^{2d}-\EE\sbr{X^1\otimes\cdots\otimes X^{2d}}}{w_1\otimes w_2\otimes \cdots\otimes w_{2d}}\\
    &=\frac{1}{n^{\beta}}\sum_{i\in\cS}\prod_{k=1}^{2d}\inner{X_i^k}{w_k}-\EE\sbr{\prod_{k=1}^{2d}\inner{X^k}{w_k}}.
\end{align*}
Therefore, by Definition \ref{def:tensor_operator_norm}, we have
\begin{align*}
&\bignorm{\frac{1}{n^\beta}\sum_{i\in\cS}X_i^1\otimes\cdots\otimes X_i^{2d}-\EE\sbr{X^1\otimes\cdots\otimes X^{2d}}}_{op}\\
=&\sup_{\|w_1\|\le 1,\cdots,\|w_{2d}\|\le 1}\abr{\frac{1}{n^{\beta}}\sum_{i\in\cS}\prod_{k=1}^{2d}\inner{X_i^k}{w_k}-\EE\sbr{\prod_{k=1}^{2d}\inner{X^k}{w_k}}}.
\end{align*}
Surprisingly, we obtain a quantity with the same structure as the multi-product empirical process in Lemma \ref{lemma:multiproduct_empiricalprocess}. We define $\cF^k$ to be $\cF^k\equiv\cF_{lin}=\cbr{x\mapsto \inner{x}{w}:\|w\|_{2}\le 1}$, where the inner product is the natural inner product in $\CC^{2N+1}$.

To proceed, we invoke the following lemmas, whose proofs are deferred to Appendix \ref{app:proofs_lemmas_in_proofs_sec:high_dimension}.
First, we observe that with $\underline{w} \le \frac{d\nu}{d\mu}(x) \le \widebar{w}$ for all ${x \in \Omega}$, denoting $\nu_k$ and $\mu_k$ as the marginal measures of $\nu$ and $\mu$ on the $k$-th dimension, and letting $x_{-k}$ denote all coordinates other than $k$. By disintegrating the measure $\mu$ as $d\mu(x) = d\mu(x_{-k}|x_k) d\mu_k(x_k)$, the marginal density can be obtained by integrating out the remaining dimensions. Thus, we have:
\begin{align*}
    \underline{w} = \int_{\Omega_{-k}} \underline{w} \, d\mu(x_{-k}|x_k) 
    &\le \frac{d\nu_k}{d\mu_k}(x_k) = \int_{\Omega_{-k}} \frac{d\nu}{d\mu}(x_k, x_{-k}) \, d\mu(x_{-k}|x_k) \le \int_{\Omega_{-k}} \widebar{w} \, d\mu(x_{-k}|x_k) = \widebar{w}.
\end{align*}
Here, $\Omega_{-k}$ denotes the product space formed by all dimensions excluding the $k$-th coordinate.

In order to apply Lemma \ref{lemma:multiproduct_empiricalprocess}, we first introduce the following lemmas.
\begin{lemma}\label{lemma:psi2_diameter}
Let $(x_1, \dots, x_d) \sim \nu$ be a random vector defined on $[0,1]^d$, and let $\nu$ denote its joint density. Suppose that $\nu$ is absolutely continuous with respect to the Lebesgue measure $\mu$ on $[0,1]^d$, and its Radon-Nikodym derivative $p(x) = \frac{d\nu}{d\mu}(x)$ is bounded away from zero and infinity, i.e., there exist constants $0 < \underline{w} \le \widebar{w} < \infty$ such that $\underline{w} \le p(x) \le \widebar{w}$.

Define the Fourier basis vector $v_N(t) \in \mathbb{C}^{2N+1}$ as 
\[
v_N(t) = \left( e^{-2\pi \ib N t}, \cdots, 1, \cdots, e^{2\pi \ib N t} \right)^T,
\]
and the linear functional class associated with the $k$-th dimension as
\[
\mathcal{F}_{lin}^{(k)} = \left\{ f_w(x) = \langle v_N(x_k), w \rangle : w \in \mathbb{C}^{2N+1}, \|w\|_2 \le 1 \right\},
\]
where the inner product is the natural inner product in $\mathbb{C}^{2N+1}$. Then, the $2$-Orlicz diameter of $\mathcal{F}_{lin}^{(k)}$ in Definition \ref{def:orlicz} under the marginal measure $\nu_k$ satisfies
\[
d_{\psi_2}(\mathcal{F}_{lin}^{(k)}) = \Theta\left( \sqrt{\frac{N}{\log N}} \right).
\]
\end{lemma}
Through the following lemma, we connect the Orlicz norm $d_{\psi_2}(\cdot)$ with the norm $\|\cdot\|_2$ that is canonical in generic chaining of Talagrand $\gamma$ functional, which enabling us to apply the Majorizing Measure Theorem (Lemma \ref{lemma:majorizing_measure}). 
\begin{lemma}\label{lemma:Orlicz_boundby_2-norm}
Let $\mu$ be the Lebesgue measure on $[0,1]$, and let $\nu$ be a probability measure on $[0,1]$ that is absolutely continuous with respect to $\mu$. For its Radon-Nikodym derivative $p(x) = \frac{d\nu}{d\mu}(x)$, there exist absolute constants $0 < \underline{w} \le \widebar{w} < \infty$ such that $c \le p(x) \le C$.

For any integer $N \ge 1$, define the Fourier basis vector $v_N(x) \in \mathbb{C}^{2N+1}$ as 
\[
v_N(t) = \left( e^{-2\pi \ib Nx}, \dots, 1, \dots, e^{2\pi \ib N x} \right)^T.
\]
For any $v, w \in \mathbb{C}^{2N+1}$, let $d_{\psi_2}(v, w) = \|\langle v_N(x), v - w \rangle\|_{\psi_2}$ denote the distance induced by the $2$-Orlicz norm under the measure $\nu$. Then, there exists $C' > 0$ such that for all $v, w \in \mathbb{C}^{2N+1}$,
\[
d_{\psi_2}(v, w) \le C' \sqrt{\frac{N}{\log N}} \|v - w\|_2.
\]
\end{lemma}
By the definition of the Talagrand $\gamma$ functional, we further have
\[
\gamma_2(\cF_{lin},d_{\psi_2})\le C' \sqrt{\frac{N}{\log N}}\gamma_2(\cF_{lin},\|\cdot\|_2).
\]
Applying the majorizing measure Theorem \citep{talagrand2014upper}, we have that
\[
\gamma_2(\cF_{lin},\|\cdot\|_2)\le L\EE_{g\sim N(0,I_{2N+1})}\sup_{\|w\|_2\le 1}\inner{g}{w}\le L\sqrt{2}\sqrt{2N+1}\le 3L\sqrt{N}
\]
for some universal constant $L$. Then, we apply Lemma \ref{lemma:multiproduct_empiricalprocess} to get that 
\begin{align*}
&\EE\bignorm{\frac{1}{n^\beta}\sum_{i\in\cS}X_i^1\otimes\cdots\otimes X_i^{2d}-\EE\sbr{X^1\otimes\cdots\otimes X^{2d}}}_{op}\\
=&\EE\sup_{\|w_1\|\le 1,\cdots,\|w_{2d}\|\le 1}\abr{\frac{1}{n^{\beta}}\sum_{i\in\cS}\prod_{k=1}^{2d}\inner{X_i^k}{w_k}-\EE\sbr{\prod_{k=1}^{2d}\inner{X^k}{w_k}}}\\
\le&C(\widebar{w},\underline{w},d)\rbr{\frac{N}{\log N}}^{d}\rbr{\frac{3dL\sqrt{N}}{\sqrt{n^\beta}}+\frac{(3L\sqrt{N})^{2d}+(\beta\log n)^{1/2}}{n^{\beta}}}.
\end{align*}
By Lemma \ref{lemma:Mcdiarmid}, with probability at least $1-\delta$,
\begin{align*}
&\bignorm{\frac{1}{n^\beta}\sum_{i\in\cS}X_i^1\otimes\cdots\otimes X_i^{2d}-\EE\sbr{X^1\otimes\cdots\otimes X^{2d}}}_{op}\\
\le& C(\widebar{w},\underline{w},d)\rbr{\frac{N}{\log N}}^{d}\rbr{\frac{3dL\sqrt{N}}{\sqrt{n^\beta}}+\frac{(3L\sqrt{N})^{2d}+(\beta\log n)^{1/2}}{n^{\beta}}}+\frac{(2N+1)^d\sqrt{\log(1/\delta)}}{\sqrt{n^\beta}}.
\end{align*}
We define $\zeta^\delta_{n^\beta}(N)$ to be
\[
\zeta^\delta_{n^\beta}(N):=C(\widebar{w},\underline{w},d)\rbr{\frac{N}{\log N}}^{d}\rbr{\frac{3dL\sqrt{N}}{\sqrt{n^\beta}}+\frac{(3L\sqrt{N})^{2d}+(\beta\log n)^{1/2}}{n^{\beta}}}+\frac{(2N+1)^d\sqrt{\log(1/\delta)}}{\sqrt{n^\beta}}.
\]
Similar to Lemma \ref{lemma:Toeplitz}, by the property of the Toeplitz operator, we know that
\[
\underline{w}\int_{[0,1]^d}|h_N(x)|^2d\mu(x)\le\int_{[0,1]^d}|h_N(x)|^2d\nu(x)\le \widebar{w}\int_{[0,1]^d}|h_N(x)|^2d\mu(x).
\]
Therefore, with probability at least $1-2\delta$, we have
\[
\|h_N\|_{\cS}^2\le \frac{\widebar{w}+\zeta^\delta_{n^\beta}(N)}{\underline{w}-\zeta^\delta_{n}(N)}\|h_N\|_{\cD}^2.
\]
Recall that the remainder term has $\sum_{k\in\ZZ^d,\|k\|_{\infty}>N}|\hat{h}(k)|^2\le \frac{4M_v^2S_d}{2v-d}\frac{1}{N^{2v-d}}$. We have
\begin{align*}
    \|h\|_{\cS}^2=&\|h_N+\sum_{|k|>N}\hat{h}(k)e_k\|_{\cS}^2\\
        \le& 2\rbr{\|h_N\|_{\cS}^2+\frac{4M_v^2S_d}{2v-d}\frac{1}{N^{2v-d}}}\\
        \le& 2\rbr{\frac{\widebar{w}+\zeta^\delta_{n^\beta}(N)}{\underline{w}-\zeta^\delta_{n}(N)}\|h_N\|_{\cD}^2+\frac{4M_v^2S_d}{2v-d}\frac{1}{N^{2v-d}}}\\
        \le&2\frac{\widebar{w}+\zeta^\delta_{n^\beta}(N)}{\underline{w}-\zeta^\delta_{n}(N)}\|h-\sum_{|k|>N}\hat{h}(k)e_k\|_{\cD}^2+2\frac{4M_v^2S_d}{2v-d}\frac{1}{N^{2v-d}}\\
        \le&4\frac{\widebar{w}+\zeta^\delta_{n^\beta}(N)}{\underline{w}-\zeta^\delta_{n}(N)}\|h\|_{\cD}^2+\rbr{2\frac{\widebar{w}+\zeta^\delta_{n^\beta}(N)}{\underline{w}-\zeta^\delta_{n}(N)}+1}\frac{4M_v^2S_d}{2v-d}\frac{1}{N^{2v-d}}.
\end{align*}
Finally, we set $N\asymp\frac{n^{\frac{\beta}{2d+1}}}{\log n}$. Plugging this back, when $n\gtrsim \tilde{C}(\widebar{w},\underline{w},d)(\log(1/\delta))^{(2d+1)/\beta}$ for some $(\widebar{w},\underline{w},d)$ is sufficiently large, we have 
\[
\|h\|_S^2\le 9\frac{\widebar{w}}{\underline{w}}\|h\|_\cD^2+\frac{12\widebar{w}S_dM_v^2}{\underline{w}(2v-d)}\frac{(\log n)^{2v-d}}{n^{\frac{(2v-d)\beta}{2d+1}}}.
\]
Thus, we finish the proof.
\end{proof}
\begin{proof}[Proof of Theorem \ref{thm:risk_bound_high_dimension}]
    Notice that the proof of Theorem \ref{thm:risk_bound_fixed_design} does not really depend on the dimension $d$ except for the step of bounding $\cW_{\cD}^{\varepsilon}(2r)$ and $\cH_\cD^\varepsilon(2r)$ via Lemma \ref{lemma:unbias}. Specifically, we follow the path of proving Theorem \ref{thm:risk_bound_fixed_design} to obtain that for any $0<\delta<1$ and any $r\ge\hat{r}_n$,
    \[
    \Opt^*_{\cD}(\Breve{f})\le \cW_{\cD}^{\varepsilon}(2r)+\cH_{\cD}^\varepsilon(2r)+r\frac{10\sqrt{2}\tau \sqrt{\log(1/\delta)}}{\sqrt{n}}+V_{2r}(\widebar{f}).
    \]
    Given any realization of $\varepsilon$, denote $\varphi_1$ and $\varphi_2$ as the functions in $\BB_{2r}(\Breve{f};\cD)$ that achieves the supremum in $\cW_\cD^{\varepsilon}(2r)$ and $\cH_{\cD}^\varepsilon(2r)$, respectively, 
    Then, by Lemma \ref{lemma:norm_equi_tensor}, we know that for some constant $S_d$, with probability at least $1-\delta$,
    \[
    \|h-\Breve{f}\|_{\cS_k}\le\sqrt{9\frac{\widebar{w}}{\underline{w}}(2r)^2+\frac{6\widebar{w}M_v^2S_d}{\underline{w}(2v-d)}\frac{(\log n)^{2v-d}}{n^{\frac{(2v-d)\beta}{2d+1}}}}\le 3\sqrt{\frac{\widebar{w}}{\underline{w}}}(2r)+M_v\sqrt{\frac{6\widebar{w}S_d}{\underline{w}(2v-d)}}\frac{(\log n)^{v-d/2}}{n^{\frac{(2v-d)\beta}{4d+2}}}
    \]
    Thus, we define $\tilde{r}$ to be
    \[
    \tilde{r}:=3\sqrt{\frac{\widebar{w}}{\underline{w}}}r+4M_v\sqrt{\frac{\widebar{w}S_d}{\underline{w}(2v-d)}}\frac{(\log n)^{v-d/2}}{n^{\frac{(v-d/2)\beta}{2d+1}}}.
    \]
    Then, with probability at least $1-\delta$, we have $\varphi_1\in\BB_{2\tilde{r}}(\Breve{f};\cS_k)$. Applying Lemma~\ref{lemma:unbias}, we obtain that
    \begin{align*}
    \cW_\cD^\varepsilon(2r)=\frac{1}{n}\sum_{i=1}^{n}\varepsilon_iv_i(\varphi_1(x_i)-\Breve{f}(x_i))=A_n(\varphi_1)\le \frac{1}{K}\sum_{i=1}^{K}B_{\cS_k}(\varphi_1)+\frac{8\|\varphi_1-\Breve{f}\|_{\cD}\tau\sqrt{\log(K/\delta)}}{\sqrt{K}}
\end{align*}
Since it holds that $\varphi_1\in\BB_{2r}(\Breve{f};\cD)$ and  $\varphi_1\in\BB_{2\tilde{r}}(\Breve{f};\cS_k)$, we have $B_{\cS_k}(\varphi_1)\le \cT_{\cS_k}^\varepsilon(2\tilde{r})$. By Lemma \ref{lemma:wild_optmism_bound_empirical_process} and the setup of Algorithm \ref{alg:wild-refitting}, we have that
\begin{align*}
\cT_{\cS_k}^\varepsilon(2\tilde{r})=\cT_{\cS_k}^\varepsilon(\|\tilde{f}^k_{\rho_1^k}-\Breve{f}\|_{\cS_k})=\tilde{\Opt}^{\rho_1^k}_{\cS_k}(\tilde{f}^k_{\rho_1^k}).
\end{align*}
Thus, we have
\[
\cW_\cD^\varepsilon(2r)\le \frac{1}{K}\sum_{k=1}^{K}\tilde{\Opt}^{\rho_1^k}_{\cS_k}(\tilde{f}^k_{\rho_1^k})+\frac{16r\tau\sqrt{\log(K/\delta)}}{\sqrt{K}}.
\]
We apply the same argument to $\cH_{\cD}^\varepsilon$ to get
\[
\cH_\cD^\varepsilon(2r)\le \frac{1}{K}\sum_{k=1}^{K}\wcheck{\Opt}^{\rho_2^k}_{\cS_k}(\wcheck{f}^k_{\rho_2^k})+\frac{16r\tau\sqrt{\log(K/\delta)}}{\sqrt{K}}.
\]
Plugging these two bounds back, we are able to bound the empirical excess risk $\cE_\cD(\Breve{f})$ as
\[
\cE_\cD(\Breve{f})\le \frac{1}{K}\sum_{k=1}^{K}\rbr{\tilde{\Opt}^{\rho_1^k}_{\cS_k}(\tilde{f}^k_{\rho_1^k})+\wcheck{\Opt}^{\rho_2^k}_{\cS_k}(\wcheck{f}^k_{\rho_2^k})}+r\frac{10\sqrt{2}\tau \sqrt{\log(1/\delta)}}{\sqrt{n}}+V_{2r}(\widebar{f})+32r\frac{\tau\sqrt{\log(K/\delta)}}{\sqrt{K}}.
\]
Now, we introduce the following lemma analogous to Lemma \ref{lemma:critical_radius} such to help us control the true excess risk $\cE(\Breve{f})$ via $\cE_\cD(\Breve{f})$.
\begin{lemma}\label{lemma:critical_radius_high_dimension}
    Denote $\DD(\cF)^2$ to be the function class $\cbr{(f-f')^2:f,f'\in\cF}$. For any $\delta>0$, there exist constants $\tilde{C}_v$ such that, with probability at least $1-\delta$,
$$\EE_{x\sim\nu}[(f(x)-f'(x))^2]\le \frac{4\widebar{w}}{\underline{w}}\|f-f'\|_{\cD}^2+\tilde{C}_{v}^d\frac{\widebar{w}}{\underline{w}}\frac{\log n(\log\log n)\log(1/\delta)}{n^{1-\frac{d+1}{2v-d+1}}},$$
uniformly over $f-f'\in\DD(\cF)$ for some constant $\tilde{C}_v^d$ only related to $v$ and $d$.
\end{lemma}
The proof of Lemma \ref{lemma:critical_radius_high_dimension} is deferred to Appendix \ref{app:proofs_lemmas_in_proofs_sec:high_dimension}.
With this lemma, following the same process of the Proof of Theorem \ref{thm:excess_risk_random_design}, we obtain that with probability at least $1-6\delta$, when $n$ is large enough,
\begin{align*}
\cE(\Breve{f})\le \frac{4\widebar{w}}{\underline{w}}\rbr{\frac{1}{K}\sum_{k=1}^{K}\rbr{\tilde{\Opt}_{\cS_k}^{\rho^k_1}(\tilde{f}^k_{\rho^k_1})+\wcheck{\Opt}_{\cS_k}^{\rho^k_2}(\wcheck{f}^k_{\rho^k_2})}+R(\delta)+V_{2r}(\widebar{f})}+\tilde{C}_{v}^d\frac{\widebar{w}}{\underline{w}}\frac{\log n(\log\log n)\log(1/\delta)}{n^{1-\frac{d+1}{2v-d+1}}},
\end{align*}
where
\[
R(\delta)=r\frac{10\sqrt{2}\tau \sqrt{\log(1/\delta)}}{\sqrt{n}}+32r\frac{\tau\sqrt{\log(K/\delta)}}{\sqrt{K}}
\]
and
\begin{align*}V_{2r}(\bar{f})&=\sup_{f\in\BB_{2r}(\Breve{f};\cD)}\cbr{\frac{1}{n}\sum_{i=1}^{n}\varepsilon_i(\widebar{f}(x_i)-f^*(x_i))(f(x_i)-\Breve{f}(x_i))}\\
        &+\sup_{f\in\BB_{2r}(\Breve{f};\cD)}\cbr{\frac{1}{n}\sum_{i=1}^{n}\varepsilon_i(\widebar{f}(x_i)-f^*(x_i))(\Breve{f}(x_i)-f(x_i))}.
        \end{align*}
        We finish the proof.

\end{proof}

\section{Proofs of Lemmas Appendix \ref{app:math_tools}}\label{app:proofs_app:mathtools}
\begin{proof}[Proof of Lemma \ref{lemma:Toeplitz}]
By the Plancherel theorem, there is an isometric isomorphism between sequences $v \in \ell^2(\mathbb{Z})$ and functions $V \in L^2([0,1])$ given by the Fourier series $V(x) = \sum_{k \in \mathbb{Z}} v_k e^{2\pi \ib k x}$, which satisfies $\|v\|_{\ell^2} = \|V\|_{L^2}$. 

Now, consider the operator $A$ acting on the vector $v$, yielding $u = Av$. The $i$-th component of this matrix-vector multiplication is a discrete convolution:
\begin{equation*}
    u_i = (Av)_i = \sum_{j \in \mathbb{Z}} a_{i,j} v_j = \sum_{j \in \mathbb{Z}} \hat{w}(i-j) v_j.
\end{equation*}
According to the convolution theorem for Fourier transforms, a discrete convolution in the frequency domain corresponds strictly to pointwise multiplication in the time domain. Thus, the continuous function $U(x)$ corresponding to the sequence $u$ is exactly the product of $w(x)$ and $V(x)$:
\begin{equation*}
    U(x) = \sum_{i \in \mathbb{Z}} u_i e^{2\pi \ib i x} = w(x) V(x).
\end{equation*}
Therefore, the operator norm of $A$ on $\ell^2(\mathbb{Z})$ is equivalent to the operator norm of the multiplication operator $M_w: V(x) \mapsto w(x)V(x)$ on $L^2([0,1])$:
\begin{equation*}
    \|A\|_{\text{op}} = \sup_{v \neq 0} \frac{\|Av\|_{\ell^2}}{\|v\|_{\ell^2}} = \sup_{V \neq 0} \frac{\|w V\|_{L^2}}{\|V\|_{L^2}}.
\end{equation*}
Next, we prove that the norm of this multiplication operator is exactly $\|w\|_{L^\infty}$.

For any $V \in L^2([0,1])$, we have:
\begin{equation*}
    \|w V\|_{L^2}^2 = \int_0^1 |w(x) V(x)|^2 \, dx \le \int_0^1 \|w\|_{L^\infty}^2 |V(x)|^2 \, dx = \|w\|_{L^\infty}^2 \|V\|_{L^2}^2.
\end{equation*}
Taking the square root and the supremum yields $\|A\|_{\text{op}} \le \|w\|_{L^\infty}$.

On the other hand, by the definition of the essential supremum, for any given $\epsilon > 0$, the set $S_\epsilon = \{x \in [0,1] : |w(x)| > \|w\|_{L^\infty} - \epsilon\}$ has a strictly positive Lebesgue measure, i.e., $\mu(S_\epsilon) > 0$. 
We construct a specific indicator function:
\begin{equation*}
    V_\epsilon(x) = \frac{1}{\sqrt{\mu(S_\epsilon)}} \mathbf{1}_{S_\epsilon}(x).
\end{equation*}
Clearly, $\|V_\epsilon\|_{L^2} = 1$. Applying the operator to this function gives:
\begin{align*}
    \|w V_\epsilon\|_{L^2}^2 &= \int_{S_\epsilon} |w(x)|^2 \frac{1}{\mu(S_\epsilon)}dx \ge \int_{S_\epsilon} (\|w\|_{L^\infty} - \epsilon)^2 \frac{1}{\mu(S_\epsilon)} dx = (\|w\|_{L^\infty} - \epsilon)^2.
\end{align*}
This shows that for any $\epsilon > 0$, $\|A\|_{\text{op}} \ge \|w\|_{L^\infty} - \epsilon$. Letting $\epsilon \to 0$ gives $\|A\|_{\text{op}} \ge \|w\|_{L^\infty}$.
Combining the upper and lower bounds, we establish that
$\|A\|_{\text{op}} = \|w\|_{L^\infty}.$
 
Furthermore, since $w(x)$ is real-valued, $A$ is self-adjoint. By the Rayleigh-Ritz theorem, the spectral lower bound of $A$ is given by the infimum of its Rayleigh quotient. Using the same isomorphism, we evaluate the quadratic form:
\begin{equation*}
    \langle v, A v \rangle_{\ell^2} = \int_0^1 \overline{V(x)} (w(x) V(x)) \, dx = \int_0^1 w(x) |V(x)|^2 \, dx.
\end{equation*}
Under the assumption that $w(x) \ge c > 0$, we have:
\begin{equation*}
    \int_0^1 w(x) |V(x)|^2 \, dx \ge \inf_{x \in [0,1]} w(x) \int_0^1 |V(x)|^2 \, dx \ge c \|v\|_{\ell^2}^2.
\end{equation*}
Taking the infimum over all unit vectors $\|v\|_{\ell^2}=1$ yields the desired spectral lower bound:
\begin{equation*}
    \lambda_{\min}(A) = \inf_{\|v\|_{\ell^2}=1} \langle v, A v \rangle_{\ell^2} \ge \inf_{x \in [0,1]} w(x) \ge c.
\end{equation*}

Finally, if $w(x) \ge c > 0$ almost everywhere on $[0,1]$, the absolute value of $w(x)$ is bounded below by $c$. Then, $\|w\|_{L^\infty} \ge \inf_{x \in [0,1]} |w(x)| \ge c.$
Combining this with $\|A\|_{\text{op}} = \|w\|_{L^\infty}$, we obtain $\|A\|_{\text{op}} \ge c$ and finish the proof.
\end{proof}
\begin{proof}[Proof of Lemma \ref{lemma:Toeplitz_multidim}]
    Let $v$ be an arbitrary sequence in $\ell^2(\mathbb{Z}^d)$ with finite support. We can associate $v$ with a $d$-dimensional trigonometric polynomial defined as:
    $$
    f(\mathbf{x}) = \sum_{\mathbf{k} \in \mathbb{Z}^d} v_{\mathbf{k}} e^{2\pi \mathrm{i} \mathbf{k} \cdot \mathbf{x}}.
    $$
    By Parseval's identity, we know that $\|f\|_{L^2([0,1]^d)}^2 = \|v\|_{\ell^2}^2$. 
    
    Now, we consider the quadratic form associated with the operator $A$:
    $$
    \langle v, A v \rangle = \sum_{\mathbf{i}} \sum_{\mathbf{j}} \bar{v}_{\mathbf{i}} a_{\mathbf{i},\mathbf{j}} v_{\mathbf{j}} = \sum_{\mathbf{i}} \sum_{\mathbf{j}} \bar{v}_{\mathbf{i}} \hat{w}(\mathbf{i}-\mathbf{j}) v_{\mathbf{j}}.
    $$
    Substituting the definition of the Fourier coefficient $\hat{w}(\mathbf{i}-\mathbf{j})$ into the summation yields:
    $$
    \langle v, A v \rangle = \sum_{\mathbf{i}} \sum_{\mathbf{j}} \bar{v}_{\mathbf{i}} v_{\mathbf{j}} \int_{[0,1]^d} w(\mathbf{x}) e^{-2\pi \mathrm{i} (\mathbf{i}-\mathbf{j}) \cdot \mathbf{x}} d\mathbf{x}.
    $$
    Since the summation is finite, we can interchange the order of integration and summation:
    $$
    \langle v, A v \rangle = \int_{[0,1]^d} w(\mathbf{x}) \left( \sum_{\mathbf{j}} v_{\mathbf{j}} e^{2\pi \mathrm{i} \mathbf{j} \cdot \mathbf{x}} \right) \left( \sum_{\mathbf{i}} \bar{v}_{\mathbf{i}} e^{-2\pi \mathrm{i} \mathbf{i} \cdot \mathbf{x}} \right) d\mathbf{x}.
    $$
    Recognizing the terms inside the parentheses as $f(\mathbf{x})$ and its complex conjugate $\overline{f(\mathbf{x})}$, we obtain the fundamental identity:
    $$
    \langle v, A v \rangle = \int_{[0,1]^d} w(\mathbf{x}) |f(\mathbf{x})|^2 d\mathbf{x}.
    $$
    To bound the operator norm $\|A\|_{op}$, we take the absolute value of the quadratic form:
    $$
    |\langle v, A v \rangle| \le \int_{[0,1]^d} |w(\mathbf{x})| |f(\mathbf{x})|^2 d\mathbf{x} \le \|w\|_{L^\infty} \int_{[0,1]^d} |f(\mathbf{x})|^2 d\mathbf{x} = \|w\|_{L^\infty} \|v\|_{\ell^2}^2.
    $$
    Taking the supremum over all such $v$ with $\|v\|_{\ell^2} = 1$, we conclude that $\|A\|_{op} \le \|w\|_{L^\infty}$.

    For the lower bound of the spectrum, since $w(\mathbf{x}) > 0$ almost everywhere, the operator $A$ is self-adjoint, and its spectrum lies on the real line. Utilizing the same identity, we have:
    $$
    \langle v, A v \rangle = \int_{[0,1]^d} w(\mathbf{x}) |f(\mathbf{x})|^2 d\mathbf{x} \ge \left( \inf_{\mathbf{x}\in[0,1]^d} w(\mathbf{x}) \right) \int_{[0,1]^d} |f(\mathbf{x})|^2 d\mathbf{x}.
    $$
    This implies $\langle v, A v \rangle \ge \inf w(\mathbf{x}) \|v\|_{\ell^2}^2$. By the Rayleigh-Ritz theorem for self-adjoint operators, this establishes that $\inf\{\sigma(A)\} \ge \inf_{\mathbf{x}\in[0,1]^d} w(\mathbf{x})$.
\end{proof}

\section{Proofs of Lemmas in Appendix \ref{app:proofs_sec:statistical_guarantee}}\label{app:proofs_Appendix_risk_guarantee}

\begin{proof}[Proof of Lemma \ref{lemma:Opt*<Z}]
    First because $r\ge\hat{r}_n$, we have the deterministic bound such that
        \[\frac{1}{n}\sum_{i=1}^nw_i(\Breve{f}(x_i)-f^*(x_i))\le Y(w):=\sup_{f\in\BB_r(f^*;\cD)}\frac{1}{n}\sum_{i=1}^{n}w_i(f(x_i)-\Breve{f}(x_i)).\]
        Then, we view the empirical process on the right hand side as a function of $w$ and have that
        \begin{align*}
            Y(w)-Y(w')=&\sup_{f\in\BB_r(f^*;\cD)}\frac{1}{n}\sum_{i=1}^{n}w_i(f(x_i)-\Breve{f}(x_i))-\sup_{f\in\BB_r(f^*;\cD)}\frac{1}{n}\sum_{i=1}^{n}w_i'(f(x_i)-\Breve{f}(x_i))\\
            \le&\sup_{f\in\BB_r(f^*;\cD)}\frac{1}{n}\sum_{i=1}^{n}(w_i-w_i')(f(x_i)-\Breve{f}(x_i))\\
            \le&\|w-w'\|_2\frac{r}{\sqrt{n}}.
        \end{align*}
        Same bound holds for $Y(w')-Y(w)$. Thus, $Y(w)$ is $\frac{r}{\sqrt{n}}$-Lipschitz continuous. By Lemma \ref{lemma:concentration_Lip}, we have that with probability at least $1-e^{-t^2}$,
        \[
        Y(w)\le \EE_{w}[Y(w)]+r\frac{\sqrt{2}\tau t}{\sqrt{n}}.
        \]
        Now, we analyze the term $\EE_{w}[Y(w)]$. By Jenson's inequality, we have
        \begin{align*}
            \EE_{w}[Y(w)]=&\EE_{w}[\sup_{f\in\BB_r(f^*;\cD)}\{\frac{1}{n}\sum_{i=1}^{n}w_i(f(x_i)-\Breve{f}(x_i))\}]\\
            =&\EE_{w}[\sup_{f\in\BB_r(f^*;\cD)}\{\frac{1}{n}\sum_{i=1}^{n}(w_i-\EE[w_i'])(f(x_i)-\Breve{f}(x_i))\}]\\
            \le &\EE_{w,w'}[\sup_{f\in\BB_r(f^*;\cD)}\{\frac{1}{n}\sum_{i=1}^{n}(w_i-w_i')(f(x_i)-\Breve{f}(x_i))\}],\\
        \end{align*}
        where $w'$ is an independent copy of $w$. Denote $\tilde{w}$ to be $(w-w')/2$ which has a symmetric distribution, then, we have that
        \begin{align*}
        &\EE_{w,w'}[\sup_{f\in\BB_r(f^*;\cD)}\{\frac{1}{n}\sum_{i=1}^{n}(w_i-w_i')(f(x_i)-\Breve{f}(x_i))\}]\\
        =&2\EE_{\varepsilon,\tilde{w}}[\sup_{f\in\BB_r(f^*;\cD)}\{\frac{1}{n}\sum_{i=1}^{n}(w_i-w_i')(f(x_i)-\Breve{f}(x_i))\}]=2\EE_{\tilde{w},\varepsilon}[\cZ_\cD^\varepsilon(r)].
        \end{align*}
        Plugging this back and we finish the proof.
    \end{proof}
    \begin{proof}[Proof of Lemma \ref{lemma:Bound_EE[Z]<EE[Q]}]
        For any realization of $\varepsilon$, denote $h$ as the function that achieves the maximum in $\cZ^\varepsilon_{\cD}$. Then, we have 
        \begin{align*}
            \cZ_{\cD}^\varepsilon(r)=\frac{1}{n}\sum_{i=1}^{n}\varepsilon_i\tilde{w}_i(h(x_i)-f^*(x_i))
            &=\frac{1}{n}\sum_{i=1}^{n}\varepsilon_i\tilde{w}_i(h(x_i)-\Breve{f}(x_i))+\frac{1}{n}\sum_{i=1}^{n}\varepsilon_i\tilde{w}_i(\Breve{f}(x_i)-f^*(x_i))\\
        \end{align*}
        For the term $\frac{1}{n}\sum_{i=1}^{n}\varepsilon_i\tilde{w}_i(\Breve{f}(x_i)-f^*(x_i))$, we take expectation with respect to $\varepsilon_i,i\in[n]$ and notice that $\hat{f}$ and $f^*$ are independent of $\cbr{\varepsilon_{i}}_{i=1}^{n}$, then by the tower property of conditional expectation, we have
        \[
        \EE_{\tilde{w}}[\EE_{\varepsilon}[\frac{1}{n}\sum_{i=1}^{n}\varepsilon_i\tilde{w}_i(\Breve{f}(x_i)-f^*(x_i))]]=0.
        \]
        On the other hand, $h\in\BB_r(f^*;\cD)$, then $\|h-\Breve{f}\|_{\cD}\le \|h-f^*\|_\cD+\|f^*-\Breve{f}\|_\cD\le r+\hat{r}_n$. Thus,
        \[
        \frac{1}{n}\sum_{i=1}^{n}\varepsilon_i\tilde{w}_i(h(x_i)-\Breve{f}(x_i))\le \sup_{f\in\BB_{r+\hat{r}_n}(\Breve{f};\cD)}\cbr{\frac{1}{n}\sum_{i=1}^{n}\varepsilon_i\tilde{w}_i(f(x_i)-\Breve{f}(x_i))}.
        \]
        Taking expectation with respect to $\tilde{w}$ and $\varepsilon$ on both sides, we finish the proof.
    \end{proof}
    \begin{proof}[Proof of Lemma \ref{lemma:bound_Q_by_W+H}]
    By rearranging $\varepsilon_iv_i=\varepsilon_i(f^*(x_i)+w_i-\widebar{f}(x_i))$, we have $$\varepsilon_iw_i=\varepsilon_iv_i+\varepsilon_i(\widebar{f}(x_i)-f^*(x_i)).$$
    
By some algebra, we have
        \begin{align*}
\cQ_{\cD}^\varepsilon(r)=&\frac{1}{2}\sup_{f\in\BB_r(\Breve{f};\cD)}\cbr{\frac{1}{n}\sum_{i=1}^{n}\varepsilon_i(w_i-w_i')(f(x_i)-\Breve{f}(x_i))}\\
\le& \frac{1}{2}\sup_{f\in\BB_r(\Breve{f};\cD)}\cbr{\frac{1}{n}\sum_{i=1}^{n}\varepsilon_iw_i(f(x_i)-\Breve{f}(x_i))}+\sup_{f\in\BB_r(\Breve{f};\cD)}\cbr{\frac{1}{n}\sum_{i=1}^{n}\varepsilon_iw_i'(\Breve{f}(x_i)-f(x_i))}.
        \end{align*}
        For the first term, we have
        \begin{align*}
            \sup_{f\in\BB_r(\Breve{f};\cD)}\cbr{\frac{1}{n}\sum_{i=1}^{n}\varepsilon_iw_i(f(x_i)-\Breve{f}(x_i))}\le \cW_\cD^\varepsilon(r)+V_1(\widebar{f}),
        \end{align*}
        where $V_1(\widebar{f})=\sup_{f\in\BB_r(\Breve{f};\cD)}\cbr{\frac{1}{n}\sum_{i=1}^{n}\varepsilon_i(\widebar{f}(x_i)-f^*(x_i))(f(x_i)-\Breve{f}(x_i))}$.
        
        For the second term, we apply Lemma \ref{lemma:concentration_Lip} twice to obtain that with probability at least $1-2e^{-t^2}$,
        \[
        \sup_{f\in\BB_r(\Breve{f};\cD)}\cbr{\frac{1}{n}\sum_{i=1}^{n}\varepsilon_iw_i'(\Breve{f}(x_i)-f(x_i))}\le \cH_{\cD}^\varepsilon(r)+r\frac{4\sqrt{2}\tau t}{\sqrt{n}}+V_2(\widebar{f}),
        \]
        where $V_2(\widebar{f})=\sup_{f\in\BB_r(\Breve{f};\cD)}\cbr{\frac{1}{n}\sum_{i=1}^{n}\varepsilon_i(\widebar{f}(x_i)-f^*(x_i))(\Breve{f}(x_i)-f(x_i))}$.
        Combing these two inequalities together yields the result we need.
    \end{proof}
\begin{proof}[Proof of Lemma \ref{lemma:peeling}]
    For $\cZ_\cD^\varepsilon(r)$, applying Lemma \ref{lemma:concentration_Lip}, $\forall t>0$, with probability at least $1-e^{-t^2}$, we have
    \[
    \EE_{w,w',\varepsilon}[\cZ_\cD^\varepsilon(r)]\le \cZ_\cD^\varepsilon(r)+r\frac{2\tau t}{\sqrt{n}}.
    \]
    Taking $t:=nr/s$, we have that
    \begin{align}\label{ineq:peeling}
    \PP(\EE_{w,w',\varepsilon}[\cZ_\cD^\varepsilon(r)]\ge \cZ_\cD^\varepsilon(r)+\frac{2\tau}{s}r^2)\le \exp(\frac{-r^2n}{s^2})\le e^{-s^2}.
    \end{align}
    Denoting $q_0:=\frac{s^2}{\sqrt{n}}$ and $q_m:=(1+\frac{1}{s})^mq_0$, we define the following events:
    \[
    \cM=\cbr{\exists r\ge \frac{s^2}{\sqrt{n}}:\EE_{w,w',\varepsilon}[\cZ_\cD^\varepsilon(r)] > \cZ_\cD^\varepsilon([1+\frac{1}{s}]r)+\frac{4\tau}{s}r^2},
    \]
    \[
    \cM_m:=\cbr{\exists r\in[q_m,q_{m+1}):\EE_{w,w',\varepsilon}[\cZ_\cD^\varepsilon(r)] > \cZ_\cD^\varepsilon([1+\frac{1}{s}]r)+\frac{4\tau}{s}r^2}.
    \]
    Then, we have $\cM=\cup_{m=0}^\infty M_m$ and we only need to show $\PP(\cM)\le e^{-s^2}$.
    Now, we upper bound $\PP(\cM_m)$, If $\cM_m$ holds, then, we have
    \begin{align*}
        \EE_{w,w',\varepsilon}[\cZ_\cD^\varepsilon(q_{m+1})]\ge \EE_{w,w',\varepsilon}[\cZ_\cD^\varepsilon(r)]>&\cZ_\cD^\varepsilon([1+\frac{1}{s}]r)+\frac{4\tau}{s}r^2\\
        \ge&\cZ_\cD^\varepsilon(q_{m+1})+\frac{4\tau}{s}q_m^2\\
        \ge&\cZ_\cD^\varepsilon(q_{m+1})+\frac{2\tau}{s}q_{m+1}^2.
    \end{align*}
    The first inequality holds because $q_{m+1}\ge r$; the second inequality holds by the definition of $\cM_m$; the third one holds because $(1+1/s)r\ge(1+1/s)q_m=q_{m+1}$ and $r\ge q_m$; the last one holds because $\frac{q_m^2}{q_{m+1}^2}=\frac{1}{(1+1/s)^2}\ge\frac{1}{2}$ for $s\ge 3$.

    Thus, applying inequality (\ref{ineq:peeling}) with $r=q_{m+1}$, we have
    \[
    \PP(\cM_m)\le \PP(\EE_{w,w',\varepsilon}[\cZ_\cD^\varepsilon(q_{m+1})]\ge \cZ_\cD^\varepsilon(q_{m+1})+\frac{2\tau}{s}q_{m+1}^2)\le e^{\frac{-n}{s^2}q_{m+1}^2}.
    \]
    By a union bound, we have
    \[
    \PP(\cM)\le \sum_{m\ge 0}\PP(\cM_m)\le \sum_{m=0}^{\infty}e^{\frac{-n}{s^2}q_{m+1}^2}\le e^{-s^2}\sum_{m=0}^{\infty}e^{-2(m+1)s}\le e^{-s^2}.
    \]
    We finish the proof.
\end{proof}
    
\section{Proofs of Lemmas in Appendix \ref{app:proofs_sec:high_dimension}}\label{app:proofs_lemmas_in_proofs_sec:high_dimension}
\begin{proof}[Proof of Lemma \ref{lemma:psi2_diameter}]
For any $k$, by integrating over all other $d-1$ dimensions, we know that the marginal Radon-Nikodym derivative $c\le p_k(x)=\frac{d\nu_k}{d\mu_k}\le C$, where $\nu_k$ and $\mu_k$ are the marginal measures of $\nu$ and $\mu$ on the $k$th component, respectively. For any fixed $w \in \mathbb{C}^{2N+1}$ with $\|w\|_2 \le 1$, define the random variable $Z = \langle v_N(x_k), w \rangle$. The $2$-Orlicz norm is defined as $\|Z\|_{\psi_2} = \inf \{ s > 0 : \mathbb{E}_{\nu_k}[\exp(|Z|^2/s^2)] \le 2 \}$. 
\paragraph{Upper Bound:}
By the Cauchy-Schwarz inequality, the absolute value of $Z$ is deterministically bounded:
\[
|Z| \le \|v_N(x_k)\|_2 \|w\|_2 = \sqrt{2N+1} \cdot 1 = \sqrt{2N+1}.
\]
Under the marginal measure $\nu_k$, the second moment of $Z$ is bounded by exploiting the upper bound of the density $p_k(x) \le \widebar{w}$:
\[
\mathbb{E}_{\nu_k}[|Z|^2] = \int_0^1 |\langle v_N(x), w \rangle|^2 p_k(x) dx \le \widebar{w} \int_0^1 |\langle v_N(x), w \rangle|^2 dx = \widebar{w} \|w\|_2^2 \le \widebar{w},
\]
where we used the orthonormality of the Fourier basis under the Lebesgue measure $\mu$.
Using the Taylor expansion of the exponential function and the absolute bound $|Z| \le \sqrt{2N+1}$, we can bound the exponential moments for any $s > 0$:
\begin{align*}
\mathbb{E}_{\nu_k}\left[ \exp\left( \frac{|Z|^2}{s^2} \right) \right] &= 1 + \sum_{m=1}^{\infty} \frac{\mathbb{E}_{\nu_k}[|Z|^{2m}]}{m! s^{2m}} \le 1 + \sum_{m=1}^{\infty} \frac{(\sqrt{2N+1})^{2m-2} \cdot \mathbb{E}_{\nu_k}[|Z|^2]}{m! s^{2m}} \\
&\le 1 + \frac{\widebar{w}}{2N+1} \sum_{m=1}^{\infty} \frac{1}{m!} \left( \frac{2N+1}{s^2} \right)^m = 1 + \frac{\widebar{w}}{2N+1} \left( \exp\left( \frac{2N+1}{s^2} \right) - 1 \right).
\end{align*}
To satisfy the $\psi_2$-norm condition, we require expectation to be at most $2$, which yields:
\[
\exp\left( \frac{2N+1}{s^2} \right) - 1 \le \frac{2N+1}{\widebar{w}} \implies s^2 \ge \frac{2N+1}{\log\left( 1 + \frac{2N+1}{\widebar{w}} \right)}.
\]
Taking the supremum over all $\|w\|_2 \le 1$, we obtain the upper bound $d_{\psi_2}(\mathcal{F}_{lin}^{(k)}) \le \mathcal{O}\left( \sqrt{N / \log N} \right)$.

\paragraph{Lower Bound:}
We construct a specific worst-case weight vector $w^* \in \CC^{2N+1}$ to establish the lower bound. Let $w^*_j = \frac{1}{\sqrt{2N+1}}$ for all $-N \le j \le N$, which trivially satisfies $\|w^*\|_2 = 1$. The corresponding random variable is given by the normalized Dirichlet kernel:
\[
Z^*(x_k) = \frac{1}{\sqrt{2N+1}} \sum_{j=-N}^{N} e^{-2\pi \ib j x_k}.
\]
It is a standard property of the Dirichlet kernel that near the origin, specifically for $x_k \in \mathcal{A} := [-\frac{1}{8N}, \frac{1}{8N}]$, the component waves constructively interfere such that $|Z^*(x_k)| \ge \frac{1}{2}\sqrt{2N+1}$.
The Lebesgue measure of the set $\mathcal{A}$ is $\mu(\mathcal{A}) = \frac{1}{4N}$. Utilizing the positive lower bound of the density $p_k(x) \ge \underline{w} > 0$, the probability of the set $\mathcal{A}$ under the marginal measure $\nu_k$ is:
\[
\nu_k(\mathcal{A}) = \int_{\mathcal{A}} p_k(x) dx \ge \underline{w} \int_{\mathcal{A}} dx = \frac{\underline{w}}{4N}.
\]
We now bound the exponential expectation from below by restricting the integral to $\mathcal{A}$,
\begin{align*}
\mathbb{E}_{\nu_k}\left[ \exp\left( \frac{|Z^*|^2}{s^2} \right) \right] &\ge \int_{\mathcal{A}} \exp\left( \frac{|Z^*(x)|^2}{s^2} \right) d\nu_k(x) \ge \nu_k(\mathcal{A}) \exp\left( \frac{(\frac{1}{2}\sqrt{2N+1})^2}{s^2} \right) \ge \frac{\underline{w}}{4N} \exp\left( \frac{2N+1}{4s^2} \right).
\end{align*}
For the $\psi_2$-norm condition to hold, $s$ must satisfy:
\[
\frac{\underline{w}}{4N} \exp\left( \frac{2N+1}{4s^2} \right) \le 2 \implies \exp\left( \frac{2N+1}{4s^2} \right) \le \frac{8N}{\underline{w}}.
\]
Taking the natural logarithm on both sides gives: $$s^2 \ge \frac{2N+1}{4 \log(8N/\underline{w})}.$$
Thus, for the specific function $f_{w^*} \in \mathcal{F}_{lin}^{(k)}$, its $\psi_2$-norm is bounded below by $\Omega(\sqrt{N / \log N})$. Since $d_{\psi_2}(\cF_{lin})$ is the supremum over the class, we have $d_{\psi_2}(\mathcal{F}_{lin}^{(k)}) \ge \Omega\left( \sqrt{\frac{N}{\log N}} \right)$.
\end{proof}
\begin{proof}[Proof of Lemma \ref{lemma:Orlicz_boundby_2-norm}]
Fix $v, w \in \mathbb{C}^{2N+1}$ and let $u = v - w$. Define the random variable $Z(x) = \langle v_N(x), u \rangle$. Our goal is to bound the Orlicz norm $\|Z\|_{\psi_2}$, where
$$\|Z\|_{\psi_2} := \inf \{ s > 0 : \mathbb{E}_{\nu}[\exp(\frac{|Z|^2}{s^2})] \le 2 \}.$$
First, we establish a deterministic absolute upper bound for $Z(x)$. By the Cauchy-Schwarz inequality and the fact that each component of $X(x)$ is a complex exponential with unit modulus, we have for almost all $x \in [0,1]$:
\[
|Z(x)| = |\langle v_N(x), u \rangle| \le \|v_N(x)\|_2 \|u\|_2 = \sqrt{2N+1} \|u\|_2.
\]
Second, we bound the second moment of $Z(x)$ under the measure $\nu$. Using the upper bound of the Radon-Nikodym derivative $p(x) \le\widebar{w}$ and the orthonormality of the Fourier basis under the Lebesgue measure $\mu$, we obtain:
\[
\mathbb{E}_{\nu}[|Z|^2] = \int_0^1 |\langle X(x), u \rangle|^2 p(x) d\mu(x) \le \widebar{w} \int_0^1 |\langle v_N(x), u \rangle|^2 d\mu(x) = \widebar{w} \|u\|_2^2.
\]

To bound the $\psi_2$-norm, we evaluate the expectation of the exponential function using its Taylor series expansion. For any $s > 0$:
\[
\mathbb{E}_{\nu}\left[ \exp\left( \frac{|Z|^2}{s^2} \right) \right] = 1 + \sum_{m=1}^{\infty} \frac{\mathbb{E}_{\nu}[|Z|^{2m}]}{m! s^{2m}}.
\]
We bound the higher-order moments ($m \ge 1$) by extracting one second moment and bounding the remaining factors using the deterministic absolute upper bound:
\[
\mathbb{E}_{\nu}[|Z|^{2m}] = \mathbb{E}_{\nu}\left[ |Z|^{2m-2} |Z|^2 \right] \le \left( \sqrt{2N+1} \|u\|_2 \right)^{2m-2} \mathbb{E}_{\nu}[|Z|^2] \le \left( 2N+1 \right)^{m-1} \|u\|_2^{2m-2} \cdot \widebar{w} \|u\|_2^2.
\]
Substituting this upper bound into the Taylor series yields:
\begin{align*}
\mathbb{E}_{\nu}\left[ \exp\left( \frac{|Z|^2}{s^2} \right) \right] &\le 1 + \sum_{m=1}^{\infty} \frac{\widebar{w} (2N+1)^{m-1} \|u\|_2^{2m}}{m! s^{2m}} \\
&= 1 + \frac{\widebar{w}}{2N+1} \sum_{m=1}^{\infty} \frac{1}{m!} \left( \frac{(2N+1) \|u\|_2^2}{s^2} \right)^m \\
&= 1 + \frac{\widebar{w}}{2N+1} \left( \exp\left( \frac{(2N+1) \|u\|_2^2}{s^2} \right) - 1 \right).
\end{align*}

By the definition of the $\psi_2$-norm, we require this expectation to be at most $2$. Setting the upper bound to be $\le 2$, we get:
\[
\exp\left( \frac{(2N+1) \|u\|_2^2}{s^2} \right) - 1 \le \frac{2N+1}{\widebar{w}} \implies \exp\left( \frac{(2N+1) \|u\|_2^2}{s^2} \right) \le 1 + \frac{2N+1}{\widebar{w}}.
\]
Taking the natural logarithm on both sides:
\[
\frac{(2N+1) \|u\|_2^2}{s^2} \le \log\left( 1 + \frac{2N+1}{\widebar{w}} \right).
\]
Rearranging, we find the required scale $s^2 \ge \frac{2N+1}{\log\left( 1 + \frac{2N+1}{\widebar{w}} \right)} \|u\|_2^2.$ Thus, when $s_0=\sqrt{\frac{2N+1}{\log\left( 1 + \frac{2N+1}{\widebar{w}} \right)}}\|u\|_2$, we have $\EE_\nu[\exp(\frac{|Z|^2}{s_0^2})]\le 2$. Since $\|Z\|_{\psi_2}$ is the infimum over all valid $s$, we have
\[
\|Z\|_{\psi_2}\le s_0\le \sqrt{\frac{2N+1}{\log\left( 1 + \frac{2N+1}{C} \right)}}\|u\|_2\le C'\sqrt{\frac{N}{\log N}}\|u\|_2,
\]
which leads to
\[
d_{\psi_2}(v, w) = \|Z\|_{\psi_2} \le C' \sqrt{\frac{N}{\log N}} \|u\|_2 = C' \sqrt{\frac{N}{\log N}} \|v - w\|_2.
\]
$C'$ is only related to $\widebar{w}$. This completes the proof.
\end{proof}
\begin{proof}[Proof of Lemma \ref{lemma:critical_radius_high_dimension}]
    Similar to Lemma \ref{lemma:critical_radius}, we first approximate $\DD(\cF)^2=\cbr{(f-f')^2:f,f'\in\cF}$.

    For some threshold $N>0$ where the value of $N$ will be determined later, we define the truncated function family
    $\cG_N$ as
    \[
    \cG_N:=\cbr{\rbr{\sum_{k\in\ZZ^d,\|k\|_{\infty}\le N, k\neq 0}(\hat{f}(k)-\hat{f'}(k))e^{2\pi\ib kx}}^2:f,f'\in\cF}.
    \]
    That is, the surrogate function family $\cG_N$ truncates the functions in $\cF$ to the lowest $2N$ non-zero frequencies. Since $e^{\ib x}=\cos(x)+\ib\sin(x)$, we know that $\cG_N$ lies in a $(2N)^d$ dimensional real linear space. By Lemma \ref{lemma:VC_dim}, the VC dimension of $\cG_N$ $\dim_{VC}(\cG_N)$ is upper bounded by $(2N)^d+2$, and the critical radius $\delta_n(\cG_N)$ is upper bounded by $\sqrt\frac{((2N)^d+2)\log(n/((2N)^d+2)))}{n}$.
    
    By the construction of $\cG_N$, we can upper bound the magnitude of $\cG_N$ by
    \[
    \rbr{\sum_{k\in\ZZ^d,\|k\|_{\infty}\le N, k\neq 0}(\hat{f}(k)-\hat{f'}(k))e^{2\pi\ib kx}}^2\le \rbr{\sum_{k\in\ZZ^d,\|k\|_{\infty}\le N, k\neq 0}|\hat{f}(k)-\hat{f'}(k)|}^2\le \rbr{\sum_{k\in\ZZ^d,\|k\|_{\infty}\le N, k\neq 0}\frac{M_v}{|k|^v}}^2.
    \]
    Since $d/2<v\le d$, we define $\eta_v(N):=\rbr{\sum_{k\in\ZZ^d,\|k\|_{\infty}\le N,k\neq 0}\frac{M_v}{|k|^v}}^2$ and have that
    \[
\rbr{\sum_{k\in\ZZ^d,\|k\|_{\infty}\le N, k\neq 0}\frac{M_v}{|k|^v}}^2\le
\begin{dcases}
\frac{4M_v^2S_d^2}{(d-v)^2}N^{2d-2v}, & 1/2<v<1,\\
4M_v^2S_d^2(\log N)^2, & v=1,
\end{dcases}
\]
where $S_d$ is some moderate constant.
Now, we define the normalized function class $\tilde{\cG}_N$ as
\[
\tilde{\cG}_N:=\cbr{\frac{g_N}{\eta_v(N)}: g_N\in\cG_N}.
\]
By construction, every function in $\tilde{\mathcal{G}}_N$ takes values in $[0,1]$. Consequently, the critical radius of the normalized class satisfies
\[
\delta_n(\tilde{\cG}_N)=\frac{\delta_n(\cG_N)}{\eta_v(N)}.
\]
For any $\delta>0$, we apply Lemma \ref{lemma:critical_radius_rakhlin}, to obtain that with probability at least $1-\delta$,
\begin{align*}
&\EE_{x\sim\nu}\sbr{\frac{1}{\eta_v(N)}\rbr{\sum_{k\in\ZZ^d,\|k\|_{\infty}\le N, k\neq 0}(\hat{f}(k)-\hat{f'}(k))e^{2\pi\ib kx}}^2}\\
\le& \frac{2}{n}\sum_{i=1}^{n}\frac{1}{\eta_v(N)}\rbr{\sum_{k\in\ZZ^d,\|k\|_{\infty}\le N, k\neq 0}(\hat{f}(k)-\hat{f'}(k))e^{2\pi\ib kx_i}}^2\\
&+c[\delta_n(\tilde{\cG}_N)]^2+\frac{c'(\log(1/\delta)+\log\log n)}{n},
\end{align*}
uniformly over $g\in\tilde{\cG}_N$. Thus, with probability at least $1-\delta$, uniformly over $\cG_N$, we have
\begin{align}\label{ineq:critical_ineq_G_N_high_dimension}
\EE_{x\sim\nu}\sbr{\rbr{\sum_{|k|\le N, k\neq 0}(\hat{f}(k)-\hat{f'}(k))e^{2\pi\ib kx}}^2}&\le \frac{2}{n}\sum_{i=1}^{n}\rbr{\sum_{ k\in\ZZ^d,\|k\|_{\infty}\le N, k\neq 0}(\hat{f}(k)-\hat{f'}(k))e^{2\pi\ib kx_i}}^2\nonumber\\
&+c[\delta_n(\cG_N)]^2+\frac{c'\eta_v(N)(\log(1/\delta)+\log\log n)}{n}.
\end{align}
On the other hand, we have that for all $(f-f')^2\in\DD(\cF)^2$,
\begin{align*}
    &\EE_{x\sim\nu}\sbr{(f(x)-f'(x))^2}\\
    =&\int_{[0,1]^d}|f(x)-f'(x)|^2d\nu(x)\le \widebar{w}\int_{[0,1]^d}|f(x)-f'(x)|^2d\mu(x)=\widebar{w}\sum_{k\in\ZZ^d}|\hat{f}(k)-\hat{f'}(k)|^2\\
    =&\widebar{w}\rbr{\sum_{k\in\ZZ^d,\|k\|_{\infty}\le N, k\neq 0}|\hat{f}(k)-\hat{f'}(k)|^2+|\hat{f}(0)-\hat{f'}(0)|^2+\sum_{\|k\|_{\infty}>N}|\hat{f}(k)-\hat{f'}(k)|^2}\\
    =&\widebar{w}\int_{0}^{1}\big|\sum_{k\in\ZZ^d,\|k\|_{\infty}\le N, k\neq 0}(\hat{f}(k)-\hat{f'}(k))e^{2\pi\ib kx}\big|^2d\mu(x)+\widebar{w}|\hat{f}(0)-\hat{f'}(0)|^2+\widebar{w}\frac{4M_v^2S_d}{2v-d}\frac{1}{N^{2v-d}}\\
    \le&\frac{\widebar{w}}{\underline{w}}\EE_{x\sim\nu}\sbr{\rbr{\sum_{|k|\le N, k\neq 0}(\hat{f}(k)-\hat{f'}(k))e^{2\pi\ib kx}}^2}+\widebar{w}|\hat{f}(0)-\hat{f'}(0)|^2+\widebar{w}\frac{4M_v^2S_d}{2v-d}\frac{1}{N^{2v-d}}\\
\end{align*}
Them, for the term $\EE_{x\sim\nu}\sbr{\rbr{\sum_{|k|\le N, k\neq 0}(\hat{f}(k)-\hat{f'}(k))e^{2\pi\ib kx}}^2}$, we apply inequality (\ref{ineq:critical_ineq_G_N_high_dimension}) to obtain
\begin{align*}
    &\frac{\widebar{w}}{\underline{w}}\EE_{x\sim\nu}\sbr{\rbr{\sum_{|k|\le N, k\neq 0}(\hat{f}(k)-\hat{f'}(k))e^{2\pi\ib kx}}^2}+\widebar{w}|\hat{f}(0)-\hat{f'}(0)|^2+\widebar{w}\frac{4M_v^2S_d}{2v-d}\frac{1}{N^{2v-d}}\\
    \le&\frac{\widebar{w}}{\underline{w}}\frac{2}{n}\sum_{i=1}^{n}\rbr{\sum_{|k|\le N, k\neq 0}(\hat{f}(k)-\hat{f'}(k))e^{2\pi\ib kx_i}}^2+\frac{\widebar{w}c}{\underline{w}}[\delta_n(\cG_N)]^2+\frac{c'\widebar{w}\eta_v(N)(\log(1/\delta)+\log\log n)}{\underline{w}n}\\
    &+\frac{\widebar{w}}{\underline{w}}|\hat{f}(0)-\hat{f'}(0)|^2+\widebar{w}\frac{4M_v^2}{2v-1}\frac{1}{N^{2v-1}}\\
    \le&\frac{2\widebar{w}}{\underline{w}}\bignorm{(f(x)-f'(x))-\sum_{|k|>N}(\hat{f}(k)-\hat{f'}(k))e^{2\pi\ib kx_i}}_{\cD}^2+\frac{\widebar{w}c}{\underline{w}}[\delta_n(\cG_N)]^2+\widebar{w}\frac{4M_v^2S_d}{2v-d}\frac{1}{N^{2v-d}}\\
    &+\frac{c'\widebar{w}\eta_v(N)(\log(1/\delta)+\log\log n)}{\underline{w}n}
\end{align*}
Since $(a+b)^2\le 2(a^2+b^2)$, we have
\[
\bignorm{(f(x)-f'(x))-\sum_{|k|>N}(\hat{f}(k)-\hat{f'}(k))e^{2\pi\ib kx_i}}_{\cD}^2\le 2\rbr{\|f-f'\|_{\cD}^2+\frac{4M_v^2S_d}{2v-d}\frac{1}{N^{2v-d}}}.
\]
Plugging this bound back into the last inequality we obtained, we have that
\begin{align*}
    &\frac{2\widebar{w}}{\underline{w}}\bignorm{(f(x)-f'(x))-\sum_{|k|>N}(\hat{f}(k)-\hat{f'}(k))e^{2\pi\ib kx_i}}_{\cD}^2+\frac{\widebar{w}c}{\underline{w}}[\delta_n(\cG_N)]^2+\frac{4\widebar{w}M_v^2S_d}{(2v-d)N^{2v-d}}\\
    &+\frac{c'\widebar{w}\eta_v(N)(\log(\frac{1}{\delta})+\log\log n)}{\underline{w}n}\\
    \le&\frac{4\widebar{w}}{\underline{w}}\|f-f'\|_{\cD}^2+\frac{20\widebar{w}M_v^2S_d}{\underline{w}(2v-d)}\frac{1}{N^{2v-d}}+\frac{\widebar{w}c}{\underline{w}}[\delta_n(\cG_N)]^2+\frac{c'\widebar{w}\eta_v(N)(\log(\frac{1}{\delta})+\log\log n)}{\underline{w}n}.
\end{align*}
Finally, noticing the upper bounds of $\delta_n(\cG_N)\le \sqrt\frac{((2N)^d+2)\log(n/((2N)^d+2)))}{n}$ and $\eta_v(N)$ as discussed before, we set $N$ to balance the rate of all the terms to obtain
\begin{enumerate}
    \item if $d/2<v<d$, we set $N\asymp n^{\frac{1}{2v}}$ to obtain
    \begin{align*}
    \EE_{x\sim\nu}[(f(x)-f'(x))^2]\le& \frac{4\widebar{w}}{\underline{w}}\|f-f'\|_{\cD}^2+\frac{20\widebar{w}M_v^2S_d}{\underline{w}(2v-d)}\frac{1}{n^{1-\frac{d}{2v}}}+\frac{9\widebar{w}c}{\underline{w}}\frac{\log(n/8)}{n^{1-\frac{d}{2v}}}\\
    &+\frac{4c'\widebar{w}M_v^2S_d^2(\log(\frac{1}{\delta})+\log\log n)}{\underline{w}(d-v)^2n^{2-\frac{d}{v}}},
    \end{align*}
    \item if $v=d$, we also set $N\asymp n^{\frac{1}{2v}}$ to obtain
    \begin{align*}
    \EE_{x\sim\nu}[(f(x)-f'(x))^2]\le& \frac{4\widebar{w}}{\underline{w}}\|f-f'\|_{\cD}^2+\frac{20\widebar{w}M_d^2S_d}{\underline{w}d}\frac{1}{n^{1/2}}+\frac{9\widebar{w}c}{\underline{w}}\frac{\log(n/8)}{n^{1/2}}\\
    &+\frac{8c'\widebar{w}M_d^2S_d^2\log n (\log(\frac{1}{\delta})+\log\log n)}{\underline{w}n}
    \end{align*}
\end{enumerate}
Combining these two circumstances together, we have that with probability at least $1-\delta$,
$$\EE_{x\sim\nu}[(f(x)-f'(x))^2]\le \frac{4\widebar{w}}{\underline{w}}\|f-f'\|_{\cD}^2+\tilde{C}_{v}^d\frac{\widebar{w}}{\underline{w}}\frac{\log n(\log\log n)\log(1/\delta)}{n^{1-\frac{d}{2v}}},$$
uniformly over $f-f'\in\DD(\cF)$. Thus, we finish the proof.
\end{proof}
\clearpage



\vskip 0.2in
\bibliography{refs}

@article{goel2017eigenvalue,
  title={Eigenvalue decay implies polynomial-time learnability for neural networks},
  author={Goel, Surbhi and Klivans, Adam},
  journal={Advances in Neural Information Processing Systems},
  volume={30},
  year={2017}
}

@article{magaril2003convex,
  title={Convex Analysis: Theory and Applications},
  author={Magaril-Il’yaev, G and Tikhomirov, V},
  journal={Translations of Mathematical Monographs},
  year={2003},
  publisher={American Mathematical Society}
}

@book{wainwright2019high,
  title={High-dimensional statistics: A non-asymptotic viewpoint},
  author={Wainwright, Martin J},
  volume={48},
  year={2019},
  publisher={Cambridge university press}
}

@article{bartlett2005local,
  title={Local rademacher complexities},
  author={Bartlett, Peter L and Bousquet, Olivier and Mendelson, Shahar},
  year={2005}
}

@misc{rakhlin2022mathstat,
  author       = {Sasha Rakhlin},
  title        = {Mathematical Statistics: A Non-Asymptotic Approach},
  year         = {2022},
  howpublished = {\url{https://www.mit.edu/~rakhlin/courses/mathstat/rakhlin_mathstat_sp22.pdf}},
  note         = {Lecture notes, IDS.160, MIT, Spring 2022}
}

@article{wainwright2025wild,
  title={Wild refitting for black box prediction},
  author={Wainwright, Martin J},
  journal={arXiv preprint arXiv:2506.21460},
  year={2025}
}

@incollection{vapnik2015uniform,
  title={On the uniform convergence of relative frequencies of events to their probabilities},
  author={Vapnik, Vladimir N and Chervonenkis, A Ya},
  booktitle={Measures of complexity: festschrift for alexey chervonenkis},
  pages={11--30},
  year={2015},
  publisher={Springer}
}

@misc{hu2025perturbingderivativewildrefitting,
      title={Perturbing the Derivative: Wild Refitting for Model-Free Evaluation of Machine Learning Models under Bregman Losses}, 
      author={Haichen Hu and David Simchi-Levi},
      year={2025},
      eprint={2509.02476},
      archivePrefix={arXiv},
      primaryClass={stat.ML},
      url={https://arxiv.org/abs/2509.02476}, 
}

@book{vershynin2018high,
  title={High-dimensional probability: An introduction with applications in data science},
  author={Vershynin, Roman},
  volume={47},
  year={2018},
  publisher={Cambridge university press}
}

@book{vapnik2013nature,
  title={The nature of statistical learning theory},
  author={Vapnik, Vladimir},
  year={2013},
  publisher={Springer science \& business media}
}

@article{floyd1995sample,
  title={Sample compression, learnability, and the Vapnik-Chervonenkis dimension},
  author={Floyd, Sally and Warmuth, Manfred},
  journal={Machine learning},
  volume={21},
  number={3},
  pages={269--304},
  year={1995},
  publisher={Springer}
}

@misc{van2000empirical,
  title={Empirical process theory and applications},
  author={van de Geer, SA},
  year={2000},
  publisher={Cambridge: Cambridge University Press}
}

@book{massart2007concentration,
  title={Concentration inequalities and model selection: Ecole d'Et{\'e} de Probabilit{\'e}s de Saint-Flour XXXIII-2003},
  author={Massart, Pascal},
  year={2007},
  publisher={Springer}
}

@inproceedings{
hu2025contextual,
title={Contextual Online Decision Making with Infinite-Dimensional Functional Regression},
author={Haichen Hu and Rui Ai and Stephen Bates and David Simchi-Levi},
booktitle={Forty-second International Conference on Machine Learning},
year={2025},
url={https://openreview.net/forum?id=hFnM9AqT5A}
}

@article{vapnik1991principles,
  title={Principles of risk minimization for learning theory},
  author={Vapnik, Vladimir},
  journal={Advances in neural information processing systems},
  volume={4},
  year={1991}
}

@article{barron2002universal,
  title={Universal approximation bounds for superpositions of a sigmoidal function},
  author={Barron, Andrew R},
  journal={IEEE Transactions on Information theory},
  volume={39},
  number={3},
  pages={930--945},
  year={2002},
  publisher={IEEE}
}

@inproceedings{ongiefunction,
  title={A Function Space View of Bounded Norm Infinite Width ReLU Nets: The Multivariate Case},
  author={Ongie, Greg and Willett, Rebecca and Soudry, Daniel and Srebro, Nathan},
  booktitle={International Conference on Learning Representations}
}

@article{kim2021fast,
  title={Fast convergence rates of deep neural networks for classification},
  author={Kim, Yongdai and Ohn, Ilsang and Kim, Dongha},
  journal={Neural Networks},
  volume={138},
  pages={179--197},
  year={2021},
  publisher={Elsevier}
}

@article{tancik2020fourier,
  title={Fourier features let networks learn high frequency functions in low dimensional domains},
  author={Tancik, Matthew and Srinivasan, Pratul and Mildenhall, Ben and Fridovich-Keil, Sara and Raghavan, Nithin and Singhal, Utkarsh and Ramamoorthi, Ravi and Barron, Jonathan and Ng, Ren},
  journal={Advances in neural information processing systems},
  volume={33},
  pages={7537--7547},
  year={2020}
}

@misc{hu2025doublywildrefittingmodelfree,
      title={Doubly Wild Refitting: Model-Free Evaluation of High Dimensional Black-Box Predictions under Convex Losses}, 
      author={Haichen Hu and David Simchi-Levi},
      year={2025},
      eprint={2511.18789},
      archivePrefix={arXiv},
      primaryClass={cs.LG},
      url={https://arxiv.org/abs/2511.18789}, 
}

@article{kaplan2020scaling,
  title={Scaling laws for neural language models},
  author={Kaplan, Jared and McCandlish, Sam and Henighan, Tom and Brown, Tom B and Chess, Benjamin and Child, Rewon and Gray, Scott and Radford, Alec and Wu, Jeffrey and Amodei, Dario},
  journal={arXiv preprint arXiv:2001.08361},
  year={2020}
}

@article{hoffmann2022training,
  title={Training compute-optimal large language models},
  author={Hoffmann, Jordan and Borgeaud, Sebastian and Mensch, Arthur and Buchatskaya, Elena and Cai, Trevor and Rutherford, Eliza and Casas, DDL and Hendricks, Lisa Anne and Welbl, Johannes and Clark, Aidan and others},
  journal={arXiv preprint arXiv:2203.15556},
  volume={10},
  year={2022}
}

@article{ouyang2022training,
  title={Training language models to follow instructions with human feedback},
  author={Ouyang, Long and Wu, Jeffrey and Jiang, Xu and Almeida, Diogo and Wainwright, Carroll and Mishkin, Pamela and Zhang, Chong and Agarwal, Sandhini and Slama, Katarina and Ray, Alex and others},
  journal={Advances in neural information processing systems},
  volume={35},
  pages={27730--27744},
  year={2022}
}

@article{bai2022training,
  title={Training a helpful and harmless assistant with reinforcement learning from human feedback},
  author={Bai, Yuntao and Jones, Andy and Ndousse, Kamal and Askell, Amanda and Chen, Anna and DasSarma, Nova and Drain, Dawn and Fort, Stanislav and Ganguli, Deep and Henighan, Tom and others},
  journal={arXiv preprint arXiv:2204.05862},
  year={2022}
}

@article{horvitz1952generalization,
  title={A generalization of sampling without replacement from a finite universe},
  author={Horvitz, Daniel G and Thompson, Donovan J},
  journal={Journal of the American statistical Association},
  volume={47},
  number={260},
  pages={663--685},
  year={1952},
  publisher={Taylor \& Francis}
}

@article{Bousquet2002,
author = {Bousquet, Olivier},
year = {2002},
month = {01},
pages = {},
title = {Concentration inequalities and empirical processes theory applied to the analysis of learning algorithms}
}

@article{adamczak2008tail,
  title={A tail inequality for suprema of unbounded empirical processes with applications to Markov chains},
  author={Adamczak, Radoslaw},
  year={2008}
}

@article{bousquet2002bennett,
  title={A Bennett concentration inequality and its application to suprema of empirical processes},
  author={Bousquet, Olivier},
  journal={Comptes Rendus Mathematique},
  volume={334},
  number={6},
  pages={495--500},
  year={2002},
  publisher={Elsevier}
}

@article{chen2025sharp,
  title={Sharp concentration of simple random tensors II: Asymmetry},
  author={Chen, Jiaheng and Sanz-Alonso, Daniel},
  journal={arXiv preprint arXiv:2505.24144},
  year={2025}
}

@article{talagrand2014upper,
  title={Upper and lower bounds for stochastic processes},
  author={Talagrand, Michel},
  year={2014},
  publisher={Springer}
}

@article{demeter2015guide,
  title={A guide to Carleson's theorem},
  author={Demeter, Ciprian},
  journal={The Rocky Mountain Journal of Mathematics},
  volume={45},
  number={1},
  pages={169--212},
  year={2015},
  publisher={JSTOR}
}

@article{yoshizawa1954proof,
  title={A proof of the Plancherel theorem},
  author={Yoshizawa, Hisaaki},
  journal={Proceedings of the Japan Academy},
  volume={30},
  number={4},
  pages={276--281},
  year={1954},
  publisher={The Japan Academy}
}

@article{tropp2015introduction,
  title={An introduction to matrix concentration inequalities},
  author={Tropp, Joel A},
  journal={Foundations and trends{\textregistered} in machine learning},
  volume={8},
  number={1-2},
  pages={1--230},
  year={2015},
  publisher={Emerald Publishing Limited}
}

@book{wellner2013weak,
  title={Weak convergence and empirical processes: with applications to statistics},
  author={Wellner, Jon and others},
  year={2013},
  publisher={Springer Science \& Business Media}
}

@inproceedings{liangmapping,
  title={Mapping the Increasing Use of LLMs in Scientific Papers},
  author={Liang, Weixin and Zhang, Yaohui and Wu, Zhengxuan and Lepp, Haley and Ji, Wenlong and Zhao, Xuandong and Cao, Hancheng and Liu, Sheng and He, Siyu and Huang, Zhi and others},
  booktitle={First Conference on Language Modeling}
}

@article{haltaufderheide2024ethics,
  title={The ethics of ChatGPT in medicine and healthcare: a systematic review on Large Language Models (LLMs)},
  author={Haltaufderheide, Joschka and Ranisch, Robert},
  journal={NPJ digital medicine},
  volume={7},
  number={1},
  pages={183},
  year={2024},
  publisher={Nature Publishing Group UK London}
}

@incollection{wasserkrug2024combining,
  title={Combining large language models and OR/MS to make smarter decisions},
  author={Wasserkrug, Segev and Boussioux, L{\'e}onard and Sun, Wei},
  booktitle={Tutorials in operations research: Smarter decisions for a better world},
  pages={1--49},
  year={2024},
  publisher={INFORMS}
}

@inproceedings{moore2023empowering,
  title={Empowering education with llms-the next-gen interface and content generation},
  author={Moore, Steven and Tong, Richard and Singh, Anjali and Liu, Zitao and Hu, Xiangen and Lu, Yu and Liang, Joleen and Cao, Chen and Khosravi, Hassan and Denny, Paul and others},
  booktitle={International Conference on Artificial Intelligence in Education},
  pages={32--37},
  year={2023},
  organization={Springer}
}

@article{russell1995modern,
  title={A modern approach},
  author={Russell, Stuart and Norvig, Peter and Intelligence, Artificial},
  journal={Artificial Intelligence. Prentice-Hall, Egnlewood Cliffs},
  volume={25},
  number={27},
  pages={79--80},
  year={1995}
}

@book{Goodfellow-et-al-2016,
    title={Deep Learning},
    author={Ian Goodfellow and Yoshua Bengio and Aaron Courville},
    publisher={MIT Press},
    note={\url{http://www.deeplearningbook.org}},
    year={2016}
}

@article{castelvecchi2016can,
  title={Can we open the black box of AI?},
  author={Castelvecchi, Davide},
  journal={Nature News},
  volume={538},
  number={7623},
  pages={20},
  year={2016}
}

@article{cahill2025ai,
  title={The AI-enhanced surgeon--integrating black-box artificial intelligence in the operating room},
  author={Cahill, Ronan A and Duffourc, Mindy and Gerke, Sara},
  journal={International Journal of Surgery},
  volume={111},
  number={4},
  pages={2823--2826},
  year={2025},
  publisher={LWW}
}

@article{brown2020language,
  title={Language models are few-shot learners},
  author={Brown, Tom and Mann, Benjamin and Ryder, Nick and Subbiah, Melanie and Kaplan, Jared D and Dhariwal, Prafulla and Neelakantan, Arvind and Shyam, Pranav and Sastry, Girish and Askell, Amanda and others},
  journal={Advances in neural information processing systems},
  volume={33},
  pages={1877--1901},
  year={2020}
}

@article{dubey2024llama,
  title={The llama 3 herd of models},
  author={Grattafiori, Aaron and Dubey, Abhimanyu and Jauhri, Abhinav and Pandey, Abhinav and Kadian, Abhishek and Al-Dahle, Ahmad and Letman, Aiesha and Mathur, Akhil and Schelten, Alan and Vaughan, Alex and others},
  journal={arXiv preprint arXiv:2407.21783},
  year={2024}
}

@article{vaswani2017attention,
  title={Attention is all you need},
  author={Vaswani, Ashish and Shazeer, Noam and Parmar, Niki and Uszkoreit, Jakob and Jones, Llion and Gomez, Aidan N and Kaiser, {\L}ukasz and Polosukhin, Illia},
  journal={Advances in neural information processing systems},
  volume={30},
  year={2017}
}

@article{rudin2019stop,
  title={Stop explaining black box machine learning models for high stakes decisions and use interpretable models instead},
  author={Rudin, Cynthia},
  journal={Nature machine intelligence},
  volume={1},
  number={5},
  pages={206--215},
  year={2019},
  publisher={Nature Publishing Group UK London}
}

@article{kendall2017uncertainties,
  title={What uncertainties do we need in bayesian deep learning for computer vision?},
  author={Kendall, Alex and Gal, Yarin},
  journal={Advances in neural information processing systems},
  volume={30},
  year={2017}
}

@article{lipton2018mythos,
  title={The mythos of model interpretability: In machine learning, the concept of interpretability is both important and slippery.},
  author={Lipton, Zachary C},
  journal={Queue},
  volume={16},
  number={3},
  pages={31--57},
  year={2018},
  publisher={ACM New York, NY, USA}
}

@article{obermeyer2019dissecting,
  title={Dissecting racial bias in an algorithm used to manage the health of populations},
  author={Obermeyer, Ziad and Powers, Brian and Vogeli, Christine and Mullainathan, Sendhil},
  journal={Science},
  volume={366},
  number={6464},
  pages={447--453},
  year={2019},
  publisher={American Association for the Advancement of Science}
}

@article{amodei2016concrete,
  title={Concrete problems in AI safety},
  author={Amodei, Dario and Olah, Chris and Steinhardt, Jacob and Christiano, Paul and Schulman, John and Man{\'e}, Dan},
  journal={arXiv preprint arXiv:1606.06565},
  year={2016}
}

@inproceedings{villalobos2024position,
  title={Position: Will we run out of data? Limits of LLM scaling based on human-generated data},
  author={Villalobos, Pablo and Ho, Anson and Sevilla, Jaime and Besiroglu, Tamay and Heim, Lennart and Hobbhahn, Marius},
  booktitle={Forty-first International Conference on Machine Learning},
  year={2024}
}

@inproceedings{klugscaling,
  title={Scaling Laws For Deep Learning Based Image Reconstruction},
  author={Klug, Tobit and Heckel, Reinhard},
  booktitle={The Eleventh International Conference on Learning Representations}
}

@article{yu2002resampling,
  title={Resampling methods: concepts, applications, and justification},
  author={Yu, Chong Ho},
  journal={Practical Assessment, Research, and Evaluation},
  volume={8},
  number={1},
  year={2002},
  publisher={University of Massachusetts Amherst Libraries}
}

@inproceedings{bartlett1994fat,
  title={Fat-shattering and the learnability of real-valued functions},
  author={Bartlett, Peter L and Long, Philip M and Williamson, Robert C},
  booktitle={Proceedings of the seventh annual conference on Computational learning theory},
  pages={299--310},
  year={1994}
}

@article{haussler1994predicting,
  title={Predicting $\{$0, 1$\}$-functions on randomly drawn points},
  author={Haussler, David and Littlestone, Nick and Warmuth, Manfred K},
  journal={Information and Computation},
  volume={115},
  number={2},
  pages={248--292},
  year={1994},
  publisher={Elsevier}
}

@article{ansuini2019intrinsic,
  title={Intrinsic dimension of data representations in deep neural networks},
  author={Ansuini, Alessio and Laio, Alessandro and Macke, Jakob H and Zoccolan, Davide},
  journal={Advances in Neural Information Processing Systems},
  volume={32},
  year={2019}
}

@article{browne2000cross,
  title={Cross-validation methods},
  author={Browne, Michael W},
  journal={Journal of mathematical psychology},
  volume={44},
  number={1},
  pages={108--132},
  year={2000},
  publisher={Elsevier}
}

@article{tibshirani2017statistical,
  title={Statistical learning},
  author={Tibshirani, Rob and Hastie, Trevor},
  year={2017}
}

@article{antal2014new,
  title={A new resampling method for sampling designs without replacement: the doubled half bootstrap},
  author={Antal, Erika and Till{\'e}, Yves},
  journal={Computational Statistics},
  volume={29},
  number={5},
  pages={1345--1363},
  year={2014},
  publisher={Springer}
}

@inproceedings{quenouille1949approximate,
  title={Approximate tests of correlation in time-series 3},
  author={Quenouille, Maurice H},
  booktitle={Mathematical Proceedings of the Cambridge Philosophical Society},
  volume={45},
  number={3},
  pages={483--484},
  year={1949},
  organization={Cambridge University Press}
}

@article{tukey1958bias,
  title={Bias and confidence in not quite large samples},
  author={Tukey, John},
  journal={Ann. Math. Statist.},
  volume={29},
  pages={614},
  year={1958}
}

@incollection{efron1992bootstrap,
  title={Bootstrap methods: another look at the jackknife},
  author={Efron, Bradley},
  booktitle={Breakthroughs in statistics: Methodology and distribution},
  pages={569--593},
  year={1992},
  publisher={Springer}
}

@article{politis1994large,
  title={Large sample confidence regions based on subsamples under minimal assumptions},
  author={Politis, Dimitris N and Romano, Joseph P},
  journal={The Annals of Statistics},
  pages={2031--2050},
  year={1994},
  publisher={JSTOR}
}

@article{vitter1985random,
  title={Random sampling with a reservoir},
  author={Vitter, Jeffrey S},
  journal={ACM Transactions on Mathematical Software (TOMS)},
  volume={11},
  number={1},
  pages={37--57},
  year={1985},
  publisher={ACM New York, NY, USA}
}

@article{bentley1987programming,
  title={Programming pearls: a sample of brilliance},
  author={Bentley, Jon and Floyd, Bob},
  journal={Communications of the ACM},
  volume={30},
  number={9},
  pages={754--757},
  year={1987},
  publisher={ACM New York, NY, USA}
}

@article{mammen1993bootstrap,
  title={Bootstrap and wild bootstrap for high dimensional linear models},
  author={Mammen, Enno},
  journal={The annals of statistics},
  volume={21},
  number={1},
  pages={255--285},
  year={1993},
  publisher={Institute of Mathematical Statistics}
}

@book{efron1994introduction,
  title={An introduction to the bootstrap},
  author={Efron, Bradley and Tibshirani, Robert J},
  year={1994},
  publisher={Chapman and Hall/CRC}
}

@article{wong2015performance,
  title={Performance evaluation of classification algorithms by k-fold and leave-one-out cross validation},
  author={Wong, Tzu-Tsung},
  journal={Pattern recognition},
  volume={48},
  number={9},
  pages={2839--2846},
  year={2015},
  publisher={Elsevier}
}

@article{fukunaga2002leave,
  title={Leave-one-out procedures for nonparametric error estimates},
  author={Fukunaga, Keinosuke and Hummels, Donald M.},
  journal={IEEE transactions on pattern analysis and machine intelligence},
  volume={11},
  number={4},
  pages={421--423},
  year={2002},
  publisher={IEEE}
}

@article{wong2019reliable,
  title={Reliable accuracy estimates from k-fold cross validation},
  author={Wong, Tzu-Tsung and Yeh, Po-Yang},
  journal={IEEE Transactions on Knowledge and Data Engineering},
  volume={32},
  number={8},
  pages={1586--1594},
  year={2019},
  publisher={IEEE}
}

@incollection{kolmogorov2019varepsilon,
  title={$\varepsilon$-Entropy and $\varepsilon$-Capacity of Sets in Functional Spaces (Excerpt)},
  author={Kolmogorov, AN and Tihomirov, VM},
  booktitle={Classics On Fractals},
  pages={298--339},
  year={2019},
  publisher={CRC Press}
}

@inproceedings{rahaman2019spectral,
  title={On the spectral bias of neural networks},
  author={Rahaman, Nasim and Baratin, Aristide and Arpit, Devansh and Draxler, Felix and Lin, Min and Hamprecht, Fred and Bengio, Yoshua and Courville, Aaron},
  booktitle={International conference on machine learning},
  pages={5301--5310},
  year={2019},
  organization={PMLR}
}

@article{wu2024role,
  title={On the role of attention masks and layernorm in transformers},
  author={Wu, Xinyi and Ajorlou, Amir and Wang, Yifei and Jegelka, Stefanie and Jadbabaie, Ali},
  journal={Advances in Neural Information Processing Systems},
  volume={37},
  pages={14774--14809},
  year={2024}
}

@inproceedings{xiong2020layer,
  title={On layer normalization in the transformer architecture},
  author={Xiong, Ruibin and Yang, Yunchang and He, Di and Zheng, Kai and Zheng, Shuxin and Xing, Chen and Zhang, Huishuai and Lan, Yanyan and Wang, Liwei and Liu, Tieyan},
  booktitle={International conference on machine learning},
  pages={10524--10533},
  year={2020},
  organization={PMLR}
}

@article{durstenfeld1964algorithm,
  title={Algorithm 235: random permutation},
  author={Durstenfeld, Richard},
  journal={Communications of the ACM},
  volume={7},
  number={7},
  pages={420},
  year={1964},
  publisher={ACM New York, NY, USA}
}

@article{funahashi1989approximate,
  title={On the approximate realization of continuous mappings by neural networks},
  author={Funahashi, Ken-Ichi},
  journal={Neural networks},
  volume={2},
  number={3},
  pages={183--192},
  year={1989},
  publisher={Elsevier}
}

@article{hornik1991approximation,
  title={Approximation capabilities of multilayer feedforward networks},
  author={Hornik, Kurt},
  journal={Neural networks},
  volume={4},
  number={2},
  pages={251--257},
  year={1991},
  publisher={Elsevier}
}

@article{hanin2017approximating,
  title={Approximating continuous functions by relu nets of minimal width},
  author={Hanin, Boris and Sellke, Mark},
  journal={arXiv preprint arXiv:1710.11278},
  year={2017}
}

@article{yarotsky2017error,
  title={Error bounds for approximations with deep ReLU networks},
  author={Yarotsky, Dmitry},
  journal={Neural networks},
  volume={94},
  pages={103--114},
  year={2017},
  publisher={Elsevier}
}

@article{zhou2020universality,
  title={Universality of deep convolutional neural networks},
  author={Zhou, Ding-Xuan},
  journal={Applied and computational harmonic analysis},
  volume={48},
  number={2},
  pages={787--794},
  year={2020},
  publisher={Elsevier}
}

@article{hu2025universal,
  title={Universal approximation with softmax attention},
  author={Hu, Jerry Yao-Chieh and Liu, Hude and Chen, Hong-Yu and Wu, Weimin and Liu, Han},
  journal={arXiv preprint arXiv:2504.15956},
  year={2025}
}

@article{yun2019transformers,
  title={Are transformers universal approximators of sequence-to-sequence functions?},
  author={Yun, Chulhee and Bhojanapalli, Srinadh and Rawat, Ankit Singh and Reddi, Sashank J and Kumar, Sanjiv},
  journal={arXiv preprint arXiv:1912.10077},
  year={2019}
}

@inproceedings{zhai2023stabilizing,
  title={Stabilizing transformer training by preventing attention entropy collapse},
  author={Zhai, Shuangfei and Likhomanenko, Tatiana and Littwin, Etai and Busbridge, Dan and Ramapuram, Jason and Zhang, Yizhe and Gu, Jiatao and Susskind, Joshua M},
  booktitle={International Conference on Machine Learning},
  pages={40770--40803},
  year={2023},
  organization={PMLR}
}

@inproceedings{
miyato2018spectral,
title={Spectral Normalization for Generative Adversarial Networks},
author={Takeru Miyato and Toshiki Kataoka and Masanori Koyama and Yuichi Yoshida},
booktitle={International Conference on Learning Representations},
year={2018},
url={https://openreview.net/forum?id=B1QRgziT-},
}

@article{gray2024normalization,
  title={Normalization layer per-example gradients are sufficient to predict gradient noise scale in transformers},
  author={Gray, Gavia and Tiwari, Aman and Bergsma, Shane and Hestness, Joel},
  journal={Advances in Neural Information Processing Systems},
  volume={37},
  pages={93510--93539},
  year={2024}
}

@article{xu2019frequency,
  title={Frequency principle: Fourier analysis sheds light on deep neural networks},
  author={Xu, Zhi-Qin John and Zhang, Yaoyu and Luo, Tao and Xiao, Yanyang and Ma, Zheng},
  journal={arXiv preprint arXiv:1901.06523},
  year={2019}
}

@article{ronen2019convergence,
  title={The convergence rate of neural networks for learned functions of different frequencies},
  author={Ronen, Basri and Jacobs, David and Kasten, Yoni and Kritchman, Shira},
  journal={Advances in Neural Information Processing Systems},
  volume={32},
  year={2019}
}

@inproceedings{bar2022spectral,
  title={A spectral perspective of DNN robustness to label noise},
  author={Bar, Oshrat and Drory, Amnon and Giryes, Raja},
  booktitle={International Conference on Artificial Intelligence and Statistics},
  pages={3732--3752},
  year={2022},
  organization={PMLR}
}

\end{document}